\documentclass[a4paper, twoside]{article}

\usepackage{amssymb}
\usepackage{amsmath}
\usepackage{array}
\usepackage{bm}
\usepackage{algorithm,algpseudocode}
\usepackage{xspace}
\usepackage{xparse}   
\usepackage{xstring}
\usepackage{xifthen} 
\usepackage{boxedminipage}
\usepackage{graphicx}
\usepackage{adjustbox} 
\usepackage{float}
\usepackage{rotating}
\usepackage[colorlinks=true, linkcolor=red]{hyperref} 
\usepackage{url}
\usepackage{nicefrac}
\usepackage{multirow}
\usepackage{comment}
\usepackage{soul}
\usepackage{verbatim}
\usepackage{csquotes} 
\usepackage[T1]{fontenc}    
\usepackage[utf8]{inputenc} 
\usepackage{pifont}
\usepackage{marvosym}
\usepackage[usenames,dvipsnames]{xcolor}
\usepackage{pgfplots}\pgfplotsset{compat=1.18}

\usepackage{fullpage}
\usepackage{listings}
\usepackage[most]{tcolorbox}

\usepackage[top=5mm,includehead,headheight=45pt,left=1.5cm,right=1.5cm,headsep=0.3cm]{geometry} 
\usepackage{fancyhdr}

\usepackage{amsthm}
\usepackage{placeins}
\usepackage{booktabs}

\newif\ifSLIDES
\SLIDESfalse

\newif\ifIEEE
\IEEEfalse

\newif\ifTWOCOLUMNS
\TWOCOLUMNSfalse

\newif\ifANONYMOUS
\ANONYMOUSfalse

\ifTWOCOLUMNS
\newcommand{\coeffresizeboxCOLW}{2.0}  
\newcommand{\coeffonefigCOLW}{1}        
\else
\newcommand{\coeffresizeboxCOLW}{1}     
\newcommand{\coeffonefigCOLW}{0.7}      
\fi

\newcommand{\figBegins}{
\ifTWOCOLUMNS
\begin{figure*}[t]
\else
\begin{figure}[htb]
\fi
}

\newcommand{\figEnds}{
\ifTWOCOLUMNS
\end{figure*} 
\else
\end{figure}
\fi
}

\newcommand{\stmath}[1]{\ifmmode\text{\sout{\ensuremath{#1}}}\else\sout{#1}\fi} 

\definecolor{darkgreen}{rgb}{0,0.7,0}
\definecolor{blue-violet}{rgb}{0.54, 0.17, 0.89}

\newcommand{\tored}{\color{red}\xspace}

\newcommand{\toblue}{\color{blue}\xspace}

\newcommand{\toblack}{\color{black}\xspace}

\newcommand{\rred}[1]{{\textcolor{red}{#1}}}         
\newcommand{\bblue}[1]{\textcolor{blue}{#1}}

\newcommand{\ppurple}[1]{\textcolor{purple}{#1}}
\newcommand{\ggreen}[1]{\textcolor{darkgreen}{#1}}

\definecolor{revisionred}{RGB}{180,0,0}

\NewDocumentCommand\sbold{g}{%
\IfNoValueTF{#1}
{$\bblue{\triangleright}$}
{$\bblue{\triangleright}$ {\bf #1}}
}

\NewDocumentCommand\sred{g}{%
\IfNoValueTF{#1}
{$\bblue{\triangleright}$}
{$\bblue{\triangleright}$ \textcolor{red}{#1}}
}
\NewDocumentCommand\sblue{g}{%
\IfNoValueTF{#1}
{$\bblue{\triangleright}$}
{$\bblue{\triangleright}$ \textcolor{blue}{#1}}
}
\NewDocumentCommand\sdgreen{g}{%
\IfNoValueTF{#1}
{$\ggreen{\triangleright}$}
{$\ggreen{\triangleright}$ \textcolor{darkgreen}{#1}}
}
\NewDocumentCommand\sgreen{g}{%
\IfNoValueTF{#1}
{$\ggreen{\triangleright}$}
{$\ggreen{\triangleright}$ \textcolor{darkgreen}{#1}}
}
\NewDocumentCommand\spurple{g}{%
\IfNoValueTF{#1}
{$\ppurple{\triangleright}$}
{$\ppurple{\triangleright}$ \textcolor{purple}{#1}}
}

\NewDocumentCommand\sbulem {g}{\noindent $\bullet$\IfNoValueTF{#1}{}{{\em #1}}}
\NewDocumentCommand\sbullet{g}{\noindent $\bullet$\IfNoValueTF{#1}{}{{#1}}}

\NewDocumentCommand\quoteen{o m}{\IfNoValueTF{#1}{}{(#1)}``{\em #2}''}
\NewDocumentCommand\citeen{o m}{\IfNoValueTF{#1}{}{(#1)}``{\em #2}''}

\newif\ifDRAFT
\DRAFTtrue     

\NewDocumentCommand\bienfr{G{0} m m}{
\ifthenelse{\equal{#1}{0}}{\noindent{\bf En.} #2\newline}{}
\ifthenelse{\equal{#1}{2}}{\noindent{\bf En.} #2\newline}{}

\ifthenelse{\equal{#1}{1}}{\noindent{\bf Fr.} #3\newline}{}
\ifthenelse{\equal{#1}{2}}{\noindent{\bf Fr.} #3\newline}{}
}

\ifDRAFT 

\newcommand{\makeremark}[2]{
  \newcommand{#1}[1]
    {$\longrightarrow$\textcolor{red}{\sc #2: ##1}$\leftarrow$ \medskip}}

\newcommand{\makeremarkcolor}[3]{
  \newcommand{#1}[1]
    {$\longrightarrow$\textcolor{#3}{\sc #2: ##1}$\leftarrow$ \medskip}}

\newcommand{\added}[1]{{\textcolor{red}{#1}}}         
\newcommand{\changed}[1]{{\textcolor{red}{\sc #1}}}   
\newcommand{\tcomment}[1]{{\textcolor{blue}{\sc #1}}} 

\newcommand{\makeremarkMargin}[2]{\marginpar{\tiny{\color{red}#1}:\color{blue}#2}}

\else 

\newcommand{\makeremark}[2]{\newcommand{#1}[1]{}}

\newcommand{\makeremarkMargin}[2]{}

\newcommand{\added}[1]{{}}
\newcommand{\changed}[1]{{}}
\newcommand{\tcomment}[1]{{}}
\newcommand{\rem}[1]{{}}

\fi

\newcommand*\mnote[3][0pt]{%
  \if l#2\reversemarginpar\def\pointer{\triangleright}%
    \def\stackalignment{r}\fi%
  \if r#2\normalmarginpar\def\pointer{\triangleleft}%
    \def\stackalignment{l}\fi%
  \marginpar{%
    \topinset{%
      \scalebox{1.5}{\textcolor{blue}{$\pointer$}}}{%
      \belowbaseline[-1.5\baselineskip-#1]{%
        \stackengine%
          {-5pt}%
          {\fcolorbox{blue}{white}{\parbox{1.8cm}%
            {\vspace{3pt}\raggedright#3}}}%
          {~\colorbox{white}{\sffamily Note}}%
          {O}%
          {l}%
          {F}%
          {F}%
          {S}%
        }%
      }{%
      3ex+#1}{%
      -2ex}%
}
}

\makeremark{\vgsays}{Victor says}
\makeremark{\vmsays}{Vincent says}

\makeremark{\anoops}{Anoop says}
\makeremark{\arinjays}{Arinjay says}
\makeremark{\sikaos}{Sikao says}
\makeremark{\michals}{Michal says}
\makeremark{\nelsons}{Nelson says}

\makeremark{\fc}{Frederic says}
\makeremark{\FC}{Frederic says}
\makeremarkcolor{\fcb}{Frederic says}{blue}

\makeremark{\ac}{Antoine says}
\makeremark{\ah}{Antoine says}
\makeremark{\sq}{Simon says}
\makeremark{\gc}{Guillaume says}
\makeremark{\as}{Augusto says}
\makeremark{\gs}{Guilherme says}
\makeremark{\nmd}{Noel says}
\makeremark{\td}{Tom says}
\makeremark{\ar}{Andrea says}
\makeremark{\chr}{Charles says}
\makeremark{\da}{Deepesh says}
\makeremark{\al}{Alix says}
\makeremark{\smar}{Simon says}
\makeremark{\dc}{David says}
\makeremark{\db}{Denys says}
\makeremark{\ms}{Meline says}
\makeremark{\tod}{Timothee  says}
\makeremark{\lgold}{Louis says}
\makeremark{\ayos}{Ayoub says}

\makeremark{\dm}{Dorian says}
\makeremark{\pb}{PierreB says}

\makeremark{\ds}{Darsh says}
\makeremark{\rt}{Romain  says}

\makeremark{\edo}{Edoardo says}
\makeremark{\vla}{Vladimir says}
\makeremark{\mma}{Maximillien says}
\makeremark{\theos}{Theo says}

\ifx\question\undefined

\NewDocumentCommand\question{m+g}{%
\IfNoValueTF{#2}
{\bigskip\noindent \ding{43} \quoteen{\em #1}}
{\bigskip\noindent \ding{43} \quoteen{\textcolor{#1}{\em #2}}}
}

\definecolor{block-gray}{gray}{0.85}
\definecolor{amaranth}{rgb}{0.9, 0.17, 0.31}
\definecolor{candypink}{rgb}{0.89, 0.44, 0.48}

\newtcolorbox{qcolgray}{colback=block-gray,boxrule=0pt,boxsep=0pt,breakable}
\newtcolorbox{qcolcyan}{colback=cyan,boxrule=0pt,boxsep=0pt,breakable}
\newtcolorbox{qcolred}{colback=candypink,boxrule=0pt,boxsep=0pt,breakable}

\newcommand{\questionsymbol}{Q}
\newcounter{questiongcounter}
\fi

\newcommand{\beginsupplement}{
\setcounter{table}{0}
\renewcommand{\thetable}{S\arabic{table}}%
\setcounter{figure}{0}
\renewcommand{\thefigure}{S\arabic{figure}}%

\setcounter{section}{0}
\renewcommand{\thesection}{S\arabic{section}}               

}

\newcommand{\beginSI}{\beginsupplement}

\newcommand{\footurl}[1]{\footnote{\url{#1}}}

\ifSLIDES
\else

\newtheorem{lemma}{Lemma.}
\newtheorem{example}{\noindent Example}{}
{}
{}
\newtheorem{problem}{Problem.}

\newtheorem{remark}{Remark}

\fi
\newtheorem{proposition}{Proposition.}

\newcounter{explecounter}

\NewDocumentEnvironment{proof-not-ams}{G{gray}}{%
  \noindent \color{#1}\emph{Proof.}%
}{%
  \hfill$\Box$\par\medskip
}

\newenvironment{proof-app}[1]
 {\noindent \emph{Proof. [#1]}}{$\Box$\par\medskip}

{}
{}
{}

\newenvironment{enumeratep}
{ \begin{enumerate}
    \setlength{\itemsep}{0pt}
    \setlength{\parskip}{0pt}
    \setlength{\parsep}{0pt}     }
{ \end{enumerate}   }

\newcounter{pln}

\NewDocumentCommand\paragraphm{O{4pt} m}{%
{\vspace{#1}
\noindent{\bf #2}
}}
\NewDocumentCommand\paramini{O{4pt} m}{%
{\vspace{#1}
\noindent{\bf #2}
}}
\NewDocumentCommand\paragraphmini{O{4pt} m}{%
{\vspace{#1}
\noindent{\bf #2}
}}

\newcommand{\calA}{{\mathcal A}}

\newcommand{\calC}{{\mathcal C}}
\newcommand{\calD}{{\mathcal D}}
\newcommand{\calE}{{\mathcal E}}
\newcommand{\calF}{{\mathcal F}}

\newcommand{\calI}{{\mathcal I}}

\newcommand{\calK}{{\mathcal K}}
\newcommand{\calL}{{\mathcal L}}
\newcommand{\calM}{{\mathcal M}}
\newcommand{\calN}{{\mathcal N}}

\newcommand{\calR}{{\mathcal R}}
\newcommand{\calS}{{\mathcal S}}
\newcommand{\calT}{{\mathcal T}}
\newcommand{\calU}{{\mathcal U}}
\newcommand{\calV}{{\mathcal V}}
\newcommand{\calW}{{\mathcal W}}

\NewDocumentCommand\mmoA{g}{\IfNoValueTF{#1}{\text{\sffamily \bfseries A}}{\boldsymbol{a}_{#1}}}
\NewDocumentCommand\mmoB{g}{\IfNoValueTF{#1}{\text{\sffamily \bfseries B}}{\boldsymbol{b}_{#1}}}
\NewDocumentCommand\mmoC{g}{\IfNoValueTF{#1}{\text{\sffamily \bfseries C}}{\boldsymbol{c}_{#1}}}
\NewDocumentCommand\mmoD{g}{\IfNoValueTF{#1}{\text{\sffamily \bfseries D}}{\boldsymbol{d}_{#1}}}
\NewDocumentCommand\mmoE{g}{\IfNoValueTF{#1}{\text{\sffamily \bfseries E}}{\boldsymbol{e}_{#1}}}
\NewDocumentCommand\mmoF{g}{\IfNoValueTF{#1}{\text{\sffamily \bfseries F}}{\boldsymbol{f}_{#1}}}
\NewDocumentCommand\mmoG{g}{\IfNoValueTF{#1}{\text{\sffamily \bfseries G}}{\boldsymbol{g}_{#1}}}
\NewDocumentCommand\mmoH{g}{\IfNoValueTF{#1}{\text{\sffamily \bfseries H}}{\boldsymbol{h}_{#1}}}
\NewDocumentCommand\mmoI{g}{\IfNoValueTF{#1}{\text{\sffamily \bfseries I}}{\boldsymbol{i}_{#1}}}
\NewDocumentCommand\mmoJ{g}{\IfNoValueTF{#1}{\text{\sffamily \bfseries J}}{\boldsymbol{j}_{#1}}}
\NewDocumentCommand\mmoK{g}{\IfNoValueTF{#1}{\text{\sffamily \bfseries K}}{\boldsymbol{k}_{#1}}}
\NewDocumentCommand\mmoL{g}{\IfNoValueTF{#1}{\text{\sffamily \bfseries L}}{\boldsymbol{l}_{#1}}}
\NewDocumentCommand\mmoM{g}{\IfNoValueTF{#1}{\text{\sffamily \bfseries M}}{\boldsymbol{m}_{#1}}}
\NewDocumentCommand\mmoN{g}{\IfNoValueTF{#1}{\text{\sffamily \bfseries N}}{\boldsymbol{n}_{#1}}}
\NewDocumentCommand\mmoO{g}{\IfNoValueTF{#1}{\text{\sffamily \bfseries O}}{\boldsymbol{o}_{#1}}}
\NewDocumentCommand\mmoP{g}{\IfNoValueTF{#1}{\text{\sffamily \bfseries P}}{\boldsymbol{p}_{#1}}}
\NewDocumentCommand\mmoQ{g}{\IfNoValueTF{#1}{\text{\sffamily \bfseries Q}}{\boldsymbol{q}_{#1}}}
\NewDocumentCommand\mmoR{g}{\IfNoValueTF{#1}{\text{\sffamily \bfseries R}}{\boldsymbol{r}_{#1}}}
\NewDocumentCommand\mmoS{g}{\IfNoValueTF{#1}{\text{\sffamily \bfseries S}}{\boldsymbol{s}_{#1}}}
\NewDocumentCommand\mmoT{g}{\IfNoValueTF{#1}{\text{\sffamily \bfseries T}}{\boldsymbol{t}_{#1}}}
\NewDocumentCommand\mmoU{g}{\IfNoValueTF{#1}{\text{\sffamily \bfseries U}}{\boldsymbol{u}_{#1}}}
\NewDocumentCommand\mmoV{g}{\IfNoValueTF{#1}{\text{\sffamily \bfseries V}}{\boldsymbol{v}_{#1}}}
\NewDocumentCommand\mmoW{g}{\IfNoValueTF{#1}{\text{\sffamily \bfseries W}}{\boldsymbol{w}_{#1}}}
\NewDocumentCommand\mmoX{g}{\IfNoValueTF{#1}{\text{\sffamily \bfseries X}}{\boldsymbol{x}_{#1}}}
\NewDocumentCommand\mmoY{g}{\IfNoValueTF{#1}{\text{\sffamily \bfseries Y}}{\boldsymbol{y}_{#1}}}
\NewDocumentCommand\mmoZ{g}{\IfNoValueTF{#1}{\text{\sffamily \bfseries Z}}{\boldsymbol{z}_{#1}}}

\NewDocumentCommand\momA{g}{\IfNoValueTF{#1}{\text{\sffamily \bfseries A}}{\boldsymbol{a}_{#1}}}
\NewDocumentCommand\momB{g}{\IfNoValueTF{#1}{\text{\sffamily \bfseries B}}{\boldsymbol{b}_{#1}}}
\NewDocumentCommand\momC{g}{\IfNoValueTF{#1}{\text{\sffamily \bfseries C}}{\boldsymbol{c}_{#1}}}
\NewDocumentCommand\momD{g}{\IfNoValueTF{#1}{\text{\sffamily \bfseries D}}{\boldsymbol{d}_{#1}}}
\NewDocumentCommand\momE{g}{\IfNoValueTF{#1}{\text{\sffamily \bfseries E}}{\boldsymbol{e}_{#1}}}
\NewDocumentCommand\momF{g}{\IfNoValueTF{#1}{\text{\sffamily \bfseries F}}{\boldsymbol{f}_{#1}}}
\NewDocumentCommand\momG{g}{\IfNoValueTF{#1}{\text{\sffamily \bfseries G}}{\boldsymbol{g}_{#1}}}
\NewDocumentCommand\momH{g}{\IfNoValueTF{#1}{\text{\sffamily \bfseries H}}{\boldsymbol{h}_{#1}}}
\NewDocumentCommand\momI{g}{\IfNoValueTF{#1}{\text{\sffamily \bfseries I}}{\boldsymbol{i}_{#1}}}
\NewDocumentCommand\momJ{g}{\IfNoValueTF{#1}{\text{\sffamily \bfseries J}}{\boldsymbol{j}_{#1}}}
\NewDocumentCommand\momK{g}{\IfNoValueTF{#1}{\text{\sffamily \bfseries K}}{\boldsymbol{k}_{#1}}}
\NewDocumentCommand\momL{g}{\IfNoValueTF{#1}{\text{\sffamily \bfseries L}}{\boldsymbol{l}_{#1}}}
\NewDocumentCommand\momM{g}{\IfNoValueTF{#1}{\text{\sffamily \bfseries M}}{\boldsymbol{m}_{#1}}}
\NewDocumentCommand\momN{g}{\IfNoValueTF{#1}{\text{\sffamily \bfseries N}}{\boldsymbol{n}_{#1}}}
\NewDocumentCommand\momO{g}{\IfNoValueTF{#1}{\text{\sffamily \bfseries O}}{\boldsymbol{o}_{#1}}}
\NewDocumentCommand\momP{g}{\IfNoValueTF{#1}{\text{\sffamily \bfseries P}}{\boldsymbol{p}_{#1}}}
\NewDocumentCommand\momQ{g}{\IfNoValueTF{#1}{\text{\sffamily \bfseries Q}}{\boldsymbol{q}_{#1}}}
\NewDocumentCommand\momR{g}{\IfNoValueTF{#1}{\text{\sffamily \bfseries R}}{\boldsymbol{r}_{#1}}}
\NewDocumentCommand\momS{g}{\IfNoValueTF{#1}{\text{\sffamily \bfseries S}}{\boldsymbol{s}_{#1}}}
\NewDocumentCommand\momT{g}{\IfNoValueTF{#1}{\text{\sffamily \bfseries T}}{\boldsymbol{t}_{#1}}}
\NewDocumentCommand\momU{g}{\IfNoValueTF{#1}{\text{\sffamily \bfseries U}}{\boldsymbol{u}_{#1}}}
\NewDocumentCommand\momV{g}{\IfNoValueTF{#1}{\text{\sffamily \bfseries V}}{\boldsymbol{v}_{#1}}}
\NewDocumentCommand\momW{g}{\IfNoValueTF{#1}{\text{\sffamily \bfseries W}}{\boldsymbol{w}_{#1}}}
\NewDocumentCommand\momX{g}{\IfNoValueTF{#1}{\text{\sffamily \bfseries X}}{\boldsymbol{x}_{#1}}}
\NewDocumentCommand\momY{g}{\IfNoValueTF{#1}{\text{\sffamily \bfseries Y}}{\boldsymbol{y}_{#1}}}
\NewDocumentCommand\momZ{g}{\IfNoValueTF{#1}{\text{\sffamily \bfseries Z}}{\boldsymbol{z}_{#1}}}

\DeclareMathOperator*{\ArgminOp}{argmin}

\NewDocumentCommand{\argmin}{o}{%
  \IfNoValueTF{#1}%
    {\ArgminOp}%
    {\ArgminOp_{#1}}%
}

\NewDocumentCommand\dtms{O{k} g}{d^2_{#1}\IfNoValueTF{#2}{}{(#2)}}
\NewDocumentCommand\dtm{O{k} g}{d_{#1}\IfNoValueTF{#2}{}{(#2)}}

\NewDocumentCommand\dtmmed{g}{d_{DTMm}\IfNoValueTF{#1}{}{(#1)}}
\NewDocumentCommand\dtmemd{g}{d_{EMD}\IfNoValueTF{#1}{}{(#1)}}

\NewDocumentCommand\dpdzero{g}{d_{PD0}\IfNoValueTF{#1}{}{(#1)}}
\NewDocumentCommand\dpdone{g}{d_{PD1}\IfNoValueTF{#1}{}{(#1)}}
\NewDocumentCommand\dpdtwo{g}{d_{PD2}\IfNoValueTF{#1}{}{(#1)}}

\NewDocumentCommand\demd{g+g}{%
\IfNoValueTF{#1}{d_{\text{EMD}}}
{d_{\text{EMD}}(#1,#2)}}

\NewDocumentCommand\wasser{O{} g G{}}{
\IfNoValueTF{#3}{\calW_{#1} }{\calW_{#1}^{#3}} 
\IfNoValueTF{#2}{}{(#2)}
}
\NewDocumentCommand\dwassk{O{} g G{}}{\wasser[#1]{#2}{#3}}

\newcommand{\eg}{{\em e.g.}\xspace}
\newcommand{\ie}{{\em i.e.}\xspace}

\NewDocumentCommand\atilde{g}{\IfNoValueTF{#1}{\tilde{a}}{\tilde{a_{#1}}}}
\NewDocumentCommand\ptilde{g}{\IfNoValueTF{#1}{\tilde{p}}{\tilde{p_{#1}}}}
\NewDocumentCommand\qtilde{g}{\IfNoValueTF{#1}{\tilde{q}}{\tilde{q_{#1}}}}

\NewDocumentCommand\degree{g}{%
\IfNoValueTF{#1}
{\ensuremath{^{\circ}}}
{\ensuremath{{#1}^{\circ}}}
}

\NewDocumentCommand\Rn{G{n}}{\mathbb{R}^{#1}}
\NewDocumentCommand\Rd{G{d}}{\mathbb{R}^{#1}}
\NewDocumentCommand\Zn{G{n}}{\mathbb{Z}^{#1}}
\NewDocumentCommand\Zd{G{d}}{\mathbb{Z}^{#1}}

\NewDocumentCommand\Euclidn{G{n}}{\mathbb{E}^{#1}}
\NewDocumentCommand\Euclidd{G{d}}{\mathbb{E}^{#1}}

\NewDocumentCommand\GLn{G{n}}{GL(#1)} 

\NewDocumentCommand\Ogro{G{3}}{O(#1)} 
\NewDocumentCommand\On{G{3}}{O(#1)} 
\NewDocumentCommand\SOn{G{3}}  {SO(#1)} 
\NewDocumentCommand\SOnlie{G{3}}{\mathfrak{so}(#1)} 

\NewDocumentCommand\En{G{3}}{E(#1)} 
\NewDocumentCommand\SEn{O{} G{3}}{SE(#2) \IfNoValueTF{#1}{}{^{#1}} }

\NewDocumentCommand\SEnp{G{3}}{E^{+}(#1)} 

\NewDocumentCommand\SEnlie{O{} G{3}}{\mathfrak{se}(#2) \IfNoValueTF{#1}{}{^{#1}} }

\NewDocumentCommand\TpM{O{p} m}{T_{#1}{#2}} 

\NewDocumentCommand\UFmakeset{g}{\text{UF.make\_set} \IfNoValueTF{#1}{}{(#1)}}
\NewDocumentCommand\UFunion{g}{\text{UF.union} \IfNoValueTF{#1}{}{(#1)}}
\NewDocumentCommand\UFfind{g}{\text{UF.find} \IfNoValueTF{#1}{}{(#1)}}
\NewDocumentCommand\UFnumcc{g}{\text{UF.num\_cc} \IfNoValueTF{#1}{}{(#1)}}
\NewDocumentCommand\UFnumnodes{g}{\text{UF.num\_nodes} \IfNoValueTF{#1}{}{(#1)}}

\NewDocumentCommand\expL{O{p} m}{
\exp
\IfEqCase{#1}{
{p}{(}      
{P}{\big(}
{b}{[}      
{B}{\big[}
{c}{\{}     
{C}{\big\{}
}[\PackageError{expL}{Undefined option to expL: #1}{}]
#2
\IfEqCase{#1}{
{p}{)}
{P}{\big)}
{b}{]}
{B}{\big]}
{c}{\{}
{C}{\big\{}
}[\PackageError{expL}{Undefined option to expL: #1}{}]
}
\NewDocumentCommand\expl{O{p} m}{\expL[#1]{#2}}

\NewDocumentCommand\erf{O{p} m}{
\text{erf}
\IfEqCase{#1}{
{p}{(}      
{P}{\big(}
{b}{[}      
{B}{\big[}
{c}{\{}     
{C}{\big\{}
}[\PackageError{expL}{Undefined option to expL: #1}{}]
#2
\IfEqCase{#1}{
{p}{)}
{P}{\big)}
{b}{]}
{B}{\big]}
{c}{\{}
{C}{\big\{}
}[\PackageError{expL}{Undefined option to expL: #1}{}]
}

\newcommand{\stirlingnumfk}[2]{} 

\NewDocumentCommand\nabladec{O{} G{}}{\nabla_{#1} #2}

\NewDocumentCommand{\divergence}{O{} g}{\nabla_{#1}\!\bullet\IfNoValueTF{#2}{}{\left(#2\right)}}
\NewDocumentCommand\laplacian{O{} G{}}{\bigtriangleup_{#1} #2}
\NewDocumentCommand\gradient{O{} G{}}{\nabla_{#1}#2}

\NewDocumentCommand{\partiald}{s m m} {
\IfBooleanTF{#1}
{\frac{\partial }{\partial #3} #2} 
{\frac{\partial #2}{\partial #3}}
}

\newcommand{\partialD}[2] {\frac{d #1}{d #2}} 
\NewDocumentCommand\partialDe{g g G{2}}{(\partialD{#1}{#2})^{#3}} 

\NewDocumentCommand\partialDk{g g G{2}}   {\frac{d^{#3} {#1}}{d {#2}^{#3}} }      

\NewDocumentCommand\partialdk  {g g G{2}}   {\frac{\partial^{#3} {#1}}{\partial {#2}^{#3}} }      
\NewDocumentCommand\partialdkl{g g G{2}}    {\frac{\partial^{#3} }{\partial {#2}^{#3}} #1}
\NewDocumentCommand\partialdkat{g g G{2} g} {\frac{\partial^{#3} {#1}}{\partial {#2}^{#3}}_{|#4}}
\NewDocumentCommand\partialdklat{g g G{2} g}{ \frac{\partial^{#3} }{\partial {#2}^{#3}} #1 _{|#4}}

\NewDocumentCommand\partialdtwo{g g g}{ \frac{\partial^{2} {#1}}{\partial {#2} \partial {#3}}}
\NewDocumentCommand\partialdtwoat{g g g g}{ \frac{\partial^{2} {\tored #1}}{\partial {#2} \partial {#3}} _{| #4}}

\NewDocumentCommand\distproj{O{} g g}{d_{#2}^{#1}\IfNoValueTF{#3}{}{(#3)}}
\NewDocumentCommand\distHAOS{g g}{\overrightarrow{d}_{#1}\IfNoValueTF{#2}{}{(#2)}} 
\NewDocumentCommand\distHA{g g}{d_{HA}\IfNoValueTF{#1}{}{(#1, #2)}}                 

\NewDocumentCommand\distMHAOS{g g}{\overrightarrow{\overline{d}}_{#1}\IfNoValueTF{#2}{}{(#2)}} 
\NewDocumentCommand\distMHA{g g}{\overline{d}_{HA}\IfNoValueTF{#1}{}{(#1, #2)}}                 

\NewDocumentCommand\command{m+g}{
\IfNoValueTF{#2}
{command-1}
{command-2 }%
}

\NewDocumentCommand\vvnorm{O{} m g}{\left\lVert {#2} \right\rVert_{#1}\IfNoValueTF{#3}{}{^{#3}}}

\NewDocumentCommand\vvnormi{G{\cdot}} {{\vvnorm[\infty]{#1}}}
\NewDocumentCommand\vvnormtv{G{\cdot}}{{\| #1 \|_{TV}}}

\NewDocumentCommand\dist{g}{%
\IfNoValueTF{#1}
{\text{dist}}
{\text{dist}(#1)}
}

\NewDocumentCommand\PSD{G{d}}{S^+_{#1}}
\NewDocumentCommand\PSDcor{G{d}}{\calC^+_{#1}}

\NewDocumentCommand\onevector{G{}}{ {{\bf 1_{#1}}} }
\NewDocumentCommand\lavectorone{G{}}{ {{\bf 1_{#1}}} }

\NewDocumentCommand\identitymatrix{G{}}{ {{\bf I}_{#1}} }
\NewDocumentCommand\idmatrix{G{}}{ {{\bf I}_{#1}} }
\NewDocumentCommand\Imatrix{G{}}{ {\rred{\bf I}_{#1}} } 

\NewDocumentCommand\transpose{m G{}}{{#1}^{\mathsf{#2 T}}}
\NewDocumentCommand\latrans{m G{}}{{#1}^{\mathsf{#2 T}}}

\newcommand{\interior}[1]{ {\kern0pt#1}^{\mathrm{o}}}

\NewDocumentCommand\Tn{G{n}}{ \mathbb{T}^{#1} }
\NewDocumentCommand\Td{G{d}}{ \mathbb{T}^{#1} }
\NewDocumentCommand\torusd{G{d}}{ \mathbb{T}^{#1} }
\NewDocumentCommand\torusdc{O{c} G{d}}{\mathbb{T}_{#1}^{#2} } 
\NewDocumentCommand\torustri{O{{\bm{2\pi}}} G{d}}{\torusdc[#1]{#2}}

\NewDocumentCommand\diam{g}{%
\IfNoValueTF{#1}
{\text{diam}}
{\text{diam}(#1)}
}

\newcommand{\bfvec}[1]{{\mathbf{#1}}} 

\NewDocumentCommand\dotp{O{} g g}{ \langle #2, #3\rangle_{#1}}

\NewDocumentCommand\dotpn{o m m}{%
\langle #2, #3 \rangle>
\IfNoValueTF{#1}
{}
{_{#1}}
}

\NewDocumentCommand\Volume{O{} g}{{\tt Volume}_{#1}\IfNoValueTF{#2}{}{(#2)}}
\NewDocumentCommand\Area  {O{} g}{{\tt Area}_{#1} \IfNoValueTF{#2}{}{(#2)}}

\NewDocumentCommand\Sd{G{d-1}}{S^{#1}} 
\NewDocumentCommand\sphered{G{d-1}}{S^{#1}}

\NewDocumentCommand\Bd{G{d}}{B^{#1}}
\NewDocumentCommand\Bdcr{O{d} g} { B^{#1} \IfNoValueTF{#2}{}{(#2)} }

\NewDocumentCommand\unitSA{g}{A \IfNoValueTF{#1}{}{(#1)}}
\NewDocumentCommand\unitSV{g}{V \IfNoValueTF{#1}{}{(#1)}}

\NewDocumentCommand\spheredr{m+g}{\IfNoValueTF{#2}{S^{#1}}{S^{#1}(#2)}}
\NewDocumentCommand\balldr{m+g}{\IfNoValueTF{#2}{B^{#1}}{B^{#1}(#2)}}
\NewDocumentCommand\chord{O{} m}{\text{\tt chord}_{#1}(#2)} 

\NewDocumentCommand\spheredrarea{m+g}{\IfNoValueTF{#2}{\text{Area}_{#1}}{\text{Area}_{#1}(#2)}}
\NewDocumentCommand\balldrvol{m+g}{\IfNoValueTF{#2}{\text{Vol}_{#1}}{\text{Vol}_{#1}(#2)}}

\NewDocumentCommand\coneDARvol{m+m+g}{
\IfNoValueTF{#3}
{V^{\text{Cone}}_{#1, #2}}
{V^{\text{Cone}}_{#1, #2}(#3)}
}

\NewDocumentCommand\ballcapDARvol{m+m+g}{
\IfNoValueTF{#3}
{V^{\text{Cap}}_{#1, #2}}
{V^{\text{Cap}}_{#1, #2}(#3)}
}
\NewDocumentCommand\sphericalcapDARarea{m+m+g}{
\IfNoValueTF{#3}
{A^{\text{Cap}}_{#1, #2}}
{A^{\text{Cap}}_{#1, #2}(#3)}
}

\NewDocumentCommand\RatioCut{g}{\IfNoValueTF{#1}{\text{RatioCut}}{\text{RatioCut}[#1]}}
\NewDocumentCommand\Ncut{g}{\IfNoValueTF{#1}{\text{NCut}}{\text{NCut}[#1]}}

\NewDocumentCommand\giv{O{} g}{%
\text{name}
\IfNoValueTF{#1}{V_{G}(#2)}V_{#1}(#2){}
}

\NewDocumentCommand\NNG{O{} g}{
\text{NNG}_{#1}
\IfNoValueTF{#2}{}{(#2)}
}

\NewDocumentCommand\MST{O{} g}{
\text{MST}_{#1}
\IfNoValueTF{#2}{}{(#2)}
}

\newcommand{\uniformD}[1]{\calU(#1)}

\NewDocumentCommand\normalD{m g}{\calN(#1 \IfNoValueTF{#2}{)}{\mid #2)}}

\NewDocumentCommand\vonMF{O{} g}{\text{vMF}_{#1}\IfNoValueTF{#2}{}{(#2)}}

\NewDocumentCommand\dopearl{G{X=x}}{\text{do}(#1)} 
 
\NewDocumentCommand\probX{O{} g}{
\mathbb{P}_{#1}\IfNoValueTF{#2}{}{\left[{#2} \right]}}

\NewDocumentCommand\probXP{m+g}{%
\IfNoValueTF{#2}
{\ensuremath{ {\mathbb{P}_{#1}}}}
{\ensuremath{ {\mathbb{P}_{#1}}\left[ {#2} \right]}}
}

\NewDocumentCommand\expX{O{} g}{%
\IfNoValueTF{#2} 
{ {\mathbb{E}_{#1}} }
{ {\mathbb{E}_{#1}}\left[ {#2} \right] }
}

\NewDocumentCommand\expXB{O{} m}{
{\langle #2 \rangle_{#1} }}

\NewDocumentCommand\expXu{O{} g}{%
\mathbb{E}
\IfNoValueTF{#1}
{}
{_{#1}}
\IfNoValueTF{#2}
{}
{[#2]}
}

\NewDocumentCommand\varX{g}{\mathrm{Var} \IfNoValueTF{#1}{}{\left[ {#1} \right]}}

\NewDocumentCommand\sigX{g}{%
\IfNoValueTF{#1}
{\ensuremath{ {\mathbb{\sigma}}}}
{\ensuremath{ {\mathbb{\sigma}}\left[ {#1} \right]}}
}

\NewDocumentCommand\mcflow{g}{\varphi \IfNoValueTF{#1}{}{(#1)}} 
\NewDocumentCommand\mcconduct{g}{\phi \IfNoValueTF{#1}{}{(#1)}} 

\NewDocumentCommand\fisheremp{g}{F \IfNoValueTF{#1}{}{(#1)}}

\NewDocumentCommand\entH{g}{H\IfNoValueTF{#1}{}{(#1)}} 

\NewDocumentCommand\fisherinfo{O{} G{\theta}}{I_{#1}(#2)}

\NewDocumentCommand\fishermat{G{\theta}}{I_{#1}}

\NewDocumentCommand\priorjef{g} {J\IfNoValueTF{}{}{(#1)}}
\NewDocumentCommand\jefprior{g} {J\IfNoValueTF{}{}{(#1)}}

\NewDocumentCommand\divKL{s g g}{\IfBooleanTF{#1}{\widehat{D}}{D}_{\mathrm{KL}} \IfNoValueTF{#2}{}{\left(#2 \Vert #3\right)}}
\NewDocumentCommand\KLdiv{s g g}{\IfBooleanTF{#1}{\widehat{D}}{D}_{\mathrm{KL}} \IfNoValueTF{#2}{}{\left(#2 \Vert #3\right)}}

\NewDocumentCommand\divJS{s g g}{\IfBooleanTF{#1}{\widehat{D}}{D}_{\mathrm{JS}} \IfNoValueTF{#2}{}{\left(#2 \Vert #3\right)}}
\NewDocumentCommand\JSdiv{s g g}{\IfBooleanTF{#1}{\widehat{D}}{D}_{\mathrm{JS}} \IfNoValueTF{#2}{}{\left(#2 \Vert #3\right)}}

\NewDocumentCommand\emcres{O{t} G{ij}}{r_{#2}^{(#1)}}  
\NewDocumentCommand\emcsum{O{t} G{j}}{n_{#2}^{(#1)}}  
\NewDocumentCommand\emcmu{O{t} G{j}}{\mu_{#2}^{(#1)}} 
\NewDocumentCommand\emcw{O{t} G{j}}{w_{#2}^{(#1)}}    
\NewDocumentCommand\emisum{O{t}}{N^{(#1)}}           

\NewDocumentCommand\lacovar{G{}}{\Sigma_{#1}}
\NewDocumentCommand\lacorr{g G{}}{\Omega_{#1}}

\NewDocumentCommand\pearson{g}{\IfNoValueTF{#1}{\text{Pear.}}{\text{Pear.}(#1)}}

\NewDocumentCommand\spear{g}{\IfNoValueTF{#1}{\text{Spear.}}{\text{Spear.}(#1)}}
\NewDocumentCommand\mic{g}{\IfNoValueTF{#1}{\text{MIC}}{\text{MIC}(#1)}}
\NewDocumentCommand\micd{O{d} g}{
\text{MIC}_{#1}
\IfNoValueTF{#2}{}{ (#2)}}

\NewDocumentCommand\pvalue{o}{\text{p}\!-\!\text{val}\IfNoValueTF{#1}{}{_{#1}}}

\NewDocumentCommand\fTheta{g}{\IfNoValueTF{#1}{f_{\Theta}}{f_{\Theta}(#1)}}

\NewDocumentCommand\Xsample{g}{\IfNoValueTF{#1}{\ensuremath{x^{(n)}}\xspace}{\ensuremath{x^{(#1)}}\xspace}}
\NewDocumentCommand\Ysample{g}{\IfNoValueTF{#1}{\ensuremath{y^{(n)}}\xspace}{\ensuremath{y^{(#1)}}\xspace}}

\NewDocumentCommand\mean{O{} g}{{\text{mean}}_{#1}\IfNoValueTF{#2}{}{(#2)}}
\NewDocumentCommand\average{O{} g}{{\text{mean}}_{#1}\IfNoValueTF{#2}{}{(#2)}}

\NewDocumentCommand\median{O{} g}{{\text{median}_{#1}}\IfNoValueTF{#2}{}{(#2)}}
\NewDocumentCommand\minimum{O{} g}{{\text{minimum}_{#1}}\IfNoValueTF{#2}{}{(#2)}}

\NewDocumentCommand\likeli{G{X} G{\Theta}}{L(#1\mid #2)}
\NewDocumentCommand\loglik{G{X} G{\Theta}}{\calL(#1\mid #2)}

\NewDocumentCommand\mmd{O{} g}{%
\text{MMD}
\IfNoValueTF{#1}
{}
{_{#1}}
\IfNoValueTF{#2}
{}
{[#2]}
}

\NewDocumentCommand\mmds{O{} g}{%
\text{MMD}^2
\IfNoValueTF{#1}
{}
{_{#1}}
\IfNoValueTF{#2}
{}
{[#2]}
}

\NewDocumentCommand\normat{O{r} g}{\text{\codecx{Norm.}}_{#1} \IfNoValueTF{#2}{} {(#2)} } 
\NewDocumentCommand\lanorm{O{r} g}{\text{\codecx{Norm.}}_{#1} \IfNoValueTF{#2}{} {(#2)} } 
\NewDocumentCommand\softmax{O{r} g}{\text{\codecx{softmax}}_{#1} \IfNoValueTF{#2}{} {(#2)} }
 
\NewDocumentCommand\sinkhorn{g}{\codecx{Sinkhorn}\IfNoValueTF{#1}{}{(#1)}} 

\NewDocumentCommand\concat{g}{{\codecx{Concat}}{\codecx{Concat}(#1)}}

\NewDocumentCommand\attn{g}{\codecx{Attn} \IfNoValueTF{#1}{}{(#1)}}
\NewDocumentCommand\selfattn{g}{\codecx{SelfAttn} \IfNoValueTF{#1}{}{(#1)}}

\NewDocumentCommand\Katt{O{1} G{}} {K^{(#1)}_{#2}} 
\NewDocumentCommand\Kasink{G{}}{\Katt{#1}}

\NewDocumentCommand\attQW{s O{} g}{\IfBooleanTF{#1}{\prescript{h}{}{W_{Q_{#2}}}}{W_{Q_{#2}}} \IfNoValueTF{#3}{}{(#3)}}
\NewDocumentCommand\attWQ{s O{} g}{\IfBooleanTF{#1}{\prescript{h}{}{W_{Q_{#2}}}}{W_{Q_{#2}}} \IfNoValueTF{#3}{}{(#3)}}

\NewDocumentCommand\attKW{s O{} g}{\IfBooleanTF{#1}{\prescript{h}{}{W_{K_{#2}}}}{W_{K_{#2}}} \IfNoValueTF{#3}{}{(#3)}}
\NewDocumentCommand\attWK{s O{} g}{\IfBooleanTF{#1}{\prescript{h}{}{W_{K_{#2}}}}{W_{K_{#2}}} \IfNoValueTF{#3}{}{(#3)}}

\NewDocumentCommand\attVW{s O{} g}{\IfBooleanTF{#1}{\prescript{h}{}{W_{V_{#2}}}}{W_{V_{#2}}} \IfNoValueTF{#3}{}{(#3)}}
\NewDocumentCommand\attWV{s O{} g}{\IfBooleanTF{#1}{\prescript{h}{}{W_{V_{#2}}}}{W_{V_{#2}}} \IfNoValueTF{#3}{}{(#3)}}

\NewDocumentCommand\attQ{s O{} g}{\IfBooleanTF{#1}{\prescript{h}{}{Q_{#2}}}{Q_{#2}} \IfNoValueTF{#3}{}{(#3)}}
\NewDocumentCommand\attK{s O{} g}{\IfBooleanTF{#1}{\prescript{h}{}{K_{#2}}}{K_{#2}} \IfNoValueTF{#3}{}{(#3)}}
\NewDocumentCommand\attV{s O{} g}{\IfBooleanTF{#1}{\prescript{h}{}{V_{#2}}}{V_{#2}} \IfNoValueTF{#3}{}{(#3)}}

\NewDocumentCommand\roformer{g}{\codecx{RoFormer}\IfNoValueTF{#1}{}{(#1)}}
\NewDocumentCommand\alibi{g}{\codecx{ALiBi}\IfNoValueTF{#1}{}{(#1)}}

\newcommand{\adjdof}[1]{{#1}\text{-dof}}

\NewDocumentCommand\sijref{G{ij}}{\overset{\circ}{s}_{#1}}
\NewDocumentCommand\dijref{G{ij}}{\overset{\circ}{d}_{#1}}

\NewDocumentCommand\Vgnm{g}{V_{GNM}\IfNoValueTF{#1}{}{(#1)}}
\NewDocumentCommand\Vanm{g}{V_{ANM}\IfNoValueTF{#1}{}{(#1)}}

\NewDocumentCommand\wienerW{s O{t} g}{
\IfBooleanTF{#1}{\bar{W}_{#2}}{W_{#2}}
\IfNoValueTF{#3}{}{(#3)}}

\NewDocumentCommand\wienerB{s O{t} g}{
\IfBooleanTF{#1}{\bar{B}_{#2}}{B_{#2}}
\IfNoValueTF{#3}{}{(#3)}}

\NewDocumentCommand\gdct{O{} g}{\text{GDCT}_{#1}\IfNoValueTF{#2}{}{(#2)}}

\NewDocumentCommand\chemeq{ m m O{} O{}} {
{#1 \underset{#4}{\stackrel{#3}{\rightleftharpoons}} #2}
}

\NewDocumentCommand\brackets{m+g}{\IfNoValueTF{#2}{[#1]}{[#1]_{#2}}}

\NewDocumentCommand\deltagd{s}{\IfBooleanTF{#1}{\Delta G_d^{\circ}}{\Delta G_d}}
\NewDocumentCommand\deltaga{s}{\IfBooleanTF{#1}{\Delta G_a^{\circ}}{\Delta G_a}}

\NewDocumentCommand\ddeltagd{s}{\IfBooleanTF{#1}{\deltagd*}{\deltagd}}

\NewDocumentCommand\deltaGexp{g}{%
  \IfNoValueTF{#1}
  {\ensuremath{\Delta G}\xspace}
  {\ensuremath{\Delta G_{#1}}\xspace}
}

\NewDocumentCommand\bsa{g}{%
\IfNoValueTF{#1}
{\ensuremath{\text{BSA\xspace}}}
{\ensuremath{\text{BSA}(#1)}}
}

\NewDocumentCommand\Calpha{O{} g}{\IfNoValueTF{#2}{C_{\alpha}^{#1}}{C_{\alpha;#2}^{#1}}}
\NewDocumentCommand\Cbeta{O{} g} {\IfNoValueTF{#2}{C_{\beta}^{#1}}{C_{\beta;#2}^{#1}}}
\NewDocumentCommand\Cgamma{O{} g}{\IfNoValueTF{#2}{C_{\gamma}^{#1}}{C_{\gamma;#2}^{#1}}}
\NewDocumentCommand\Cdelta{O{} g} {\IfNoValueTF{#2}{C_{\delta}^{#1}}{C_{\delta;#2}^{#1}}}
\NewDocumentCommand\Cepsilon{O{} g} {\IfNoValueTF{#2}{C_{\epsilon}^{#1}}{C_{\epsilon;#2}^{#1}}}

\newcommand{\msa}[1]{MSA\xspace}

\NewDocumentCommand\SASA{g}{%
\IfNoValueTF{#1}
{\ensuremath{\text{SASA}}}
{\ensuremath{\text{SASA}(#1)}}
}

\NewDocumentCommand\hemogflp{g}{%
\IfNoValueTF{#1}
{FL+}{FL+[#1]}
}
\NewDocumentCommand\hemoghlm{g}{%
\IfNoValueTF{#1}
{HL-}{HL-[#1]}
}
\NewDocumentCommand\hemoghlp{g}{%
\IfNoValueTF{#1}
{HL+}{HL+[#1]}
}

\NewDocumentCommand\dlactiv{O{} g}{\sigma_{#1} \IfNoValueTF{#2}{}{(#2)}} 
\NewDocumentCommand\dlsoftmax{O{} g}{\codecx{softmax}_{#1} \IfNoValueTF{#2}{}{(#2)}} 
\NewDocumentCommand\dlonehot{g}{\codecx{onehot} \IfNoValueTF{#1}{}{(#1)}} 
\NewDocumentCommand\dlconcat{O{} g}{\codecx{concat}_{#1} \IfNoValueTF{#2}{}{(#2)}} 

\NewDocumentCommand\dllinearnb{g}{\codecx{LinearNoBias} \IfNoValueTF{#1}{}{(#1)}} 
\NewDocumentCommand\dllinear{g}{\codecx{Linear} \IfNoValueTF{#1}{}{(#1)}} 
\NewDocumentCommand\dlrelu{O{} g}{\codecx{relu}_{#1} \IfNoValueTF{#2}{}{(#2)}} 
\NewDocumentCommand\dlmlpex{g}{\dlrelu{\dllinear{\IfNoValueTF{#1}{}{#1}}}}   

\NewDocumentCommand\dlmlp{O{} g}  {\codecx{MLP}_{#1}  \IfNoValueTF{#2}{}{(#2)}} 
\NewDocumentCommand\dlmlpone{O{} g}{\codecx{MLP}_{#1} \IfNoValueTF{#2}{}{(#2)}} 
\NewDocumentCommand\dlmlptwo{O{} g}{\codecx{MLP2}_{#1}\IfNoValueTF{#2}{}{(#2)}} 

\NewDocumentCommand\dllayernorm{g}{\codecx{LayerNorm} \IfNoValueTF{#1}{}{(#1)}} 
\NewDocumentCommand\dldropout{O{} g}{\codecx{Dropout}_{#1} \IfNoValueTF{#2}{}{(#2)}} 

\NewDocumentCommand\pointj{g}{%
\IfNoValueTF{#1}
{p^{(j)}}{p^{(#1)}}
}

\NewDocumentCommand\alphaij{O{} G{ij}}{\alpha_{#2}^{#1}}
\NewDocumentCommand\betaij{O{} G{ij}} {\beta_{#2}^{#1}}
\NewDocumentCommand\sigmaij{O{} G{ij}}{\sigma_{#2}^{#1}}

\NewDocumentCommand\phiij{O{} G{i}}{\phi_{#2}^{#1}}
\NewDocumentCommand\varphiij{O{} G{i}}{\varphi_{#2}^{#1}}
\NewDocumentCommand\tauij{O{} G{ij}}{\tauij_{#2}^{#1}}
\NewDocumentCommand\deltaij{O{} G{ij}}{\delta_{#2}^{#1}}

\NewDocumentCommand\gammaij{O{}  G{ij}}{\gamma_{#2}^{#1}}
\NewDocumentCommand\Lambdaij{O{} G{ij}}{\Lambda_{#2}^{#1}}
\NewDocumentCommand\Lambdadec{O{} G{}}{\Lambda_{#2}^{#1}}

\NewDocumentCommand\thetai{O{} G{i}}{\theta_{#2}^{#1}}
\NewDocumentCommand\Thetai{O{} G{i}}{\Theta_{#2}^{#1}}

\NewDocumentCommand\lambdarg{g}{\IfNoValueTF{#1}{\lambda}{\lambda_{#1}}}

\NewDocumentCommand\aij{O{} G{ij}}{a_{#2}^{#1}}
\NewDocumentCommand\Aij{O{} G{ij}}{A_{#2}^{#1}}
\NewDocumentCommand\aijvec{O{} G{ij}}{\overrightarrow{a}_{#2}^{#1}}
\NewDocumentCommand\aijt{G{i}G{j}} {a_{#1,#2}}
\NewDocumentCommand\Aijt{G{i}G{j}} {A_{#1,#2}}

\NewDocumentCommand\bij{O{} G{ij}}{b_{#2}^{#1}}
\NewDocumentCommand\Bij{O{} G{ij}}{B_{#2}^{#1}}
\NewDocumentCommand\bijvec{O{} G{ij}}{\overrightarrow{b}_{#2}^{#1}}

\NewDocumentCommand\bijt{G{i}G{j}} {b_{#1,#2}}
\NewDocumentCommand\Bijt{G{i}G{j}} {B_{#1,#2}}

\NewDocumentCommand\cij{G{ij}} {c_{#1}}
\NewDocumentCommand\Cij{G{ij}} {C_{#1}}
\NewDocumentCommand\cijk{G{ijk}} {c_{#1}}
\NewDocumentCommand\Cijk{G{ijk}} {C_{#1}}

\NewDocumentCommand\cijt{G{i}G{j}} {c_{#1,#2}}
\NewDocumentCommand\Cijt{G{i}G{j}} {C_{#1,#2}}
\NewDocumentCommand\calCijt{G{i}G{j}} {\calC_{#1,#2}}

\NewDocumentCommand\dij{O{} G{ij}} {d_{#2}^{#1}}
\NewDocumentCommand\Dij{O{} G{ij}} {D_{#2}^{#1}}

\NewDocumentCommand\dijt{G{i}G{j}} {d_{#1,#2}}
\NewDocumentCommand\Dijt{G{i}G{j}} {D_{#1,#2}}

\NewDocumentCommand\dijts{G{i}G{j}}{d_{#1,#2}^2}
\NewDocumentCommand\dijs{G{ij}}{d_{#1}^2}

\NewDocumentCommand\eij{G{ij}} {e_{#1}}
\NewDocumentCommand\eijt{G{i}G{j}} {e_{#1,#2}}
\NewDocumentCommand\Eij{G{ij}} {E_{#1}}
\NewDocumentCommand\Eijt{G{i}G{j}} {E_{#1,#2}}

\NewDocumentCommand\eijs{G{ij}}{e_{#1}^2}
\NewDocumentCommand\eijts{G{i}G{j}}{e_{#1,#2}^2}

\NewDocumentCommand\fij{G{ij}}{f_{#1}}
\NewDocumentCommand\Fij{G{ij}}{f_{#1}}
\NewDocumentCommand\fijt{G{i}G{j}}{f_{#1,#2}}
\NewDocumentCommand\Fijt{G{i}G{j}}{F_{#1,#2}}

\NewDocumentCommand\gij{G{ij}}{g_{#1}}
\NewDocumentCommand\Gij{G{ij}}{G_{#1}}
\NewDocumentCommand\gijt{G{i}G{j}}{g_{#1,#2}}
\NewDocumentCommand\Gijt{G{i}G{j}}{G_{#1,#2}}

\NewDocumentCommand\hij{G{ij}}{h_{#1}}
\NewDocumentCommand\Hij{G{ij}}{H_{#1}}
\NewDocumentCommand\hijt{G{i}G{j}}{h_{#1,#2}}
\NewDocumentCommand\Hijt{G{i}G{j}}{H_{#1,#2}}

\NewDocumentCommand\iij{G{ij}}{i_{#1}}
\NewDocumentCommand\iijt{G{i}G{j}}{i_{#1,#2}}
\NewDocumentCommand\Iij{G{ij}}{I_{#1}}
\NewDocumentCommand\Iijt{G{i}G{j}}{I_{#1,#2}}

\NewDocumentCommand\jij{G{ij}}{j_{#1}}
\NewDocumentCommand\jijt{G{i}G{j}}{j_{#1,#2}}
\NewDocumentCommand\Jij{G{ij}}{J_{#1}}
\NewDocumentCommand\Jijt{G{i}G{j}}{J_{#1,#2}}

\NewDocumentCommand\kij{O{} G{ij}}{k_{#2}^{#1}}
\NewDocumentCommand\Kij{O{} G{ij}}{K_{#2}^{#1}}
\NewDocumentCommand\kijvec{O{} G{ij}}{\overrightarrow{k}_{#2}^{#1}}

\NewDocumentCommand\mij{G{ij}}{m_{#1}}
\NewDocumentCommand\mijhat{G{ij}}{\hat{m}_{#1}}

\NewDocumentCommand\Mij{G{ij}}{M_{#1}}
\NewDocumentCommand\mijt{G{i}G{j}}{m_{#1,#2}}
\NewDocumentCommand\Mijt{G{i}G{j}}{M_{#1,#2}}

\NewDocumentCommand\nij{O{} G{ij}}{n_{#2}^{#1}}
\NewDocumentCommand\Nij{O{} G{ij}}{N_{#2}^{#1}}

\NewDocumentCommand\oij{O{} G{ij}}{o_{#2}^{#1}}
\NewDocumentCommand\Oij{O{} G{ij}}{O_{#2}^{#1}}
\NewDocumentCommand\oijvec{O{} G{ij}}{\overrightarrow{o}_{#2}^{#1}}
\NewDocumentCommand\oijtilde{O{} G{ij}}{\tilde{o}_{#2}^{#1}}

\NewDocumentCommand\pij{O{} G{ij}}{p_{#2}^{#1}}
\NewDocumentCommand\Pij{O{} G{ij}}{P_{#2}^{#1}}
\NewDocumentCommand\pijvec{O{} G{ij}}{\overrightarrow{p}_{#2}^{#1}}

\NewDocumentCommand\rij{O{} G{ij}}{r_{#2}^{#1}}
\NewDocumentCommand\Rij{O{} G{ij}}{R_{#2}^{#1}}

\NewDocumentCommand\qij{O{} G{ij}}{q_{#2}^{#1}}
\NewDocumentCommand\Qij{O{} G{ij}}{Q_{#2}^{#1}}
\NewDocumentCommand\qijvec{O{} G{ij}}{\overrightarrow{q}_{#2}^{#1}}

\NewDocumentCommand\qbarfunc{O{} g}{\bar{q}_{#1}  \IfNoValueTF{#2}{}{(#2)}}

\NewDocumentCommand\sij{O{} G{ij}}{s_{#2}^{#1}}
\NewDocumentCommand\Sij{O{} G{ij}}{S_{#2}^{#1}}
\NewDocumentCommand\sijvec{O{} G{ij}}{\overrightarrow{s}_{#2}^{#1}}
\NewDocumentCommand\sijtilde{O{} G{ij}}{\tilde{s}_{#2}^{#1}}

\NewDocumentCommand\tij{G{ij}}{t_{#1}}
\NewDocumentCommand\Tij{G{ij}}{T_{#1}}

\NewDocumentCommand\uij{G{ij}}{u_{#1}}
\NewDocumentCommand\Uij{G{ij}}{U_{#1}}
\NewDocumentCommand\uijt{G{i}G{j}}{u_{#1,#2}}
\NewDocumentCommand\Uijt{G{i}G{j}}{U_{#1,#2}}

\NewDocumentCommand\vij{O{} G{ij}}{v_{#2}^{#1}}
\NewDocumentCommand\Vij{O{} G{ij}}{V_{#2}^{#1}}
\NewDocumentCommand\vijvec{O{} G{ij}}{\overrightarrow{v}_{#2}^{#1}}
\NewDocumentCommand\vijt{G{i}G{j}}{v_{#1,#2}}
\NewDocumentCommand\Vijt{G{i}G{j}}{V_{#1,#2}}

\NewDocumentCommand\wij{O{} G{ij}}{w_{#2}^{#1}}
\NewDocumentCommand\Wij{O{} G{ij}}{W_{#2}^{#1}}
\NewDocumentCommand\wijvec{O{} G{ij}}{\overrightarrow{w}_{#2}^{#1}}

\NewDocumentCommand\xij{O{} G{ij}}{x_{#2}^{#1}}
\NewDocumentCommand\Xij{O{} G{ij}}{X_{#2}^{#1}}

\NewDocumentCommand\yij{G{ij}}{y_{#1}}
\NewDocumentCommand\Yij{G{ij}}{Y_{#1}}

\NewDocumentCommand\zij{G{ij}}{z_{#1}}
\NewDocumentCommand\Zij{G{ij}}{Z_{#1}}

\NewDocumentCommand\piij{G{ij}}{\pi_{#1}}
\NewDocumentCommand\thetaij{G{ij}}{\theta_{#1}}
\NewDocumentCommand\Thetaij{G{ij}}{\Theta_{#1}}

\NewDocumentCommand\costpc{g}{\Delta\IfNoValueTF{#1}{}{(#1)}} 

\NewDocumentCommand\eigenvec{O{v} G{v}}{
\IfNoValueTF{#2}{#2}{#2_{#1}}}

\NewDocumentCommand\FMgeodesic{O{} g}{\gamma_{#1} \IfNoValueTF{#2}{}(#2)}

\NewDocumentCommand\bbflow{G{}}{\codecx{BBFlow#1}}

\NewDocumentCommand\OTCPL{g}{\pi \IfNoValueTF{#1}{}{(#1)}}  
\NewDocumentCommand\otcpl{g}{\pi \IfNoValueTF{#1}{}{(#1)}}  
\NewDocumentCommand\otcost{g}{c\IfNoValueTF{#1}{}{(#1)}}  

\NewDocumentCommand\RectFlow{g}{\text{RectFlow}\IfNoValueTF{#1}{}{(#1)}}

\NewDocumentCommand{\verblinett}{v}{{\tt #1}}
\NewDocumentCommand{\verblineem}{v}{{\em #1}}
\NewDocumentCommand{\verblinebf}{v}{{\bf #1}}

\newcommand{\codec}[1]{{\tt \text{#1}}} 
\newcommand{\codecx}[1]{{\tt \text{#1}}\xspace} 

\newcommand{\ucainria}{Centre Inria at Université Côte d'Azur, France\xspace}

\newcounter{classd}

\NewDocumentCommand\lloyd{g}{\codecx{\toblue Lloyd}\IfNoValueTF{#1}{}{(#1)}}
\NewDocumentCommand\lloyditers{g}{\codecx{\toblue Lloyd}\IfNoValueTF{#1}{}{(#1)}}

\NewDocumentCommand\Distproba{g}{D_{\text{COST}}\IfNoValueTF{#1}{}{(#1)}}

\NewDocumentCommand\smartseeding{g}{\codecx{\toblue SmartSeeding}\IfNoValueTF{#1}{}{(#1)}}
\NewDocumentCommand\kmeanspp{O{blue} g}{\codecx{\textcolor{#1}{k-means++}}\IfNoValueTF{#2}{}{(#2)}}

\NewDocumentCommand\Dsquare{g}{D^2\IfNoValueTF{#1}{}{(#1)}}
\NewDocumentCommand\Dsquaremean{g}{\overline{D}^2\IfNoValueTF{#1}{}{(#1)}}

\NewDocumentCommand\onehot{g}{\codecx{One\_hot}\IfNoValueTF{#1}{}{(#1)}}

\NewDocumentCommand\kmeansfunc{o G{K}}{\IfNoValueTF{#1}{\Phi_{#2}}{\Phi_{#2,#1}}}

\NewDocumentCommand\kmeansfun {G{K}}{\Phi_{#1}} 
\NewDocumentCommand\kmeansfuncmean{G{K}}{\overline{\Phi}_{#1}} 
\NewDocumentCommand\kmeansfunmean {G{K}}{\overline{\Phi}_{#1}} 

\NewDocumentCommand\kmeansinertia{G{S} G{s}}{\varphi_{#1}(#2)}
\NewDocumentCommand\kmeansfunS {O{K}}{\Phi_{#1}^\text{D}} 
\NewDocumentCommand\kmeansfunCOM {G{K}}{\Phi_{#1}^\text{S-COM}} 

\NewDocumentCommand\kmeansopt{G{K}}{\Phi_{#1,OPT}}

\newcommand{\frederic}{Fr\'ed\'eric\xspace}

\NewDocumentCommand\gmpfr{g}{%
\IfNoValueTF{#1}
{\codecx{Gmpfr}}
{\codecx{Gmpfr}[#1]}
}

\NewDocumentCommand\scoreDFM{g}{%
\IfNoValueTF{#1}
{\Phi(\calS)}
{\Phi(#1)}
}

\NewDocumentCommand\varinf{O{} g}{
\text{VI}_{#1}
\IfNoValueTF{#2}{}{(#2)}
}

\NewDocumentCommand\pibest{g}{%
\IfNoValueTF{#1}
{{\pi_{best}}}
{{\pi_{best}}(#1)}
}

\NewDocumentCommand\blosum{g}{%
\IfNoValueTF{#1}
{{\codecx{BLOSUM}}}
{{\codecx{BLOSUM#1}}}
}

\NewDocumentCommand\alphafold{G{2}}{\codecx{AlphaFold#1}}

\NewDocumentCommand\sabdab{G{blue}}{\textcolor{#1}{SAbDab}}

\newcommand{\sblref}{\cite{cazals2017structural}}

\newcommand{\sbllong}{Structural Bioinformatics Library\xspace}

\newcommand{\sblweb}{\url{https://sbl.inria.fr}} 
\NewDocumentCommand\sblwebhref{O{black}}{\href{https://sbl.inria.fr}{\textcolor{#1}{Structural Bioinformatics Library}}}

\NewDocumentCommand\powerdist{g}{\pi \IfNoValueTF{#1}{}{(#1)}}

\NewDocumentCommand\vorr{g}{%
\IfNoValueTF{#1}
{\text{Vor}}
{\text{Vor}(#1)
}}

\NewDocumentCommand\kalp{G{\alpha}}{\calK_{#1}}
\NewDocumentCommand\kalpha{G{\alpha}}{\calK_{#1}}
\NewDocumentCommand\Kalpha{G{\alpha}}{\calK_{#1}}
\NewDocumentCommand\acomplex{G{\alpha}}{\calK_{#1}}
\NewDocumentCommand\ashape{G{\alpha}}{\calW_{#1}}

\NewDocumentCommand\ballrestriction{O{R} G{}}{{#1}_{#2}}
\NewDocumentCommand\vorregion{G{i}} { V_{#1}}

\NewDocumentCommand\sfm{g}{\IfNoValueTF{#1}{\calF_{\alpha}}{\calF_{#1}}}
\NewDocumentCommand\sfmA{g}{\IfNoValueTF{#1}{\calF^(A)_{\alpha}}{\calF^{A}_{#1}}}
\NewDocumentCommand\sfmB{g}{\IfNoValueTF{#1}{\calF^(B)_{\alpha}}{\calF^{B}_{#1}}}

\NewDocumentCommand\spectraldom{g}{\IfNoValueTF{#1}{\codecx{SPECTRALDOM}}{\codecx{SPECTRALDOM/#1}}}

\NewDocumentCommand\stiffc{G{}}{ \gamma^\text{Cov}_{#1} }
\NewDocumentCommand\stiffnc{G{}}{ \gamma^\text{NCov}_{#1} }
\NewDocumentCommand\stiffhbond{G{}}{ \gamma^\text{HB}_{#1} }
\NewDocumentCommand\stiffSB{G{}}{  \gamma^\text{SB}_{#1} }

\NewDocumentCommand\tmalign{g}{\codecx{TM}\IfNoValueTF{#1}{}{(#1)}}
\NewDocumentCommand\TMXref{O{A}}{{#1}^\text{ref}} 

\NewDocumentCommand\lrmsdpw{g}{\text{lRMSDpw}\IfNoValueTF{#1}{}{(#1)}}

\NewDocumentCommand\rmsd{gg}{%
\IfNoValueTF{#2}
{\ensuremath{\text{RMSD}}\xspace }
{\ensuremath{\text{RMSD}(#1,#2)}\xspace}
}
\NewDocumentCommand\rmsdw{gg}{%
\IfNoValueTF{#2}
{\ensuremath{\text{RMSD}_\text{w}}\xspace }
{\ensuremath{\text{RMSD}_\text{w}(#1,#2)}\xspace}
}

\NewDocumentCommand\gopt{g+g}{%
\IfNoValueTF{#2}
{\ensuremath{g^{\text{OPT}}}\xspace}
{\ensuremath{g^{\text{OPT}}(#1, #2)}\xspace}
}

\NewDocumentCommand\lrmsdoptrm{gg}{%
\IfNoValueTF{#2}
{\ensuremath{g_{#1}^{\text{OPT}}}}
{\ensuremath{g_{#1}^{\text{OPT}}}(#2)}
}

\NewDocumentCommand\lrmsd{gg}{%
\IfNoValueTF{#2}
{\ensuremath{\text{lRMSD}}\xspace}
{\ensuremath{\text{lRMSD}(#1,#2)}\xspace}
}
\NewDocumentCommand\lrmsds{gg}{%
\IfNoValueTF{#2}
{\ensuremath{\text{lRMSD}^2}\xspace}
{\ensuremath{\text{lRMSD}^2(#1,#2)}\xspace}
}
\NewDocumentCommand\lrmsdvw{gg}{%
\IfNoValueTF{#2}
{\ensuremath{\text{lRMSD}_{\text{vw}}}\xspace}
{\ensuremath{\text{lRMSD}_{\text{vw}}(#1,#2)}\xspace}
}

\NewDocumentCommand\lrmsdvws{gg}{%
\IfNoValueTF{#2}
{\ensuremath{\text{lRMSD}^2_{\text{vw}}}\xspace}
{\ensuremath{\text{lRMSD}^2_{\text{vw}}(#1,#2)}\xspace}
}

\NewDocumentCommand\lrmsdew{gg}{%
\IfNoValueTF{#2}
{\ensuremath{\text{lRMSD}_{\text{ew}}}\xspace}
{\ensuremath{\text{lRMSD}_{\text{ew}}(#1,#2)}\xspace}
}

\NewDocumentCommand\lrmsdews{gg}{%
\IfNoValueTF{#2}
{\ensuremath{\text{lRMSD}^2_{\text{ew}}}\xspace}
{\ensuremath{\text{lRMSD}^2_{\text{ew}}(#1,#2)}\xspace}
}

\NewDocumentCommand\lrmsdfinal{gg}{%
\IfNoValueTF{#2}
{\ensuremath{\text{lRMSD}_{\text{final}}}\xspace}
{\ensuremath{\text{lRMSD}_{\text{final}}(#1,#2)}\xspace}
}

\NewDocumentCommand\rmsdcomb{gg}{%
\IfNoValueTF{#2}
{\ensuremath{\text{RMSD}_{\text{Comb.}}}\xspace}
{\ensuremath{\text{RMSD}_{\text{Comb.}}(#1,#2)}\xspace}
}
\NewDocumentCommand\rmsdcombs{gg}{%
\IfNoValueTF{#2}
{\ensuremath{\text{RMSD}^2_{\text{Comb.}}}\xspace}
{\ensuremath{\text{RMSD}^2_{\text{Comb.}}(#1,#2)}\xspace}
}

\NewDocumentCommand\dCalC{gg}{%
\IfNoValueTF{#2}
{d_{\calC}}
{d_{\calC}(#1,#2)}
}

\NewDocumentCommand\extendConf{g}{
\IfNoValueTF{#1}
{\codec{ExtendConformation}}
{\codec{ExtendConformation}(#1)}
}
\NewDocumentCommand\selectConf{g}{
\IfNoValueTF{#1}
{\codec{SelectConformationToExtend}}
{\codec{SelectConformationToExtend}(#1)}
}
\NewDocumentCommand\acceptConf{g}{
\IfNoValueTF{#1}
{\codec{AcceptConformation}}
{\codec{AcceptConformation}(#1)}
}
\NewDocumentCommand\recordConf{g}{
\IfNoValueTF{#1}
{\codec{RecordNewConformation}}
{\codec{RecordNewConformation}(#1)}
}

\NewDocumentCommand\energyP{gg}{
\IfNoValueTF{#1}
{\codec{E}}
{\codec{E}(#1)}
}

\NewDocumentCommand\sampleConfUnif{gg}{
\IfNoValueTF{#1}
{\codec{SampleConformationUniformly}}
{\codec{SampleConformationUniformly}(#1)} }

\NewDocumentCommand\sampleConfMoveSet{gg}{
\IfNoValueTF{#1}
{\codec{SampleConformationWithMoveSet}}
{\codec{SampleConformationWithMoveSet}(#1,#2)}}

\newcommand{\RRT}{\codecx{RRT}}   
\newcommand{\rrt}{\codecx{RRT}}   
   
\NewDocumentCommand\algoHY{g}{%
\IfNoValueTF{#1}
{\codecx{Hybrid}}
{\codecx{Hybrid-switch}-{#1}}
}
\NewDocumentCommand\algohybrid{g}{%
\IfNoValueTF{#1}
{\codecx{Hybrid}}
{\codecx{Hybrid-switch}-{#1}}
}

\NewDocumentCommand\distij{g}{%
\IfNoValueTF{#1}
{d_{ij}}
{d_{#1}}
}

\newcommand{\distik}{\distik{ik}}

\NewDocumentCommand\flowij{g}{%
\IfNoValueTF{#1}
{f_{ij}}
{f_{#1}}
}

\NewDocumentCommand\motifAB{g}{\IfNoValueTF{#1}{M^{(AB)}}{M^{(AB)}_{#1}}}
\NewDocumentCommand\motifab{g}{\IfNoValueTF{#1}{M^{(AB)}}{M^{(AB)}_{#1}}}

\NewDocumentCommand\motifa{g}{\IfNoValueTF{#1}{M^{(A)}}{M^{(A)}_{#1}}}
\NewDocumentCommand\motifb{g}{\IfNoValueTF{#1}{M^{(B)}}{M^{(B)}_{#1}}}

\NewDocumentCommand\pconsSeq{gg}{%
  \IfNoValueTF{#1}
  {PCS}
  {PCS_{\leq #2}^{#1}}
}

\NewDocumentCommand\consSeq{g}{%
\IfNoValueTF{#1}{CS}
{CS^{(#1; D)}}
}

\NewDocumentCommand\consSS{g}{%
\IfNoValueTF{#1}
{CSS}
{CSS^{(#1, D)}}
}

\NewDocumentCommand\pconsSS{ggg}{%
\IfNoValueTF{#1}
{PCSS}
{PCSS_{\leq #2, \geq #3}^{(#1; D)}
}}

\NewDocumentCommand\scoreij{g+g}{%
\IfNoValueTF{#2}
{s_{ij}}
{s_{#1, #2}}
}

\NewDocumentCommand\dihedralangles{g}{\IfNoValueTF{#1}{\text{DihedralAngles}}{\text{DihedralAngles}(#1)}}

\NewDocumentCommand\cposeA{G{i}}{A^{(#1)}}
\NewDocumentCommand\cposeB{G{i}}{B^{(#1)}}

\NewDocumentCommand\Kdist{G{i} G{j}}{R_{#1, #2}}

\NewDocumentCommand\rigdom{s m g}{D_{#2}\IfBooleanTF{#1}{^{e}}{}\IfNoValueTF{#3}{}{(#3)}}
\NewDocumentCommand\rigdomenv{m g g}{\calE_{#1}\IfNoValueTF{#2}{}{(#2, #3)}} 

\NewDocumentCommand\matrixT{G{i,t}}{T_{#1}}
\NewDocumentCommand\loopij{G{ij}}{L_{#1}}
\NewDocumentCommand\confij{s m g}{
\IfBooleanTF{#1}{\bar{C}_{#2}}{C_{#2}}
\IfNoValueTF{#3}{}{(#3)}}

\NewDocumentCommand\tgeo{g}{\tilde{t}\IfNoValueTF{#1}{}{_{#1}}}

\NewDocumentCommand\tcolbbmin{G{i} G{j} g}{t^c_{#1, #2} \IfNoValueTF{#3}{}{(#3)}}
\NewDocumentCommand\tcolbbmax{G{i} G{j} g}{t^c_{#1, #2} \IfNoValueTF{#3}{}{(#3)}}

\NewDocumentCommand\tcolminccb{g}{t^c\IfNoValueTF{#1}{}{(#1)}}
\NewDocumentCommand\tcolmaxccb{g}{t^c\IfNoValueTF{#1}{}{(#1)}}

\NewDocumentCommand\crispanglespre{g}{\IfNoValueTF{#1}{ \bm{\chi}_{pre} }{\chi_{pre}^{(#1)}}}
\NewDocumentCommand\crispanglespost{g}{\IfNoValueTF{#1}{ \bm{\chi}_{post} }{\chi_{post}^{(#1)}}}

\NewDocumentCommand\tripepml{s m}{\IfBooleanTF{#1}{\textcolor{red}{T^{'}_{#2}}}{T^{'}_{#2}}}
 \NewDocumentCommand\tripepcore{s m}{\IfBooleanTF{#1}{\textcolor{red}{T^{'}_{#2}}}{T^{'}_{#2}}}
\NewDocumentCommand\pepbody{s m g}{
\IfBooleanTF{#1}{\textcolor{blue}{P_{#2}\IfNoValueTF{#3}{}{(#3)}}}{P_{#2}\IfNoValueTF{#3}{}{(#3)}}}

\NewDocumentCommand\algLSonestep{s}{\IfBooleanTF{#1}{\codecx{LS\_one\_step\_approx}}{\codecx{LS\_one\_step}}}

\newcommand{\oneorall}{\text{One|All}}
\newcommand{\numharvecs}{N_V}
\newcommand{\lsoutputrate}{N_{OR}}
\newcommand{\numhares}{N_{ES}}

\NewDocumentCommand\alglsgen{s G{\oneorall} G{\numhares} G{\numharvecs} G{\lsoutputrate} g}{
\IfBooleanTF{#1}{\mathbb{MLS}}{\mathbf{ULS}} 
^{#4;#5}_{#2;#3}                           
\IfNoValueTF{#6}{}{[#6]}                      
}

\NewDocumentCommand\mapMA{g}{
\IfNoValueTF{#1}
{f_{\calM\rightarrow \calA}}
{f_{\calM\rightarrow \calA}(#1)}
}

\NewDocumentCommand\aalpha{G{k} G{i}}{\alpha_{#1,#2}}
\NewDocumentCommand\atau{G{k} G{i}}{\tau_{#1,#2}}
\NewDocumentCommand\asigma{G{k} G{i}}{\sigma_{#1,#2}}
\NewDocumentCommand\adelta{G{k} G{i}}{\delta_{#1,#2}}
\NewDocumentCommand\aeta{G{k} G{i}}{\eta_{#1,#2}}
\NewDocumentCommand\axi{G{k} G{i}}{\xi_{#1,#2}}

\NewDocumentCommand\anglereptuple{G{k} G{i}}{{\alpha_{#1,#2}, \eta_{#1,#2}, \xi_{#1,#2-1}, \delta_{#1,#2-1}}}

\NewDocumentCommand\anglereptau{m m g}{\IfNoValueTF{#3}{\bfvec{A}_{#1,#2}}{\bfvec{A}_{#1,#2}(#3)}}

\NewDocumentCommand\anglereptri{m g}{\IfNoValueTF{#2}{\bfvec{A}_{#1}} {\bfvec{A}_{#1}(#2)}}

\NewDocumentCommand\falpha{m m g}{\IfNoValueTF{#3}{f^{(\alpha)}_{(#1,#2)}}{f^{(\alpha)}_{(#1,#2)}(#3)}}
\NewDocumentCommand\fxi{m m g}   {\IfNoValueTF{#3}{f^{(\xi)}_{(#1,#2)}}{f^{(\xi)}_{(#1,#2)}(#3)}}
\NewDocumentCommand\feta{m m g}  {\IfNoValueTF{#3}{f^{(\eta)}_{(#1,#2)}}{f^{(\eta)}_{(#1,#2)}(#3)}}
\NewDocumentCommand\fdelta{m m g}{\IfNoValueTF{#3}{f^{(\delta)}_{(#1,#2)}}{f^{(\delta)}_{(#1,#2)}(#3)}}

\NewDocumentCommand\dovifunc{g}{\IfNoValueTF{#1}{\text{DOVI}_{\atau}}{\text{DOVI}_{\atau}(#1)}}
\NewDocumentCommand\dovifuncinv{g}{\IfNoValueTF{#1}{\text{DOVI}^{-1}_{\atau}}{\text{DOVI}^{-1}_{\atau}(#1)}}

\NewDocumentCommand\anglespace{g}{\IfNoValueTF{#1}{\calA}{\calA_{#1}}}
\NewDocumentCommand\anglespacevalid{g}{\IfNoValueTF{#1}{\calV}{\calV_{#1}}}
\NewDocumentCommand\anglespacesolution{g}{\IfNoValueTF{#1}{\calS}{\calS_{#1}}}
\NewDocumentCommand\anglespaceclashfree{g}{\IfNoValueTF{#1}{\cal{F}}{\cal{F}_{#1}}}

\NewDocumentCommand\Iboundmin{m m g}{
\IfNoValueTF{#3}
{I_{\tau}^{\text{min}}(\anglereptau{#1}{#2})}
{I_{\tau}^{\text{min}}(\anglereptau{#1}{#2}{#3})}}

\NewDocumentCommand\Iboundmax{m m g}{\IfNoValueTF{#3}
{I_{\tau}^{\text{max}}(\anglereptau{#1}{#2})}
{I_{\tau}^{\text{max}}(\anglereptau{#1}{#2}{#3})}}

\NewDocumentCommand\Iboundminrot{m m g}{
\IfNoValueTF{#3}
{I_{\tau \mid \delta}^{\text{min}}(\anglereptau{#1}{#2})}
{I_{\tau \mid \delta}^{\text{min}}(\anglereptau{#1}{#2}{#3})}}

\NewDocumentCommand\Iboundmaxrot{m m g}{
\IfNoValueTF{#3}
{I_{\tau \mid \delta}^{\text{max}}(\anglereptau{#1}{#2})}
{I_{\tau \mid \delta}^{\text{max}}(\anglereptau{#1}{#2}{#3})}}

\newcommand{\compovec}[2]{}

\NewDocumentCommand\univectrsli{O{i} g}{\IfNoValueTF{#2}{\bfvec{U}^{(t)}_{#1}}{\bfvec{U}^{(t)}_{#1;#2}}}
\NewDocumentCommand\rvtrsli{O{i}}{{C^{(t)}_{#1}}}
\NewDocumentCommand\vectrsli{O{i} g}{\IfNoValueTF{#2}{\bfvec{T}^{(t)}_{#1}}{\bfvec{T}^{(t)}_{#1;#2}}}
\NewDocumentCommand\homtrsl{O{i} g}{{\bfvec{\tilde{T}_{#1}}}\IfNoValueTF{#2}{}{(#2)}}

\NewDocumentCommand\univecroti{O{i} g}{\IfNoValueTF{#2}{\bfvec{U}^{(r)}_{#1}}{\bfvec{U}^{(r)}_{#1;#2}}}
\NewDocumentCommand\rvroti{O{i}}{{C^{(r)}_{#1}}}
\NewDocumentCommand\vecroti{O{i} g}{\IfNoValueTF{#2}{\bfvec{V}^{(r)}_{#1}}{\bfvec{V}^{(r)}_{#1;#2}}}
\NewDocumentCommand\univecrotisq{O{i} g}{\IfNoValueTF{#2}{\bfvec{U}_{#1}^{(r)}{^2}}{\bfvec{U}_{#1;#2}^{(r)}{^2}}}

\NewDocumentCommand\homrot{O{i} g}{{\bfvec{\tilde{R}_{#1}}}\IfNoValueTF{#2}{}{(#2)}}

\NewDocumentCommand\veczi{g}{\IfNoValueTF{#1}{\bf \hat{Z}_i}{\bf \hat{Z}_{#1}}} 
\NewDocumentCommand\vecxi{g}{\IfNoValueTF{#1}{\bf \hat{X}_i}{\bf \hat{X}_{#1}}} 
\NewDocumentCommand\vecyi{g}{\IfNoValueTF{#1}{\bf \hat{Y}_i}{\bf \hat{Y}_{#1}}} 

\NewDocumentCommand\tauip{s g g}{\IfBooleanTF{#1}{\tau^{'}_{#2;#3}}{\tau_{#2;#3}}}
\NewDocumentCommand\sigmaip{s g g}{\IfBooleanTF{#1}{\sigma^{'}_{#2;#3}}{\sigma_{#2;#3}}}

\NewDocumentCommand\Itausig{s m g}{
\IfBooleanTF{#1}{I^{'}}{I} 
_{#2}                      
\IfNoValueTF{#3}{}{(#3)}  
}

\NewDocumentCommand\Itausigrot{s m g}{
\IfBooleanTF{#1}{I_{#2\mid \delta}^{'} }{I_{#2\mid \delta} }
\IfNoValueTF{#3}{}{(#3)}
}
\NewDocumentCommand\Iset{m g}{\calI_{#1} \IfNoValueTF{#2}{}{(#2)}}
\NewDocumentCommand\Isetrot{m g}{\calI_{#1\mid\delta}  \IfNoValueTF{#2}{}{(#2)}}

\NewDocumentCommand\sigmast{g}{\IfNoValueTF{#1}{\sigma_{i-1}^{*}}{\sigma_{i-1}^{*}(#1)}}
\NewDocumentCommand\sigmastp{g}{\IfNoValueTF{#1}{\sigma_{i-1}^{**}}{\sigma_{i-1}^{**}(#1)}}

\NewDocumentCommand\taust{g}{\IfNoValueTF{#1}{\tau_{i}^{*}}{\tau_{i}^{*}(#1)}}
\NewDocumentCommand\taustp{g}{\IfNoValueTF{#1}{\tau_{i}^{**}}{\tau_{i}^{**}(#1)}}

\NewDocumentCommand\tlcdoublex{g}{\IfNoValueTF{#1}{\codec{TLCdouble}\xspace}{\codec{TLCdouble}[-x#1]\xspace}}

\newcommand{\TLCRO}{\widetilde{TLC}}
\renewcommand{\TLCRO}{TLC-RO}
\renewcommand{\TLCRO}{\overline{\calD}}

\NewDocumentCommand\ramdomdat{g}{\IfNoValueTF{#1}{\calR_{\calD}}{\calR_{\calD, #1}}}

\NewDocumentCommand\ramdomrec{g}{\IfNoValueTF{#1}{\calR_{\TLCRO}}{\calR_{\TLCRO, #1}}}

\NewDocumentCommand\NNinDB{O{\calD} m}{%
\IfNoValueTF{#1}{nn_{\calD}(#2)}{nn_{#1}(#2)}
}
\NewDocumentCommand\NNinDBC{O{\calD} m}{%
\IfNoValueTF{#1}{nn_{\calD}^{Class}(#2)}{nn_{#1}^{Class}(#2)}
}

\NewDocumentCommand\distToNN{O{\calD} m}{%
\IfNoValueTF{#1}{d_p^{(Y)}(#2)}{d_p^{(#1)}(#2)}
}

\NewDocumentCommand\potenev{O{} g}{
\IfNoValueTF{#1}{V}{V_{#1}}
\IfNoValueTF{#2}{}{(#2)}
}

\NewDocumentCommand\deltaVr{O{} g}{%
\IfNoValueTF{#1}{\Delta_r V}{\Delta_r V_{#1}}
\IfNoValueTF{#2}{}{(#2)}
}

\NewDocumentCommand\frepfun{O{p} g}{%
\IfNoValueTF{#1}{F_p}{F_{#1}}
\IfNoValueTF{#2}{}{(#2)}
}

\NewDocumentCommand\frepder{g}{%
\IfNoValueTF{#1}{{ F^{'}_p}} {{ F^{'}_p(#1)} }}

\NewDocumentCommand\frepderder{g}{%
\IfNoValueTF{#1}{{ F^{''}_p}} {{ F^{''}_p(#1)} }}

\NewDocumentCommand\signpred{g}{%
\IfNoValueTF{#1}
{\text{\textcolor{blue}{\tt Sign}}\xspace}
{\text{\textcolor{blue}{\tt Sign}}(#1)}
}

\NewDocumentCommand\ispositive{g}{%
\IfNoValueTF{#1}
{\text{\textcolor{blue}{\tt Is\_positive}}\xspace}
{\text{\textcolor{blue}{\tt Is\_positive}}(#1)}
}

\NewDocumentCommand\derivchangessign{g}{%
\IfNoValueTF{#1}
{\text{\textcolor{blue}{\tt Derivative\_changes\_sign}}\xspace}
{\text{\textcolor{blue}{\tt Derivative\_changes\_sign}}(#1)}
}

\NewDocumentCommand\intervaltoowide{g}{%
\IfNoValueTF{#1}
{\text{\textcolor{blue}{\tt Interval\_too\_wide}}}
{\text{\textcolor{blue}{\tt Interval\_too\_wide}}(#1)}
}

\NewDocumentCommand\findroot{g}{%
\IfNoValueTF{#1}
{\text{\tt \bf Find\_root}\xspace}
{\text{\tt \bf Find\_root}\xspace(#1)}
}
\NewDocumentCommand\evalFp{g}{%
\IfNoValueTF{#1}
{\text{\bf Evaluate}}
{\text{\bf Evaluate}F_p(#1)}
}
\NewDocumentCommand\evalFpp{g}{%
\IfNoValueTF{#1}
{\text{\bf Evaluate}}
{\text{\bf Evaluate}F^{'}_p(#1)}
}

\NewDocumentCommand\updateroot{g}{%
\IfNoValueTF{#1}
{\text{\tt \bf Update\_root}\xspace}
{\text{\tt \bf Update\_root}\xspace(#1)}
}

\newcommand{\procRewire}{\textsc{Rewire}\xspace}
\newcommand{\procBestParent}{\textsc{BestParent}\xspace}

\NewDocumentCommand\distCS{g}{\textsc{Dist} \IfNoValueTF{#1}{}{(#1)}}

\NewDocumentCommand\cradleModel{m g}{{\tt Cradle-{#1}\IfNoValueTF{#2}{}{-{#2}}}}
\NewDocumentCommand\twoDimModel{m g}{{\tt 2D-{#1}\IfNoValueTF{#2}{}{-{#2}}}}

\newcommand{\liesen}{\ensuremath{\mathfrak{se}_3}}

\definecolor{darkorange}{HTML}{eb8c15}

\newcommand{\Cfree}{\ensuremath{\calC_{\text{free}}}}
\newcommand{\Exp}{\textrm{Exp}}
\newcommand{\Log}{\textrm{Log}}

\newcommand{\qsampled}{q_\text{sampled}}
\newcommand{\qnew}{q_\text{new}}

\newcommand{\qfrom}{q_\text{from}}

\newcommand{\qstart}{q_0}

\newcommand{\qinfty}{q_\infty}
\newcommand{\qgoal}{q_\infty}

\newcommand{\Ksteps}{K_\text{steps}}

\newcommand{\krrt}{\ensuremath{k_\text{RRT}}\xspace}

\definecolor{DarkGreen}{RGB}{10,50,32}
\NewDocumentCommand\plangennew{O{} m m m}{\textcolor{blue}{\tt #2$^{#3}$-#4${}_{#1}$}\xspace}
\NewDocumentCommand\plangengreen{O{} m m m}{\textcolor{DarkGreen}{\tt #2$^{#3}$-#4${}_{#1}$}\xspace}
\NewDocumentCommand\plangen{O{} m m m}{\textcolor{blue}{\tt #2$^{#3}$-#4${}_{#1}$}\xspace}

\NewDocumentCommand{\rrtO}{O{}}{\plangen[#1]{RRT}{}{$\emptyset$}} 
\NewDocumentCommand{\rrtC}{O{}}{\plangen[#1]{RRT}{}{C}}
\NewDocumentCommand{\rrtSC}{O{}}{\plangen[#1]{RRT}{*}{C}}
\NewDocumentCommand{\rrtQC}{O{}}{\plangen[#1]{RRT}{Q}{C}}
\NewDocumentCommand{\rrtS}{O{}}{\plangen[#1]{RRT}{*}{$\emptyset$}}
\NewDocumentCommand{\rrtQ}{O{}}{\plangen[#1]{RRT}{Q}{$\emptyset$}}

\NewDocumentCommand{\rrtiO}{O{}}{\plangengreen[#1]{RRTI}{}{$\emptyset$}} 
\NewDocumentCommand{\rrtiC}{O{}}{\plangengreen[#1]{RRTI}{}{C}}
\NewDocumentCommand{\rrtiSC}{O{}}{\plangengreen[#1]{RRTI}{*}{C}}
\NewDocumentCommand{\rrtiQC}{O{}}{\plangengreen[#1]{RRTI}{Q}{C}}
\NewDocumentCommand{\rrtiS}{O{}}{\plangengreen[#1]{RRTI}{*}{$\emptyset$}}
\NewDocumentCommand{\rrtiQ}{O{}}{\plangengreen[#1]{RRTI}{Q}{$\emptyset$}}

\NewDocumentCommand{\rrteO}{O{}}{\plangengreen[#1]{RRTE}{}{$\emptyset$}} 
\NewDocumentCommand{\rrteC}{O{}}{\plangengreen[#1]{RRTE}{}{C}}
\NewDocumentCommand{\rrteSC}{O{}}{\plangengreen[#1]{RRTE}{*}{C}}
\NewDocumentCommand{\rrteQC}{O{}}{\plangengreen[#1]{RRTE}{Q}{C}}
\NewDocumentCommand{\rrteS}{O{}}{\plangengreen[#1]{RRTE}{*}{$\emptyset$}}
\NewDocumentCommand{\rrteQ}{O{}}{\plangengreen[#1]{RRTE}{Q}{$\emptyset$}}

\NewDocumentCommand\harO{O{}}{\plangennew[#1]{HARG}{}{$\emptyset$}}
\NewDocumentCommand\harC{O{}}{\plangennew[#1]{HARG}{}{C}}
\NewDocumentCommand\harS{O{}}{\plangennew[#1]{HARG}{*}{$\emptyset$}}
\NewDocumentCommand\harQ{O{}}{\plangennew[#1]{HARG}{Q}{$\emptyset$}}
\NewDocumentCommand\harSC{O{}}{\plangennew[#1]{HARG}{*}{C}}
\NewDocumentCommand\harQC{O{}}{\plangennew[#1]{HARG}{Q}{C}}

\NewDocumentCommand\harlO{O{}}{\plangennew[#1]{HARL}{}{$\emptyset$}}
\NewDocumentCommand\harlC{O{}}{\plangennew[#1]{HARL}{}{C}}
\NewDocumentCommand\harlS{O{}}{\plangennew[#1]{HARL}{*}{$\emptyset$}}
\NewDocumentCommand\harlQ{O{}}{\plangennew[#1]{HARL}{Q}{$\emptyset$}}
\NewDocumentCommand\harlSC{O{}}{\plangennew[#1]{HARL}{*}{C}}
\NewDocumentCommand\harlQC{O{}}{\plangennew[#1]{HARL}{Q}{C}}

\NewDocumentCommand\harkO{O{}}{\plangennew[#1]{HARk}{}{$\emptyset$}}
\NewDocumentCommand\harkC{O{}}{\plangennew[#1]{HARk}{}{C}}
\NewDocumentCommand\harkS{O{}}{\plangennew[#1]{HARk}{*}{$\emptyset$}}
\NewDocumentCommand\harkQ{O{}}{\plangennew[#1]{HARk}{Q}{$\emptyset$}}
\NewDocumentCommand\harkSC{O{}}{\plangennew[#1]{HARk}{*}{C}}
\NewDocumentCommand\harkQC{O{}}{\plangennew[#1]{HARk}{Q}{C}}

\NewDocumentCommand\harfO{O{}}{\plangennew[#1]{HARF}{}{$\emptyset$}}
\NewDocumentCommand\harfC{O{}}{\plangennew[#1]{HARF}{}{C}}
\NewDocumentCommand\harfS{O{}}{\plangennew[#1]{HARF}{*}{$\emptyset$}}
\NewDocumentCommand\harfQ{O{}}{\plangennew[#1]{HARF}{Q}{$\emptyset$}}
\NewDocumentCommand\harfSC{O{}}{\plangennew[#1]{HARF}{*}{C}}
\NewDocumentCommand\harfQC{O{}}{\plangennew[#1]{HARF}{Q}{C}}

\NewDocumentCommand\trajball{O{i} g}{
X_{#1}
\IfNoValueTF{#2}{}{\bigl(#2\bigr)}
}

\newcommand{\inbluetext}[1]{\bblue{\codecx{#1}}\xspace}
\newcommand\HAR[1][]{\bblue{\codecx{HARG}#1}\xspace}
\renewcommand\RRT[1][]{\bblue{\codecx{RRT}#1}\xspace}
\newcommand\har{\HAR}
\newcommand\harg{\HAR}
\newcommand\harf{\bblue{\codecx{HARF}}\xspace}
\newcommand\harl{\bblue{\codecx{HARL}}\xspace}
\renewcommand\rrt{\RRT}

\newcommand{\harstar}{\inbluetext{HARG${}^*$}}

\newcommand{\rrtstar}{\inbluetext{RRT${}^*$}}
\newcommand{\harconnect}{\inbluetext{HARG-C}}
\newcommand{\harlconnect}{\inbluetext{HARL-C}}
\newcommand{\rrtconnect}{\inbluetext{RRT-C}}

\newcommand{\quickrrtstar}{\inbluetext{RRT${}^Q$}}
\newcommand{\harstarconnect}{\inbluetext{HARG${}^*$-C}}

\newcommand{\rrtstarconnect}{\inbluetext{RRT${}^*$-C}}

\newcommand{\quickharlstarconnect}{\inbluetext{HARL${}^Q$-C}}
\newcommand{\quickrrtstarconnect}{\inbluetext{RRT${}^Q$-C}}

\newcommand{\harfconnect}{\inbluetext{HARF-C}}
\newcommand{\quickharfstar}{\inbluetext{HARF${}^Q$}}

\newcommand{\quickharfstarconnect}{\inbluetext{HARF${}^Q$-C}}

\newcommand{\pr}[1]{\,{\scriptsize (#1)}}

\newcommand{\distcol}{d_\text{col.}}
\newcommand{\timecol}{t_\text{col.}}

\NewDocumentCommand\SteerRRT{g}{\bblue{\codecx{Steer\_RRT}}\IfNoValueTF{#1}{}{(#1)}}

\NewDocumentCommand\rootdist{g}{\bblue{\codecx{Dist\_to\_root}}\IfNoValueTF{#1}{}{(#1)}}
\NewDocumentCommand\CollisionTime{g}{\bblue{\text{Collision\_time\_D}}\IfNoValueTF{#1}{}{(#1)}}

\NewDocumentCommand\colltimeD{g} {t^c  \IfNoValueTF{#1}{} {(#1)}}

\NewDocumentCommand\bestparent{g}{\bblue{\codecx{Best\_parent}}\IfNoValueTF{#1}{}{(#1)}}
\NewDocumentCommand\bestparentRed{g}{\rred{\codecx{Best\_parent}}\IfNoValueTF{#1}{}{(#1)}}

\NewDocumentCommand\rewire{g}{\bblue{\codecx{Rewire}}\IfNoValueTF{#1}{}{(#1)}}
\NewDocumentCommand\rewireRed{g}{\rred{\codecx{Rewire}}\IfNoValueTF{#1}{}{(#1)}}

\NewDocumentCommand\stepkplus{G{k}}{D^{+}_{#1}} 
\NewDocumentCommand\stepkminus{G{k}}{D^{-}_{#1}} 

\NewDocumentCommand\HCA{g} {\bblue{\codecx{HCA}}  \IfNoValueTF{#1}{} {(#1)}}
\NewDocumentCommand\CCD{g} {\bblue{\codecx{CCD}}  \IfNoValueTF{#1}{} {(#1)}}

\newcommand{\urrt}{\hat{u}_\textrm{RRT}}

\newcommand{\headerCubicles}{
\begin{qcolgray}
\textbf{Cubicles} \enspace {\small Box $405\times200\times200$  \;|\;
$\dist(\qstart,\qgoal)=205$ \;|\;
$\delta_\text{RRT}=70$, $\delta_\text{HARF}=80$, $\delta_\text{HARL}=\delta_\text{HARG}=110$ \;|\; 
$1\,\text{min}$ runs}
\end{qcolgray}
}
\newcommand{\headerTDSmall}{
\begin{qcolgray}
\textbf{\twoDimModel{small-obstacle}} \enspace {\small Box $100\times100$ \;|\;
$\dist(\qstart,\qgoal)=45.5$ \;|\;
$\delta_\text{RRT}=\delta_\text{HARF}=\delta_\text{HARL}=\delta_\text{HARG}=10$ \;|\; 
$1\,\text{min}$ runs}
\end{qcolgray}
}
\newcommand{\headerTDLarge}{
\begin{qcolgray}
\textbf{\twoDimModel{big-obstacle}} \enspace {\small Box $100\times100$ \;|\;
$\dist(\qstart,\qgoal)=45.5$ \;|\;
$\delta_\text{RRT}=12$, $\delta_\text{HARF}=\delta_\text{HARL}=\delta_\text{HARG}=7$ \;|\; 
$1\,\text{min}$ runs}
\end{qcolgray}
}

\newcommand{\headerCradleNine}{
\begin{qcolgray}
\cradleModel{9} \enspace {\small Box $228\times228\times220$ \;|\;
$\dist(\qstart,\qgoal)=30.0$ \;|\; 
$\delta_\text{RRT}=\delta_\text{HARF}=\delta_\text{HARL}=\delta_\text{HARG}=15$ \;|\; 
$1\,\text{min}$ runs}
\end{qcolgray}
}
\newcommand{\headerCradleTFcoarse}{
\begin{qcolgray}
\cradleModel{25}{coarse} \enspace {\small \small Box $280\times280\times220$ \;|\;
$\dist(\qstart,\qgoal)=60.3$ \;|\; 
$\delta_\text{HARG} = 25$, $\delta_\text{RRT} = \delta_\text{HARF}=35$, $\delta_\text{HARL}=40$ \;|\; 
$1\,\text{min}$ runs}
\end{qcolgray}
}
\newcommand{\headerCradleTFtight}{
\begin{qcolgray}
\cradleModel{25}{tight} \enspace {\small Box $236\times236\times220$ \;|\; 
$\dist(\qstart,\qgoal)=35.4$ \;|\; 
$\delta_\text{HARL}=\delta_\text{HARG}=25$ \;|\; $5\,\text{min}$ runs}
\end{qcolgray}
}
\newcommand{\headerCradleTFtightSparse}{
\begin{qcolgray}
\cradleModel{25}{tight} \enspace {\small Box $236\times236\times220$ \;|\; 
$\dist(\qstart,\qgoal)=35.4$ \;|\; 
$\delta_{RRT}=10, \delta_\text{HARF}=\delta_\text{HARL}=\delta_\text{HARG}=25$  \;|\; $1\,\text{min}$ runs}
\end{qcolgray}
}

\newcommand{\headerCradleSF}{
\begin{qcolgray}
\cradleModel{64} \enspace {\small Box $224\times224\times224$ \;|\;
$\dist(\qstart,\qgoal)=72.7$ \;|\;
$\delta_\text{HARF}=45, \delta_\text{RRT}=50$ \;|\; $5\,\text{min}$ runs}
\end{qcolgray}
}

\newif\ifSEPARATECAPTIONS
\SEPARATECAPTIONSfalse 

\title{Motion planning in high dimensional spaces hybridizing RRT and HAR via position-direction decoupling}

\author{\frederic Cazals\thanks{\ucainria; email: Frederic.Cazals@inria.fr}~ and Nelson Feyeux\thanks{\ucainria. Email: Nelson.a.Feyeux@inria.fr}}

\newcommand{\wdir}{./}
\begin{document}
\maketitle

\begin{abstract}
The exploration of high-dimensional spaces remains a challenging problem,
in particular in the presence of narrow passages and small clearances.
We propose novel sampling-based path-planning methods for
high-dimensional spaces combining Rapidly-exploring Random Trees (RRT)
and Hit-and-Run (HAR) random walks by decoupling the point being
extended from the direction of extension.
We also show that RRT and HAR appear as special cases of a generic
algorithm coupling the biases used for the point and direction
extension, respectively.
We further study a sparse-move strategy in
which only a fraction $p_r$ of the robots is moved at each step,
helping both RRT and the proposed HAR algorithms handle cluttered
instances.
Tests are presented for two families of models: classical {\em piano
mover} problems in 3D, and complex molecular systems involving tens of rigid
domains moving relatively to one another -- the latter viewed as
independent robots exploring the motion space $\SEn[N]$.
Within seconds on a standard laptop, our algorithms solve instances with
up to 64 robots and 384 degrees of freedom.
We conclude by suggesting one of our methods, HARF, as the method of
choice for complex multi-robot planning problems, being up to two
orders of magnitude faster than the classical RRT moving all robots at
each step--when it succeeds at all--, and still up to $2.4$ fold faster
on most instances when both use their best $p_r$.
\end{abstract}

\ifIEEE


\else

\fi

\section{Introduction}

\subsection{Previous work}

\paragraph{Path planning, exact versus randomized methods.}
Path planning is concerned with the problem of planning robot
admissible moves to perform prescribed tasks, in natural or designed
environments. Given the configuration space $\calC$ of the robot(s),
the objective  is to compute one or several paths in the {\em free space}, 
\ie the subset of $\calC$ that is free of obstacles and collisions.
From an algorithmic perspective, it is useful to distinguish between
different classes of motion planning algorithms and their associated
guarantees~\cite{canny1988complexity,latombe1991robot,lavalle2006planning}. At
a high level, when robots and their environments (i.e., obstacles) are
modeled as geometric entities, one can differentiate between exact and
randomized methods. Exact approaches provide certified, complete
solutions but typically suffer from poor scalability. In contrast,
randomized methods rely on sampling techniques and generally offer
only asymptotic guarantees; nevertheless, they often perform
remarkably well in practice. Although a comprehensive survey of these
approaches is beyond the scope of this work, several of their key
components are directly relevant to our study.

In the realm of randomized algorithms, two landmarks are Probabilistic
Roadmaps (PRM) \cite{kavraki1994randomized,overmars1994probablisitic}
and Rapidly-exploring Random
Trees (RRT)~\cite{lavalle1998rapidly,kuffner2000rrt} -- see also
Section \ref{sec:rrt} for a review of recent developments.
The fundamental idea of PRM is to build a graph in the free space, and
to use it to perform queries between  pairs of start and goal points. The difficulty is
naturally to provide a comprehensive coverage of the free space, discovering in
particular narrow passages and unexplored regions.  
In the spirit of {\em tabu search}, the rapidly-exploring random trees
(RRT) method uses the {\em Voronoi bias} to {\em extend} as a priority
nodes having few neighbors and large Voronoi regions.
While maintaining the Voronoi diagram in a high dimensional
space is practically infeasible, RRT solely requires a dynamic data
structure for nearest neighbor queries.  The node extension requires
the knowledge of all the samples already produced, so that the 
exploration  algorithm is not a Markov chain.

In the realm of exact algorithms, of particular interest for our work
are visibility-based methods.  Given a region containing convex
obstacles, the visibility complex decomposes free rays (\ie rays lying
in the complement of obstacles) into a set of
cells~\cite{pocchiola1993visibility}.
Free rays are of importance in the context of planning as they encode
{\em far-reaching} rather than local connections.  While the
maintenance of such partitions is also beyond reach in high
dimensional spaces, difficulties have been partially overcome by
combining visibility maps and PRM~\cite{simeon2000visibility}.

In our work, visibility is actually connected to
Hit-and-Run random walks, which we review now.

\paragraph{Random walks and Markov chains to compute the volume of convex bodies.}
There is a remarkable connection between algorithms used in robotics
and those used to compute volumes in high dimensional spaces, and the volume
of polytopes in particular. 
While there is no fast deterministic algorithm to compute the volume
of a convex body reliably~\cite{bollobas1997volume}, randomized
algorithms based on Markov chains have been developed and gradually optimized.
A key step in this realm has been the invention of the Hit-and-Run
(HAR) algorithm to identify redundant constraints in linear
programs~\cite{berbee1987hit}. As the name suggests, given a point in
a polytope, one samples a random direction, finds the nearest polytope
hyperplane in that direction, generates a sample on that segment, and iterates.
This generative process defines a Markov chain whose mixing time
characterizes the speed at which the samples approximate a target
distribution, typically a log-concave distribution, and the samples
generated can be used to estimate volumes
~\cite{kannan1997random,lovasz1999hit,cousins2016practical}.
Other random walks, sampling points in a ball~\cite{lovasz1999faster}
or using jump processes have also been studied~\cite{chevallier2022efficient}.
Finally, we also note that HAR has been used to sample and explore non-convex
spaces~\cite{abbasi2017hit}.

\paragraph{Applications in structural biology.} An important application of planning
algorithms is structural bioinformatics. The function of proteins
is determined by their structure and dynamics~\cite{dill2010molecular}.
While \alphafold provided definitive progress for the prediction of
static structures~\cite{jumper2021highly}, understanding the dynamics
remains the overarching goal, in two settings. The first is to
characterize subtle entropic phenomena as dictated by statistical
physics~\cite{schmidt2013preconfiguration}, to delineate
thermodynamics and kinetics~\cite{lelievre2010free}.
The second is to chart the conformational space of a molecule, ideally
bridging the gap to free-energy calculations. Algorithms using the
aforementioned random walks proved highly effective in this
context. One may cite the combination of RRT with a Metropolis-like test
to explore (potential) energy landscapes~\cite{jaillet2011randomized},
and the application of HAR in suitable angular spaces to sample
conformations of long backbone loops~\cite{odonnell2023enhanced}.

In fact, the current work is motivated by the exploration of the
relative motions of rigid protein domains, under the guidance of flexible
linkers~\cite{simsir2021studying}. We call such models {\em
  molecular cradles}, in reference to Newton's cradles. Phrased
differently, planning a cradle is akin to planning the relative motion
of rigid robots, each enjoying six degrees of freedom. While physical
models have been proposed for a single
protein~\cite{koehl2024minactionpath2}, we are not aware of any
attempt to collectively move a large cohort of rigid domains.
In this work, we forget about linkers whose conformations can be
sampled by combining HAR and loop closure techniques -- see
\cite{odonnell2023geometric,odonnell2023enhanced} and references
therein. Abusing terminology, we simply refer to the domains as
individual molecules.

\paragraph{Multi-robot path planning.} Planning the collective motion of protein domains is an instance of multi‑robot path planning.
Multi-robot planning is notoriously
challenging~\cite{erdmann1987multiple,van2009centralized} and three
main families of approaches have been developed, namely decoupled 
techniques, centralized approaches, and sampling-based approaches, see
\cite{shome2020drrt} and references therein.
A common paradigm consists in first planning robot motions
independently using individual roadmaps, and subsequently handling
inter-robot collisions through search in the Cartesian product of the
resulting roadmaps~\cite{wagner2011m,solovey2016finding}.
A specific order based on prioritization may be used~\cite{van2005prioritized}.
The MAPF (multi-agent path finding) and especially CBS (conflict-based
search) \cite{barer2014suboptimal} are different graph algorithms for
multi-robot path planning, also constructing independent
individual paths prior to considering different branching scenarios to
handle collisions.
In general, the combinatorial complexity of the product space is
prohibitive, which motivated the exploitation of an implicit encoding of
its structure~\cite{shome2020drrt}.
 The complexity of the initial space also motivated the design of
 dimensionality-reduction-based planners.  In a nutshell, the recent
 {\em fiber bundle}-based approach stratifies the search space, solves
a simpler planning problem, and tackles the full problem by
 searching near the planned low-dimensional path~\cite{orthey2024multilevel}.
We also note that generative AI methods have recently been proposed 
for planning~\cite{shaoul2025multi}.

\subsection{Contributions}

As illustrated by recent previous work, there are two
orthogonal and actually complementary ways to tackle complex planning
problems: by exploiting decompositions of the configuration
space--what the fiber bundle approach does, or by designing more
powerful planners able to cope with the full space.

Our work falls into the latter category and exploits a fundamental
difference in how RRT and HAR work: HAR defines a Markov chain, whereas RRT
exploits the whole sampling history through its {\em Voronoi bias}.
Our starting point is that the two methods are complementary: RRT
determines which node to extend, while HAR determines the direction in
which to extend it.  Decoupling these two decisions yields a single
family of planners.

\begin{itemize}
\item {\bf An RRT algorithm using HAR for the extension step.}
We construct a family of RRT algorithms, denoted HARG
  (for global), HARL (for local) and HARF (for fixed distance
  steering), that differ in the choice of the new point $\qnew$ 
  added to the tree at each step, and allow more flexibility in regions with low
  clearance.
Moreover, RRT and HAR can be viewed as special cases of a more generic
algorithm coupling the positional and directional samplings via a von
Mises--Fisher distribution (Section~\ref{sec:interpol}). 

\item {\bf Joint multi-robot planning without product structure.} We
  instantiate the framework on $\SEn$ and $\SEn[N]$ and plan $N$
  robots \emph{jointly} in the product configuration space,
  without building or searching any explicit product of individual
  roadmaps. We introduce two modes termed {\em mass start} and {\em rolling start}
tuning the concomitance of robots moves.
On multi-robot problems, our methods solve instances with up to 64
robots (384 degrees of freedom) within seconds on a standard laptop.
\end{itemize}

\subsection{Notations}

The bounding box of $\Rd{2}$ or $\Rd{3}$ containing a model is denoted $\calR$.
The unit sphere in dimension $d$ is denoted $\Sd{d-1}$, and a scaled version of it, that
is, a sphere of radius $\rho$, is denoted $\rho\ \Sd{d-1}$.
The special Euclidean (resp. rotation) group is denoted $\SEn{3}$ (resp. $\SOn{3}$),
and the associated Lie algebras are denoted $\SEnlie{3}$ (resp. $\SOnlie{3}$).

\section{Motion planning with Hit-and-Run sampling}

\subsection{General setup}
\label{sec:generalsetup}

We consider sampling-based planners that require simple operations on the configuration space $\calC$:
\begin{enumeratep}
\item sampling configurations;
\item ranking neighbors of a sample using  a distance metric;
\item defining a continuous path between two samples and a function returning intermediate configurations;
\item checking collisions on a continuous path.
\end{enumeratep}

\paragraph{Free space and segments.}
The {\em free space} $\Cfree\subset \calC$ corresponds to
configurations where the robots do not self-intersect nor intersect
obstacles.  For practical reasons, physical boundaries are added to
$\Cfree$ to restrain the space in which the robots move.  
For a configuration $q\in \Cfree$, the {\em clearance} is
defined as the minimum distance from $q$ to the boundary of $\Cfree$.

Without any further assumption, we connect two points $q_0$ and $q_1$
in $\calC$ by the {\em curved segment}, denoted
$[q_0, q_1]$, connecting them. We may see this segment as a
parameterized curve $q(t)$ with $t\in[0, 1]$, with $q(0)=q_0$ and
$q(1)=q_1$. Naturally, the practical analytical expression of $q(t)$ depends on
the nature of $\calC$ and the interpolation scheme selected.

When considering a set of $N$ independent robots, the joint
configuration space is $\calC = \calC_1 \times \ldots \times \calC_N$.

\paragraph{Collision detection.}
Robots and obstacles are generally described by a collection of simple
primitives, triangles or balls for instance.
A {\em collision predicate} is a function stating whether $[q_0, q_1]$
is collision-free.  Such predicates are evaluated using collision
detection algorithms \cite{schwarzer2005adaptive, tang2009CA} which
decompose the $[0, 1]$ time interval into smaller subintervals until
they are all proven to be collision-free, or until a collision is
found.
In this work, such predicates are based on distance calculations, 
and two primitives are declared in collision if their distance is less than
a small threshold $\distcol$ (\eg $\sim1\%$ of the robot size).

The common
method to speed-up the collision detection is to represent the bodies as
a Bounding Volume Hierarchy \cite{linmanocha2004collision}. For multiple robots,
collision detection can be performed efficiently by prioritizing the pairs
of nodes in the different Bounding Volume Hierarchies that are more
likely to collide \cite{tang2013HCA}.

\paragraph{Planning problems.}
Consider initial and final configurations $q_0$ and $\qinfty$ in
$\Cfree$. A {\em planning} problem consists in finding a path
connecting $q_0$ and $\qinfty$ in $\Cfree$. A valid path is typically
defined by a sequence of points $( q_0, q_1, \dots, \qinfty )$, such
that each curved segment $[\qij{i}, \qij{i+1}]$ is collision-free.

The existence of paths raises two algorithmic problems of interest in this work:
\begin{problem}[Geometric path]
\label{pb:geometric-path}
Find any path connecting $q_0$ and $\qinfty$.  
\end{problem}

\begin{problem}[Shortest geometric path]
\label{pb:shortest-geometric-path}
Find a/the shortest path connecting $q_0$ and $\qinfty$.  
\end{problem}

\subsection{The RRT planning algorithm}
\label{sec:rrt}

The Rapidly exploring Random Tree (RRT) algorithm
~\cite{lavalle1998rapidly,kuffner2000rrt} connects two points $q_0$
and $\qgoal$ in $\Cfree$ by creating a tree such that (i) each node
is a random sample in $\Cfree$, and (ii) each edge is a collision-free
path.

At every iteration, a new random point $\qsampled$ in $\calC$
is passed to the {\bf steering procedure} defining $\qnew$:
if $\qsampled \in B(\qfrom, \delta)$, $\qnew$ is $\qsampled$; else
$\qnew$ is the projection of $\qsampled$ onto the sphere $S(\qfrom, \delta)$. 
Then if the curved segment $[\qfrom, \qnew]$ is
collision-free, the node $\qnew$ and the edge $[\qfrom, \qnew]$ are added to the tree.
The iteration is repeated until a path is found, or it is continued to
refine the already found path.

Many heuristics exist and can be used to improve the algorithm. The
{\em Connect} heuristic \cite{lavalle1998rapidly,kuffner2000rrt},
denoted \rrtconnect, uses two trees anchored at $\qstart$ and
$\qgoal$ to connect faster.  The {\em Star} and {\em Quick}
algorithms, denoted \rrtstar and \quickrrtstar, rewire the
tree~\cite{karaman2011anytime,jeong2019quick} using two procedures
called \procBestParent and \procRewire. This rewiring shortens
paths. The {\em Star-Connect} and {\em Quick-Connect} methods, denoted
\rrtstarconnect and \quickrrtstarconnect, combine the previous
ones~\cite{jordan2013optimal} by rewiring the two trees.  Finally, the
{\em Informed} heuristic~\cite{gammell2014informed} prunes samples
that are too far to improve the current found path: if $L$ is the
length of the currently found path, it discards any sample not in
$\calE_L=\{q\in\calC \colon \distCS{q, \qstart} + \distCS{q, \qgoal}
\le L \}$.

\paragraph{Path shortening.} After the algorithm is stopped, the path found (if any)
is optimized by repeatedly running the following two methods within a fixed wall-clock time budget.
The first uses simple short-cutting: 
two points are drawn at random on the path
and the sub-path is replaced by the direct connection if collision-free. 
The second is the more effective procedure \textit{partial shortcutting}~\cite{geraerts2007creating} which
shortcuts the path of one random robot at a time.

These procedures greatly reduce the total length of the path while conserving a valid path.

\subsection{Decoupling positions and directions in RRT with HAR}
\label{sec:har-local-global}
 
RRT and HAR differ in two fundamental ways.  First, while HAR defines a
Markov chain, RRT does not since the entire history of the sampling
process is used.
Second, RRT places the focus on the point being extended -- the
direction of the extension being fixed upon identifying $\qfrom$,
while HAR uses randomness on directions.

It is therefore natural to hybridize these two methods by leaving the
choice of the point extended to RRT, and that of the direction to HAR.
In doing so, a rapidly-exploring random tree is built as follows:
\begin{itemize}
\item (RRT feature) The Voronoi bias is used to select the point extended
while promoting unexplored regions. 

\item (HAR feature) The direction $\theta$ of the candidate new node is chosen at random, which
offers three {\em steering strategies}:
\begin{itemize}
\item Globally (HARG): the segment stemming from $\qfrom$ in the direction
  $\theta$ is traced until the first collision point, or until it
  leaves the domain. The new node $\qnew$ is randomly drawn on this segment.

\item Locally (HARL): the new node $\qnew$ is chosen randomly on the segment of
  length $\delta$ stemming from $\qfrom$, in the direction $\theta$.
  A boundary check and a collision check are performed afterwards to
  discard the new node in case of a positive result.
  
\item Fixed (HARF): the new node $\qnew$ is chosen at distance exactly
  $\delta$ of $\qfrom$ in the random direction $\theta$.  A boundary check
  and a collision check are performed afterwards to discard the new
  node in case of a positive result.
\end{itemize}
\end{itemize}
We generically denote the resulting algorithms HARG, HARL and HARF
(Algorithm \ref{alg:RRT-variants}).
Beyond these steering procedures, the algorithms match RRT.
After a new sample $\qnew$ is correctly drawn, it is stored in a tree
which is used to connect $\qstart$ and $\qgoal$.

\begin{algorithm*}[htb]
\begin{minipage}[t]{.5\linewidth}\vspace{0pt}
\begin{algorithmic}[1]
\Require Input: $\qstart, \qgoal \in \Cfree$
\Require Hyper-parameters: $\delta>0$, $\kappa\ge 0, p_r \in (0,1]$
\Procedure{\codecx{RRT}\,/\,\codecx{HARL}\,/\,\codecx{HARF}\,/\,\codecx{HARG}}{}
\State $\calT.\text{Init}(q_0)$
\While{True}
\State $\qsampled \leftarrow \textsc{randomConfig()}$
\State $\qfrom \leftarrow \calT.\text{nearest}(\qsampled)$
\State $\qnew \leftarrow \langle\text{variant-specific step, see right}\rangle$ \label{line:variant-step}

\State //Rolling start mode: immobilization of a robot
\For{each robot $i$}
\State With probability $1-p_r$: $\qnew^i \leftarrow \qfrom^i$
\EndFor
\If{$[\qfrom,\qnew]$ not in $\Cfree$}
\State \textbf{continue}
\EndIf
\State $\calT.\text{addVertex}(\qnew)$
\State $\calT.\text{addEdge}(\qfrom \rightarrow \qnew)$
\If{$\textsc{CollisionFree}(\qnew,\qgoal)$}
\State $\calT.\text{addEdge}(\qnew \rightarrow \qgoal)$
\EndIf
\EndWhile
\EndProcedure
\end{algorithmic}
\end{minipage}\hfill
\begin{minipage}[t]{.5\linewidth}\vspace{0pt}
\rrt, \harg, \harl and \harf match the template algorithm on the left and
differ only in line~\ref{line:variant-step}, the sampling of $\qnew$,
which becomes, with $\urrt$ the direction from $\qfrom$ to $\qsampled$:

\smallskip
\rrt: \begin{algorithmic}[1]
\State $\qnew \leftarrow \textsc{Steer}(\qfrom, \qsampled, \delta)$
\end{algorithmic}

\harg: \begin{algorithmic}[1]
\State $\theta \sim \calU(\Sd{d-1})$ if $\kappa=0$ else $\theta \sim \vonMF[\urrt, \kappa]$
\State $\timecol \leftarrow \textsc{TimeOfCollision}(\qfrom, \theta)$
\State $\qnew \sim \calU([\qfrom, \qfrom + \timecol\,\theta])$
\end{algorithmic}

\harl: \begin{algorithmic}[1]
\State $\theta \sim \calU(\Sd{d-1})$ if $\kappa=0$ else $\theta \sim \vonMF[\urrt,\kappa]$
\State $\qnew \sim \calU([\qfrom, \qfrom + \delta \theta])$
\end{algorithmic}

\harf: \begin{algorithmic}[1]
\State $\theta \sim \calU(\Sd{d-1})$ if $\kappa=0$ else $\theta \sim\vonMF[\urrt,\kappa]$
\State $\qnew \leftarrow \qfrom + \delta \theta$
\end{algorithmic}

\end{minipage}
\caption{{\bf Generic sampler yielding \rrt, \harg, \harl and \harf.}
The four algorithms differ in the way $\qnew$ is sampled (line~\ref{line:variant-step}):
\rrt steers a far-away sample onto $S(\qfrom,\delta)$;
\harg draws uniformly on the ray up to the first collision point;
\harl draws uniformly on a ray of length $\delta$;
\harf jumps a fixed step $\delta$ along a random direction.
$\vonMF$ refers to a von Mises-Fisher distribution on the unit sphere $\Sd$.
The vector $\urrt$ is the normalized vector $\qfrom$ to $\qsampled$. 
}
\label{alg:RRT-variants}
\end{algorithm*}

\paragraph{Decoupling the position and direction helps near obstacles.}
The benefits of position-direction decoupling 
are best understood by inspecting the situation in low clearance regions
 (Fig.~\ref{fig:steering-procedures}).
RRT pulls a far-away sample onto the sphere $S(\qfrom,\delta)$; only
the portion of that sphere lying in \Cfree is productive, so that as
soon as $\delta$ exceeds the local clearance most samples are
rejected.  In particular, in high dimensional spaces, mass
concentration phenomena~\cite{blum2020foundationsdatascience} are
  such that $\qsampled$ is far away from $\qfrom$, so that in the presence of obstacles,
steering steps are likely to be unproductive.
HARG instead traces a maximal-clearance ray along a random
direction: \emph{every} direction yields a productive segment, and
$\qnew$ may land arbitrarily far from $\qfrom$. HARL restricts that
ray to length $\delta$: its productive region coincides with that of
RRT, but a collision-free $\qnew$ is more likely because the whole
segment and not only its endpoint (as in RRT) is available. HARF places $\qnew$
at the fixed distance $\delta$, recovering RRT's step length while
keeping HAR's random direction.
The challenge posed by collisions is further exacerbated by the
dimensionality of the problem, since a sampled configuration of $N$
robots is considered to be in collision even if a single pair of
robots collides.

As shown in Experiments and consistent with the previous description, 
the HAR-derived methods manage to sample the free configuration space 
with small clearance more easily than the RRT algorithms.

\figBegins
\begin{center}
\begin{tabular}{cc}
\includegraphics[width=0.33\textwidth]{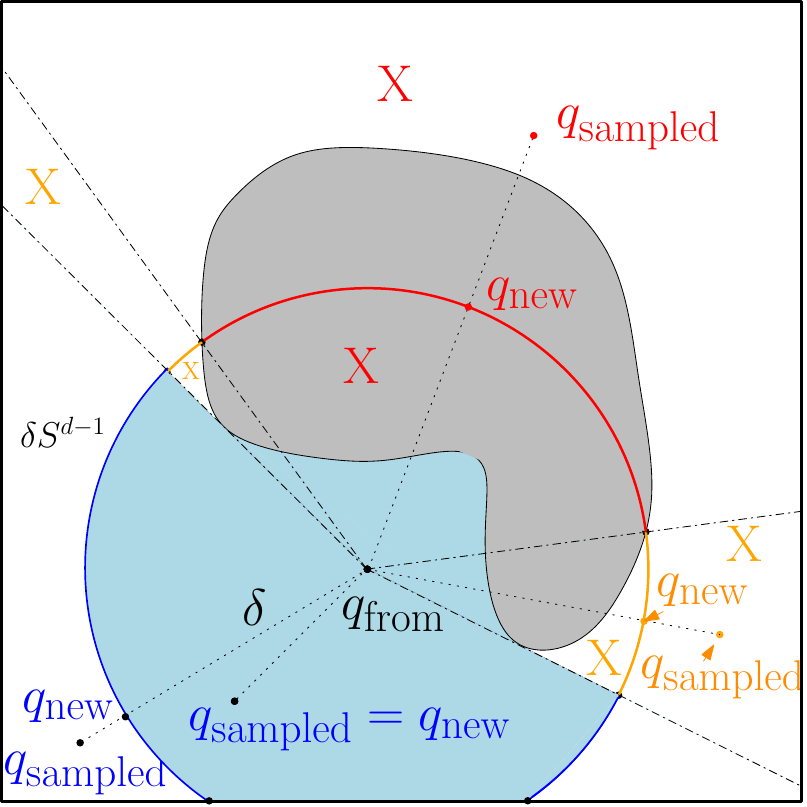} &
\includegraphics[width=0.33\textwidth]{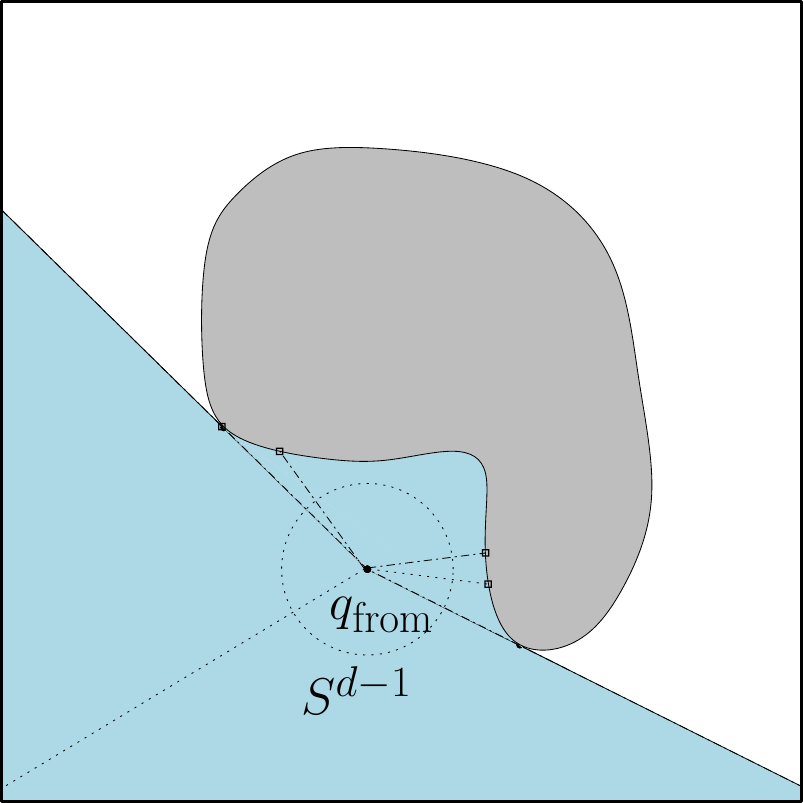}\\
RRT & Hit-and-Run global (HARG)\\
\includegraphics[width=0.33\textwidth]{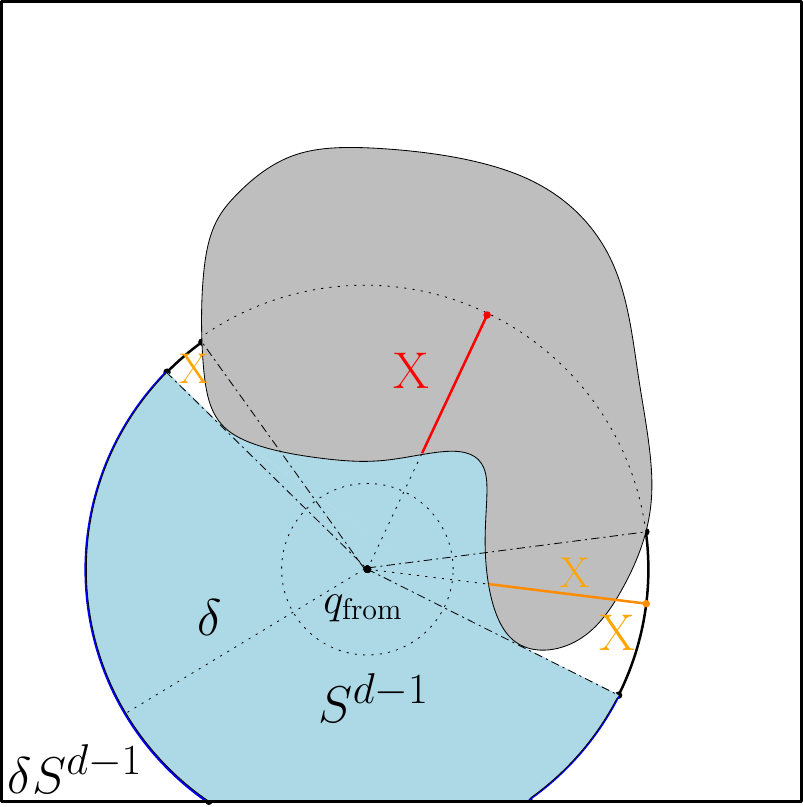}&
\includegraphics[width=0.33\textwidth]{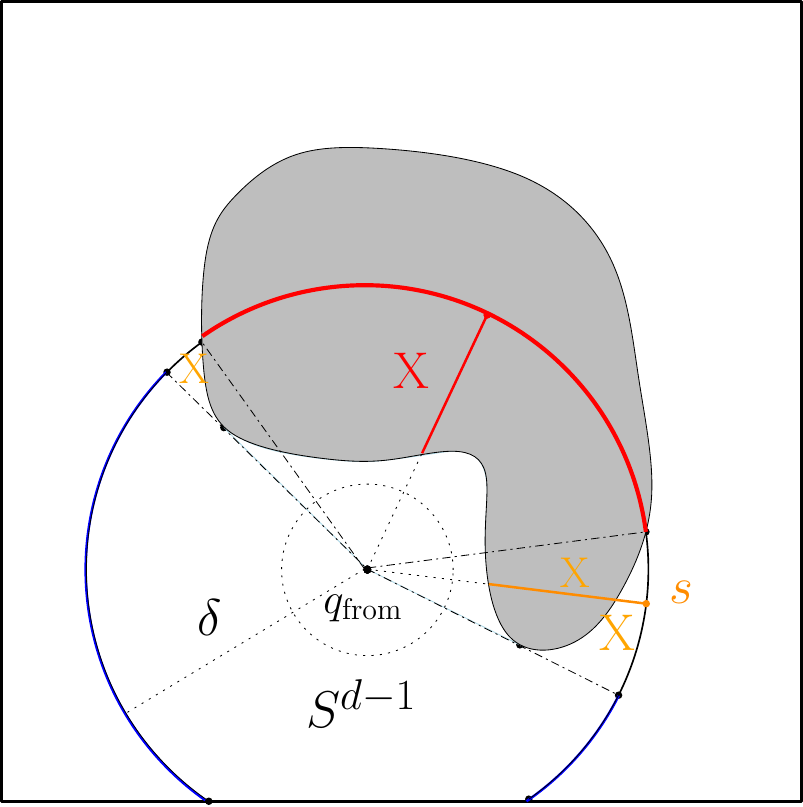}\\
Hit-and-Run local (HARL) & Hit-and-Run fixed (HARF) \\
\end{tabular}
\end{center}
\caption{
{\bf Steering procedures near obstacles: RRT vs HARG vs HARL vs HARF.}
{\bf (RRT)} Only those samples $\qsampled$ in the blue region
or projecting into the blue circle arcs (complementary of the orange and red ones)
are productive.
{\bf (Hit-and-Run global, HARG)} In any direction, the  ray centered at $\qfrom$ yields
a productive (blue) segment on which a productive sample $\qnew$ can be obtained.
{\bf (Hit-and-Run local, HARL)} 
The ray centered at $\qfrom$ of length $\delta$ on which $\qnew$ is sampled
also has a productive region identical to that of RRT. However, the likelihood to obtain
 a point in this region is higher than that of RRT.
{\bf (Hit-and-Run fixed, HARF)} 
The new node $\qnew$ is chosen at distance exactly
  $\delta$ of $\qfrom$ in the random direction $\theta$.
}
\label{fig:steering-procedures} 
\figEnds

\subsection{Recovering RRT and HAR strategies by biasing the extension direction}
\label{sec:interpol}

HAR draws a random extension direction while RRT forces
the move into the normalized direction $\urrt$ of $\qfrom$ to $\qsampled$.
These strategies have pros
and cons: if the space is clear, moving forward as in RRT is
beneficial; but in case of small clearance, moving backward might be
compulsory and the randomization in HAR helps.

Both strategies are actually recovered by biasing the choice of the
direction to be close to the vector $\urrt$.  This is achieved by
sampling the direction on the unit sphere $\Sd$ using
a von Mises-Fisher distribution $\vonMF[\urrt, \kappa]$ whose density
reads--with $\kappa\geq 0$ the concentration coefficient:
\begin{equation} 
\label{eq:vonMF}
\vonMF[\urrt, \kappa]{u}  \propto \exp(\kappa \, \latrans{u}\!\urrt).
\end{equation}
The larger $\kappa$, the more concentrated the distribution around
$\urrt$. For $\kappa=0$, one recovers a uniform distribution.

\medskip

Sampling the direction with the von Mises-Fisher distribution also
yields an interesting coupling between the positional and
directional biases: the former depends on the Voronoi bias, the
latter depends on the former via the concentration parameter.

\subsection{RRT improvements to HAR algorithms}

Since HAR samplers just modify the way $\qnew$ is sampled, the usual
improvements brought to RRT seen in Section~\ref{sec:rrt} also apply.  The \textit{Star} and
\textit{Quick} heuristics with the \procBestParent and \procRewire
procedures yield \rrtstar and \quickrrtstar, and they also work for
\harg, \harl and \harf, yielding for instance \harstar or
\quickharfstar.  The \textit{Connect} heuristic operating on two trees
instead of one also works, yielding \harconnect, \harlconnect and
\harfconnect.  Combining both the \textit{Connect} and
\textit{Star}/\textit{Quick} heuristics is possible, yielding for
instance \harstarconnect or \quickharlstarconnect.

Finally, the \textit{Informed} heuristic ensuring
that samples belong to the ellipsoidal region $\calE_L$ 
also applies.


\subsection{Coordinating the robots: the moving fraction $p_r$}
\label{sec:pr}

Planning $N$ robots over $\calC=\calC_1\times\dots\times\calC_N$ requires deciding, at each
extension, which robots actually move.  The difficulty of a planning
problem can be toned down by striking a balance between the {\em time
  trial} mode where robots move sequentially, and the {\em mass start}
mode where all robots move concomitantly.
To do so, we introduce a {\em rolling start} mode, where at each step,
a robot moves with a probability $p_r\in(0,1]$.  A non-moving robot becomes a static obstacle.  
We note that this adjustment is in the spirit of
dimensionality-reduction planners which aim at solving simpler
problems first~\cite{orthey2024multilevel, van2005prioritized}.

In all cases every tree edge is collision-free by construction, so the final tree
encodes valid joint paths. We study the impact of $p_r$ on the RRT and HAR
algorithms in the experiments.

\section{Instantiations: planning in $\SEn$ and $\SEn[N]$}
\label{sec:planning-SE}

We detail the primitives required to plan a system  whose conformational space is  a product of $\SEn{3}$.

\subsection{Rigid motions in $\SEn$}

Given a single rigid object, the {\em piano mover problem} consists in
finding a sequence of valid rigid $\SEn$ transformations
(translation/rotation) from a source location to a target location
while avoiding obstacles~\cite{canny1988complexity,latombe1991robot,lavalle2006planning}.

\paragraph{Configuration space.}
The configuration space is the space of rigid transformations $\calC = \SEn$ 
limited to the configurations confined in a bounding box $\calR$.

For the von Mises-Fisher bias in Eq. \ref{eq:vonMF}, we need an inner product in the tangent
space $\liesen$ of $\SEn{3}$.  Let's recall that the tangent space is
composed of couples $(\omega, v)$ where $\omega$ is the rotation
vector and $v$ is a velocity vector. 
We set 
\begin{equation*}
\dotp[\SEnlie]{(\omega_1, v_1)}{(\omega_2, v_2)} :=  v_1\cdot v_2 + r^2\omega_1\cdot\omega_2.
\end{equation*}

\paragraph{Distance.}
Let $\theta$ the angle of $q_1$ relative to $q_0$ and $r$ the radius
of a characteristic object.  We use the following distance~\cite{bregier2018defining}:
\begin{equation}\label{eq:distSE3linear}
\distCS{q_0, q_1} =  \sqrt{\|T_1-T_0\|^2 + r^2\theta^2}.
\end{equation}
In the sequel, the value $r$ is set to 1.

\paragraph{Continuous path between two samples.}
We connect two samples $q_0$ and $q_1$ by interpolating the
translation ($\mathbb R^3$) and rotation ($\SOn$) parts
independently. This geodesic will be used for planning algorithms to
link points. Besides, planning algorithms are looking for {\em
  collision-free} geodesics. The formulas for the geodesics are:
\begin{equation}
\label{eq:segmentlinear}
[q_0, q_1] = \left\{ q(t)=(R(t), T(t)) \mid t\in[0, 1] \right\},
\end{equation}
with
\begin{equation}
\begin{cases} 
R(t) &= \Exp_{\SOn}(t \Log_{\SOn}(R_1R_0^{-1})) R_0,\\
T(t) &= T_0 + t(T_1-T_0).
\end{cases}
\end{equation}

\begin{remark}
Alternatively, one may use the $L_1$ distance $\distCS{q_0, q_1} = \|T_1-T_0\| + r\theta$, as
in~\cite{kuffner2004effective,guo2026the-open-motion-planning-library2}.
\end{remark}

\paragraph{RRT and HAR algorithms: sampling $\qsampled$ in $\mathcal C$.}
A random transformation $\qsampled\in\mathcal C$ is needed to identify the point extended $\qfrom$. 
To do so, we simply draw $\qsampled = (R, T)$ with $R\sim \calU(\SOn)$ and $T\sim\calU(\calR)$.

\paragraph{HAR algorithms: sampling a random direction $\theta$.}
The different versions of HAR need a random direction $\theta =
(\omega, v) \in \liesen$ from $\qfrom$ to define the new point added
to the tree. Direction $\theta$ involves the rotation vector $\omega$
whose norm is uniform in $[0, \pi]$ and whose direction is uniform in
$\Sd{2}$, and the vector $v$ drawn uniformly at random in $\delta \Sd{2}$.
The radius $\delta$ balances the magnitudes of both parts: in the
direction $\theta$, the robot rotates by $\pi/2$ on average as it
moves a distance~$\delta$.

Then, from $\qfrom$ and in the direction $\theta$:
\begin{itemize}
\item the local HARL samples a new point $\qnew$ at a random distance between 0 and $\delta$;
\item the HARF samples $\qnew$ at a distance $\delta$;
\item the global HARG first extends the ray to the domain boundary or to a collision point and
then samples on this curved segment.
\end{itemize}

\subsection{Rigid motions in $\SEn[N]$}

\paragraph{Configuration space.}
Planning $N$ independent robots, each parameterized on $\SEn$, results
in a configuration space $\calC = \SEn[N]$.  For a point $q\in\calC$,
the $i$-th component $q^i$ represents the position and orientation of the
$i$-th robot.  

The inner product on the tangent spaces of $\SEn[N]$, composed of couples $(\omega^i, v^i)_{i\in \{1,\ldots,N\}}$, reads as
\begin{equation*}
\dotp[\SEnlie[N]]{ (\omega_1, v_1) } {(\omega_2, v_2)}  := \sum_i  v_1^i\cdot v_2^i + r^2\,\omega_1^i\cdot\omega_2^i.
\end{equation*}

\paragraph{Distance.}
The distance between $q_0$ and $q_1$ is defined as the quadratic average over all of their components:
\begin{equation}
\label{eq:distSE3N}
\distCS{q_0, q_1} = \sqrt{  \dfrac 1N \sum_{i=1}^N \|T_1^i - T_0^i \|^2 + r^2\,\theta_i^2}.
\end{equation}
with $\theta_i$ the angle between the two rotations $R_0^i$ and $R_1^i$.

\paragraph{Continuous path between two samples.}
The curved segment between $q_0\in \calC$ and $q_1\in\calC$ is computed
component-wise using the same formula as \eqref{eq:segmentlinear}.

\paragraph{RRT and HAR algorithms: sampling $\qsampled$ in $\calC$.}
RRT and HAR algorithms need to sample points and directions in $\SEn[N]$.

To construct $\qsampled$, the $\SEn$ procedure is applied
independently to each of the $N$ robots and for all $i$, the rotation
of $\qsampled^i$ is drawn from $\calU(\SOn)$ and its translation from
$\calU(\calR)$.  Note in particular that the $N$ translations are
drawn independently.

\paragraph{HAR algorithms: sampling a random direction $\theta$.}
As in the $\SEn$ case, the direction $\theta = ((\omega^1, v^1),
\ldots, (\omega^N, v^N))$ from $\qfrom$ must be randomly drawn.  All
the $\omega$ are independently drawn as in the $\SEn$ case;
the $(v^i)$ are drawn dependently
within the sphere $\rho\,\Sd{3N-1}$ in dimension $3N$, with $\rho$
specified in the following Lemma \ref{lem:R}--proof in Section~\ref{sec:sampling-coupled-transformations}:
\begin{lemma}
\label{lem:R}
Consider a random sample $(T_1,\dots,T_N) \sim \uniformD{\rho\ \Sd{3N-1}}$,
with $\rho = \sqrt{\frac{3\pi}{8}}\sqrt{N}\delta$. One has
\begin{equation}
\expX{ \vvnorm{T_i}} \to \delta \textrm{ when }N\to\infty  
\end{equation}
\end{lemma}

\section{Experiments}

\subsection{Setup}

\paragraph{Contenders.} We compare the classical RRT variants
to our algorithms HARG, HARL and HARF on several models. Our
assessment is based on the two problems discussed in the Introduction,
namely finding a path (Pb. \ref{pb:geometric-path}) and finding a
good/optimal path (Pb. \ref{pb:shortest-geometric-path}).
The {\em Connect} algorithms are usually a good choice for the former
problem, with the latter benefiting from {\em Quick} and {\em
  Informed} heuristics.  We discarded the {\em Star} variants which
consistently showed inferior performance in terms of final path
quality.

All algorithms depend on the steering range value
$\delta$, and all {\em Quick} variants depend on a local range -- we
use the $\krrt=50$ nearest neighbors. These choices yield the final
list of contenders (Table \ref{tab:contenders}).
\medskip

Each contender is run $N_r=50$ times on each model, with a fixed time limit. 
We also run path shortening methods for $T_s = 30$\,s.

\begin{table}[!ht] 
\begin{center}
\adjustbox{max width=\coeffonefigCOLW\columnwidth}{
\begin{tabular}{|l|l|}
\hline
Name & Explanation \\
\hline
\rrtC  & RRT using {\em Connect} heuristic (C)\\
\rrtQC & RRT using {\em Quick} (Q) and {\em Connect} heuristics\\
\hline
\harC & HARG algorithm using the {\em Connect} heuristic \\
\harQC & HARG using {\em Quick} and {\em Connect} heuristics\\
\hline
\harlC & HARL algorithm using the {\em Connect} heuristic \\
\harlQC & HARL using {\em Quick} and {\em Connect} heuristics\\ 
\hline
\harfC & HARF algorithm using the {\em Connect} heuristic \\
\harfQC & HARF using {\em Quick} and {\em Connect} heuristics\\ 
\hline
\end{tabular}}
\end{center}
\caption{{\bf Contenders compared in this study.} 
A contender is run in the mass start mode ($p_r=1$) or rolling start mode ($p_r<1$).
All the algorithms also use the {\em Informed} heuristic.}
\label{tab:contenders} 
\end{table} 

\paragraph{Statistics.} The following statistics are reported:
\begin{itemize}
\item The time to find the first path and the length of the final path;

\item The number of successful point creations per iteration. Note
  that this ratio is $<1$ in general.  If collisions were checked with
  infinite precision (Rmk \ref{rmk:precision}), this ratio would be
  one for HARG.  In practice, a collision detection is triggered when
  the segment length is $<\distcol$, whence a value $<1$ as well.

\item The success rate \ie the percentage of runs that succeeded in \% (NB: no number means
all runs succeeded);
\end{itemize}

\subsection{Implementation} 
\label{sec:impl}

\ifANONYMOUS
\else
The code is written in generic C++ and is currently being integrated into the \sbllong (\sblweb{} and \sblref).
\fi
Instantiating this code for any configuration space requires
assembling a so-called traits class providing routines for the operations
mentioned in Sec. \ref{sec:generalsetup}, in particular sampling
configurations.
See Section \ref{sec:planning-SE}
for configuration spaces products of $\SEn{3}$.

Calculations were run on a DELL precision 5480 equipped with 20 logical cores of type Intel(R) Core(TM) i9-
13900H, 32 GB of RAM, and running Fedora 42.

\ifANONYMOUS
\else
Result videos are available on \url{https://sbl.inria.fr/data/molecular-motion-planning/}.
\fi

\begin{remark}
\label{rmk:precision}
As noted in Section \ref{sec:generalsetup}, collision
predicates compare pairwise distances between primitives against a
fixed threshold $\distcol$. In theory, exact predicates require
assumptions on the number types (rational, algebraic, etc)~coding
the geometry of primitives and the operations carried
out~\cite{kettner2008classroom}.
Our instantiations involve rotations whence trigonometric
functions--see Sec. \ref{sec:planning-SE}. Robustness on these can be
achieved using multiprecision and/or interval number types--such as
MPFR or MPFI~\cite{fousse2005mpfr}.
Practically, in using a $\distcol$ much larger than numeric precision, we did not face
robustness issues using double floating point number types.
\end{remark}

\subsection{Hyper-parameter tuning}
\label{sec:hpt}

\paragraph{Steering range $\delta$.} 
To get the best for all algorithms, we first tune the steering range
$\delta$.  The {\em Connect} versions of all algorithms 
(\rrtconnect, \harfconnect, \harlconnect and
\harconnect, with $\kappa=0$) are run $N_r=50$ times for equally
spaced values of $\delta$ and for each algorithm, the value of
$\delta$ yielding the lowest mean time to obtain the first path is
retained (Fig. \ref{fig:delta-comparison}).

\begin{figure}[htb]
\begin{center}
\includegraphics[width=\coeffonefigCOLW\columnwidth]{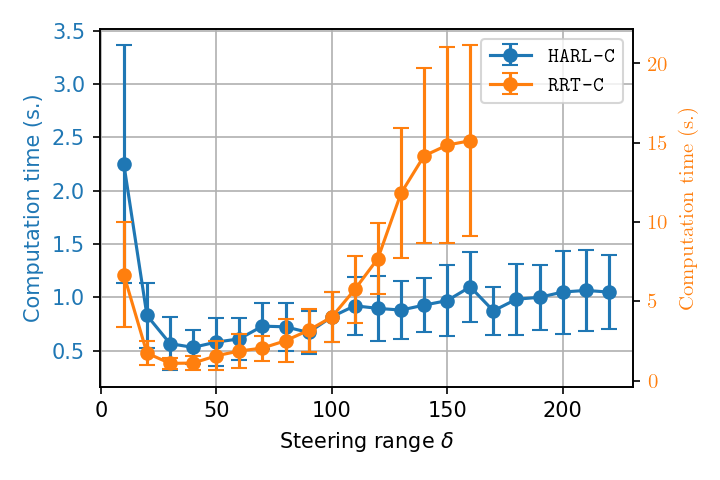}
\end{center}
\caption{
{\bf Hyper-parameter tuning: computation time to find a path for \rrtconnect (right axis)
and for \harlconnect (left axis), averaged over $N_r=50$ runs, with respect to $\delta$.}
Model used \cradleModel{25}{coarse}.
For these runs, the minimum value is 1.1 seconds ($\delta=30$) for \rrtconnect 
and 0.5 second ($\delta=40$) for \harlconnect.}
\label{fig:delta-comparison} 
\end{figure}

\paragraph{The hybrid interpolating between HAR and RRT with $\kappa$.} 
We have seen that RRT and HAR strategies are obtained as limit cases
of a generic method biasing the direction of the
extension (Section~\ref{sec:interpol}).
This behavior is observed in practice using \harfC with a von
Mises-Fisher distribution on directions, with $\kappa=0$
(resp. $\kappa=\infty$) giving HAR (resp.  RRT) (Fig. \ref{fig:tfc-vs-kapp}).

To delineate the intrinsic merits of RRT and HAR, we simply set $\kappa=0$ in the sequel.

\begin{figure}[htb]
\centerline{
\includegraphics[width=\coeffonefigCOLW\columnwidth]{\wdir/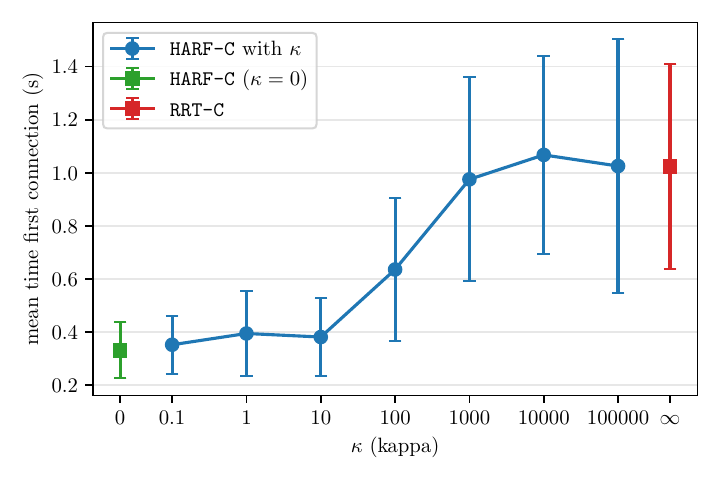}
}
\caption{{\bf Incidence of the bias on directions $\kappa$ on the time of first connection.}
Runs on \cradleModel{25}{coarse} averaged over $N_r=50$ repetitions.
When $\kappa \to \infty$, HARF recovers RRT's behavior.
}
\label{fig:tfc-vs-kapp}
\end{figure}

\subsection{Models}
\label{sec:models}

Our experiments are primarily focused on configuration spaces which
are products of $\SEn{3}$, for which we selected/designed challenging cases
with large number of degrees of freedom.

\paragraph{Models from the  OMPL library.}
We chose the challenging model {\em Cubicles} from the OMPL
library~\cite{sucan2012open} to test all algorithms (Fig.~\ref{fig:piano-plots} - row 1). 
In this $\adjdof{6}$ scenario, which corresponds to $\calC = \SEn{3}$, an object 
moves through obstacles and corridors to reach its final destination.

\paragraph{2D models.} We design 2D planning problems by 
requiring $N=10$ circular robots with varying radii to move in a
square. With $\calR$ the bounding box, this $\adjdof{2N}=\adjdof{20}$
model corresponds to $\calC=\calR^N$.  Robots must reach their final
destination while avoiding the other robots and a rectangular
obstacle.  We consider two scenarios named
\twoDimModel{small-obstacle} and \twoDimModel{big-obstacle}
respectively with a small and large central obstacle.
(Fig.~\ref{fig:piano-plots} - rows 2 and 3).

\paragraph{Multi-molecular cradle models.}
We define a molecular cradle as a set of $N$ rigid molecules/domains
moving independently. These are $\adjdof{6N}$ models with
configuration space $\calC=\SEn[N]$.
To explore a vast array of situations, we generate cradles at
random using {\em polycubes}, namely a connected random domain of $n$
cubes on the regular cubic lattice (Fig. \ref{fig:random-domain}). We
generate such random domains in a greedy fashion. The polycube is
initialized with a unit cube at the origin. A face of a cube $c$ is
termed {\em free} if $c$ does not have a neighboring cube sharing that
face. We maintain a list of cubes with free faces. Choosing at random
a cube and a free face of it, we extend the polycube by adding a cube
sharing that face. The process stops when $n$ cubes have been
generated.  We then convert the polycube into a pseudo-molecule by
replacing each cube by a ball of random radius.
\medskip

In practice, we
create four models of increasing complexity:
\cradleModel{9} -- $\adjdof{54}$, 
\cradleModel{25}{coarse} -- $\adjdof{150}$, 
\cradleModel{25}{tight}  -- $\adjdof{150}$, 
(Fig.~\ref{fig:cradles-plots}) and \cradleModel{64} (Fig.~\ref{fig:cradles-plots-sparse} - row 3) --  $\adjdof{384}$.

\subsection{RRT versus HAR in mass start mode}
\label{sec:reading-tables}

We run contenders in mass start mode ($p_r = 1$)
and report results using scatter plots
(Fig. \ref{fig:piano-plots}, detailed in Fig. \ref{fig:piano-tables}).

\paragraph{Piano mover and the \emph{Cubicles} model.}
The fastest method to report a path is \rrtconnect ($0.85$s, against
$3.5$--$7.7$\,s for the HAR variants), and final path lengths are on
par (e.g.\ $1696$ for \quickharfstarconnect, the shortest, versus
$1711$ for \quickrrtstarconnect). 
On this example and other one-robot cases (data not shown), RRT's
fixed-$\delta$ steering is very efficient, and the move {\em forward
  strategy} associated with the fixed direction $\urrt$ from $\qfrom$
towards $\qsampled$ pays off.

\paragraph{2D disk models.}
The picture reverses as soon as several robots move jointly
(Fig. \ref{fig:piano-plots} detailed in rows 2 and 3 of Fig. \ref{fig:piano-tables} ).
On the easier \twoDimModel{small-obstacle} scenario, \harfconnect finds a first path in
$0.017$\,s against $0.51$\,s for \rrtconnect ($\sim\!30$ times faster)
and on the harder \twoDimModel{big-obstacle} scenario in $0.40$\,s against $42.8$\,s for RRT ($\sim\!100$ times faster),
with consistently shorter paths.
All HAR variants, by contrast, run with $100\%$ success, exploring
the space efficiently as seen with the many nodes generated by HARG
and HARL.

A second observation is the sensitivity of the methods to the size of
the obstacle. Increasing this size results in a loss in
running time (24 fold for \harfconnect, 84 fold for \rrtconnect), a consequence of the collision inherent to the steering
procedures (Fig. \ref{fig:steering-procedures}).

\paragraph{Molecular cradles.}
The three molecular cradle models (Fig.\ref{fig:cradles-plots} detailed in rows 1 to 3 of
Fig.~\ref{fig:cradles-tables}) are of increasing difficulty 
and reveal a transition.
On the easy \cradleModel{9} ($\adjdof{54}$), all robots are moved
at once: \harfconnect is fastest to connect
($0.037$\,s) with \rrtconnect coming second  ($0.049$\,s), path
lengths being slightly shorter for \quickharlstarconnect algorithm.  
On the medium \cradleModel{25}{coarse}
($\adjdof{150}$), the gap opens: \harfconnect reaches a
path in $0.4$\,s against $1.13$\,s for \rrtconnect, with also shorter paths.
The tight \cradleModel{25}{tight} ($\adjdof{150}$) is starting to be
problematic for all methods: RRT and HARF do not find any path, and
HARL and HARG, while still succeeding on every run, become slow
($62$--$179$\,s) within the 5 minutes wall clock budget.

No algorithm was able to solve the \cradleModel{64} model in the mass
start mode within a 10 minutes wall clock constraint, which calls for experiments
in rolling start mode.

\figBegins
\begin{center}
\headerCubicles
\resizebox{\coeffresizeboxCOLW\columnwidth}{!}{ 
\begin{tabular}{ccc}
\includegraphics[width=.2\linewidth]{\wdir/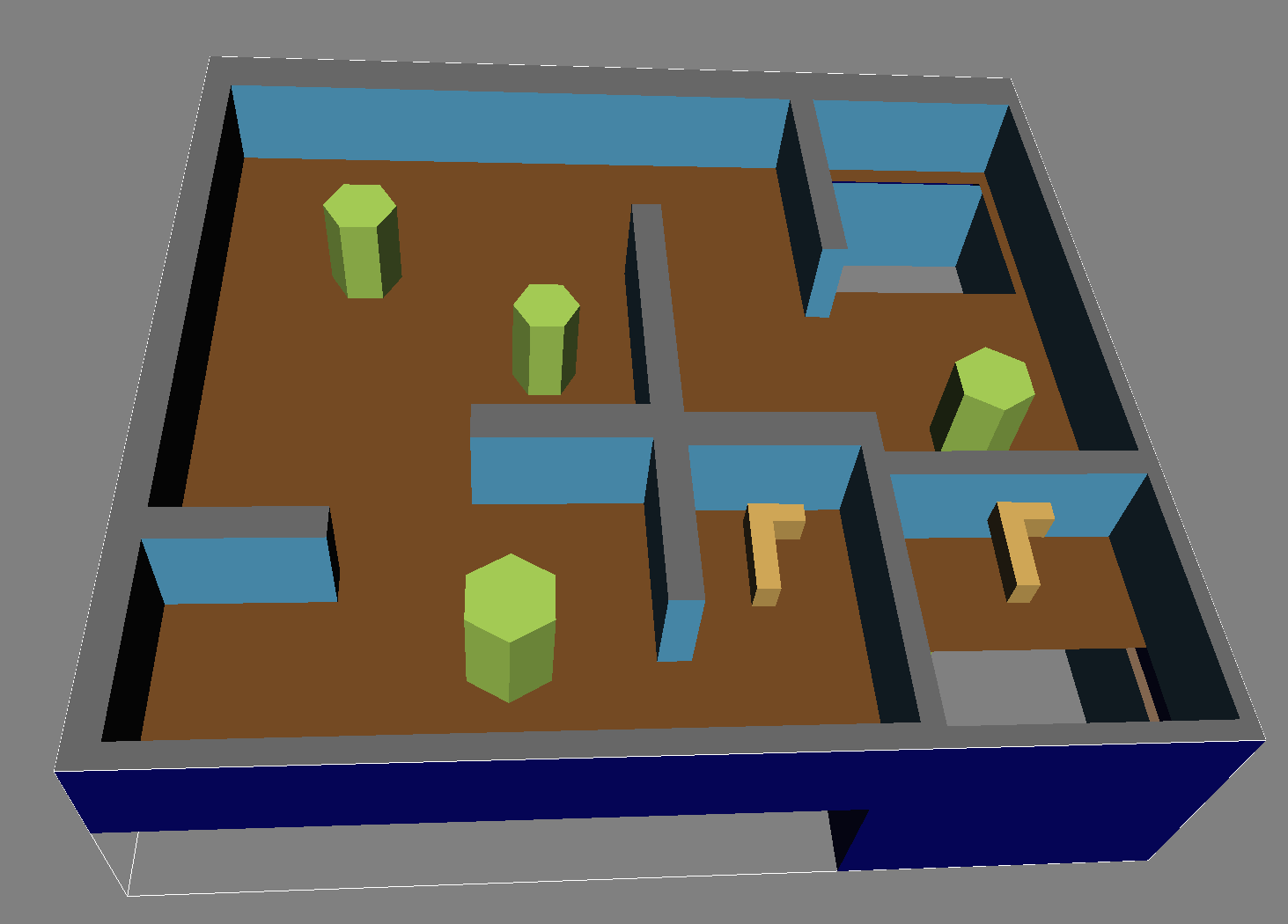} &
\includegraphics[width=.4\linewidth]{\wdir/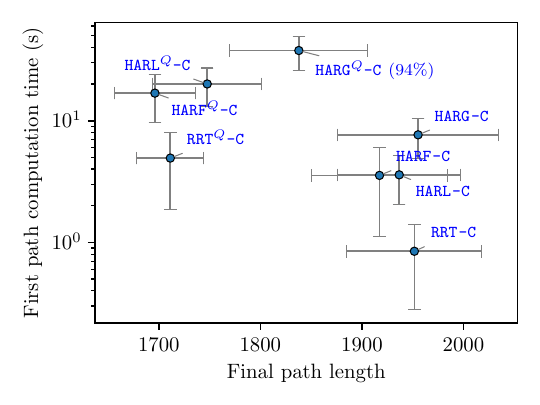} & 
\includegraphics[width=.4\linewidth]{\wdir/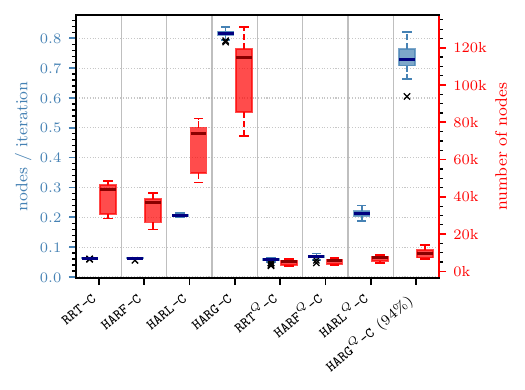}
\end{tabular}}
\headerTDSmall
\resizebox{\coeffresizeboxCOLW\columnwidth}{!}{ 
\begin{tabular}{ccc}
\includegraphics[width=.2\linewidth]{\wdir/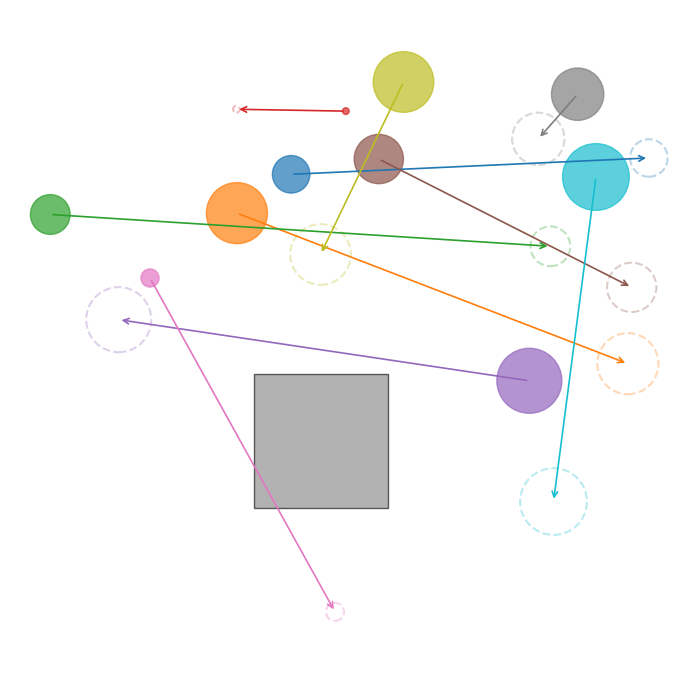} &
\includegraphics[width=.4\linewidth]{\wdir/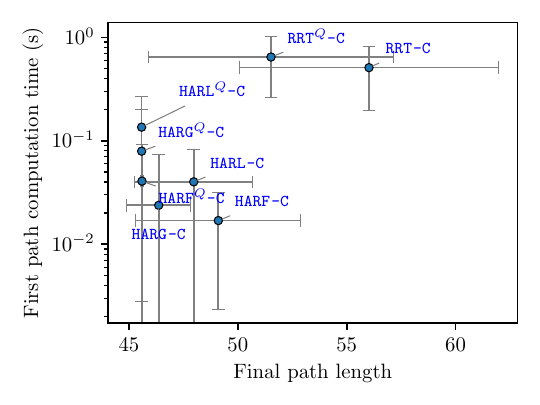} & 
\includegraphics[width=.4\linewidth]{\wdir/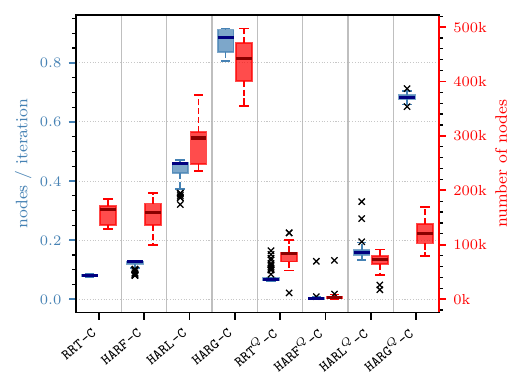}
\end{tabular}}
\headerTDLarge
\resizebox{\coeffresizeboxCOLW\columnwidth}{!}{ 
\begin{tabular}{ccc}
\includegraphics[width=.2\linewidth]{\wdir/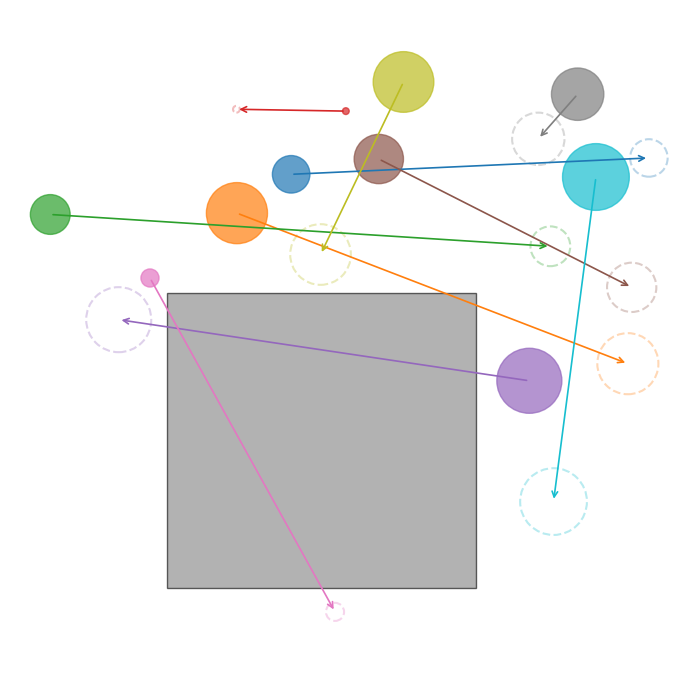} &
\includegraphics[width=.4\linewidth]{\wdir/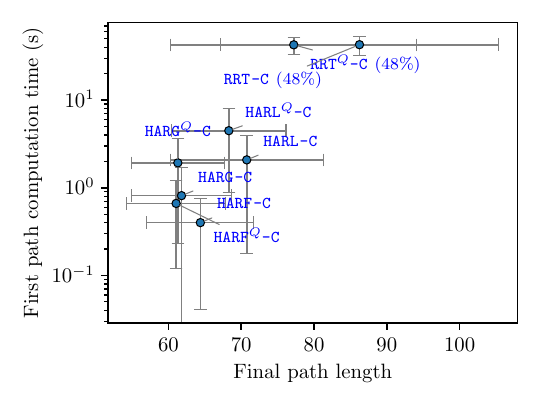} & 
\includegraphics[width=.4\linewidth]{\wdir/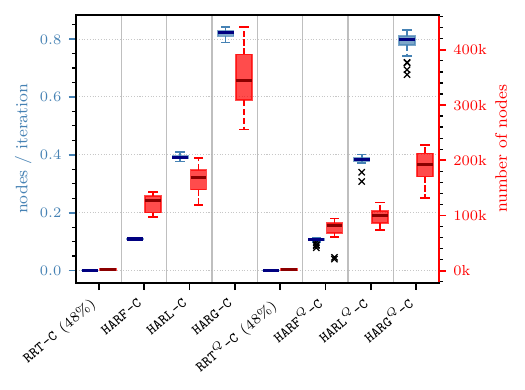}
\end{tabular}}
\end{center}
\caption{
{\bf Results for Cubicles and the 2D models in mass start mode ($p_r=1$), averaged over $N_r=50$ runs. }
{\bf (Left column)} The models.
{\bf (Middle column)} Performances of contenders from Table
\ref{tab:contenders}. The nearest to the bottom left corner, the
better.  Percentages refer to the success rate over the $N_r$ repeats.
{\bf (Right column)} Successful nodes per iterations (blue) and total number of created nodes (red).}
\label{fig:piano-plots}
\figEnds

\figBegins
\begin{center}
\headerCradleNine
\resizebox{\coeffresizeboxCOLW\columnwidth}{!}{ 
\begin{tabular}{ccc}
\includegraphics[width=.2\linewidth]{\wdir/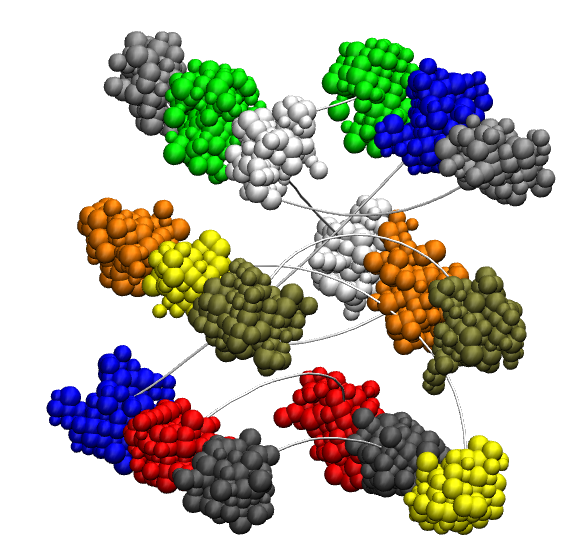} &
\includegraphics[width=.4\linewidth]{\wdir/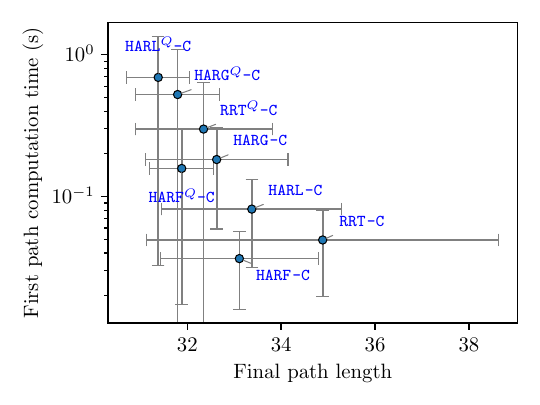} & 
\includegraphics[width=.4\linewidth]{\wdir/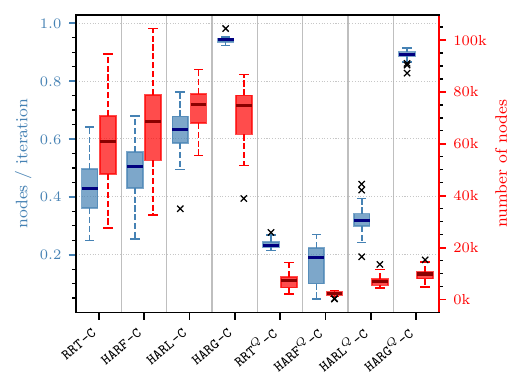}
\end{tabular}}
\headerCradleTFcoarse
\resizebox{\coeffresizeboxCOLW\columnwidth}{!}{ 
\begin{tabular}{ccc}
\includegraphics[width=.2\linewidth]{\wdir/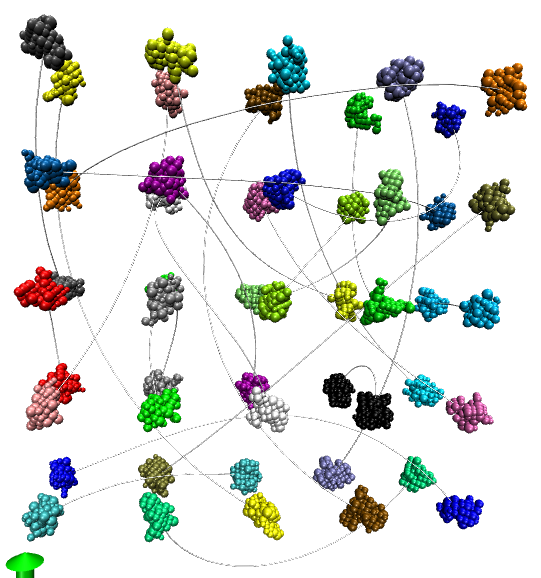} &
\includegraphics[width=.4\linewidth]{\wdir/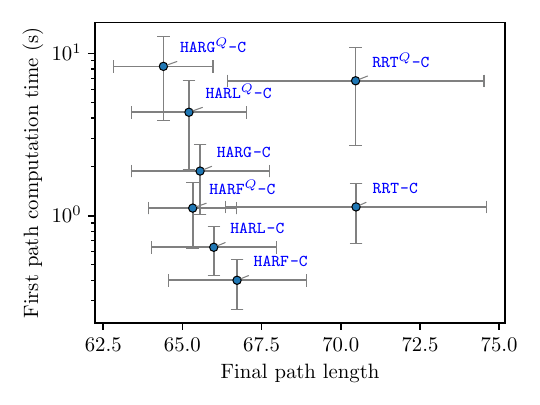} & 
\includegraphics[width=.4\linewidth]{\wdir/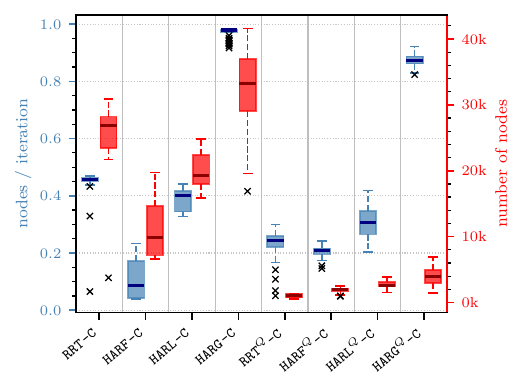}
\end{tabular}}
\headerCradleTFtight
\resizebox{\coeffresizeboxCOLW\columnwidth}{!}{ 
\begin{tabular}{ccc}
\includegraphics[width=.2\linewidth]{\wdir/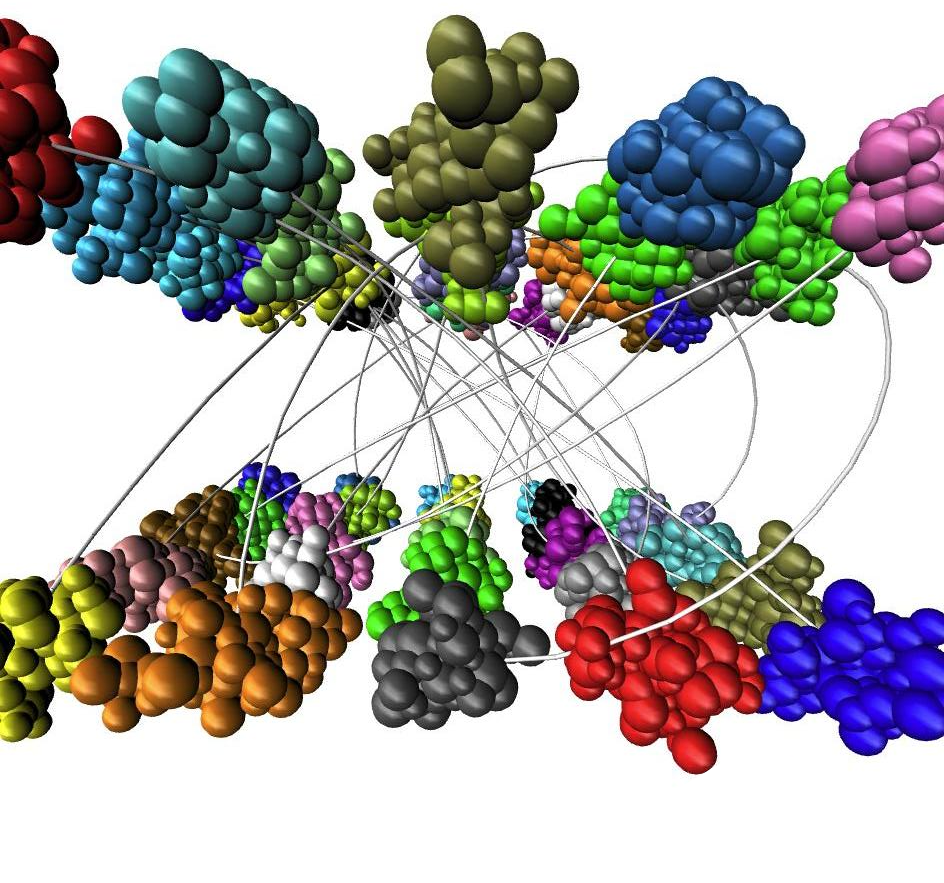} &
\includegraphics[width=.4\linewidth]{\wdir/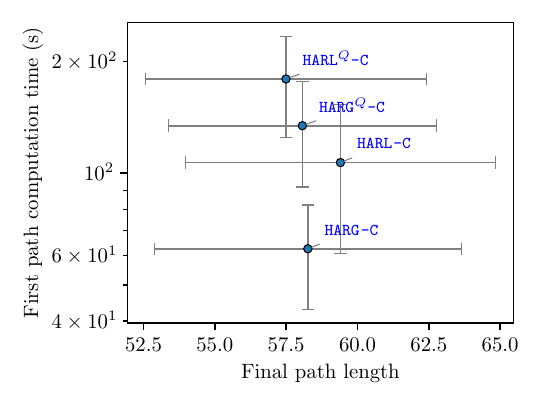} & 
\includegraphics[width=.4\linewidth]{\wdir/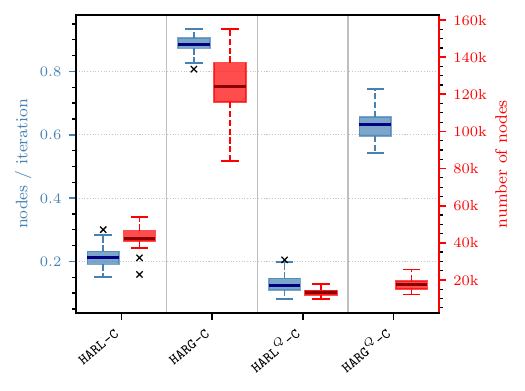}
\end{tabular}}
\end{center}
\caption{
{\bf Results for molecular cradle models in mass start mode ($p_r=1$), averaged over $N_r=50$ runs. }
{\bf (Left column)} The models.
{\bf (Middle column)} Performances of contenders from Table
\ref{tab:contenders}. The nearest to the bottom left corner, the
better.  Percentages refer to the success rate over the $N_r$ repeats.
{\bf (Right column)} Successful nodes per iterations (blue) and total number of created nodes (red).}
\label{fig:cradles-plots}
\figEnds

\subsection{RRT versus HAR in rolling start mode}
\label{sec:rs-mode}

The dimensionality reduction/projection methods discussed in Introduction 
aim at lowering the dimensionality  of the search space for complex planning problems.
A simple measure to do so is to make a robot move with probability $p_r\in (0,1]$ at each step,
with fixed robots acting as obstacles.
Experiments with the complex cases \cradleModel{25}{coarse} and 
\cradleModel{25}{tight} are conducted to assess the
role of $p_r$ (Fig. \ref{fig:tfc-vs-pr}). These show that
there is an optimal value of $p_r$, which may be equal to one
for easy/tame problems.
We observe that this optimal value is smaller for RRT than for HARL.
Indeed, the move forward strategy associated with 
the direction $\urrt$ from $\qfrom$ towards $\qsampled$ is getting
detrimental when crowding increases, hence the help of a lower
value $p_r$. Whereas HAR has more steps backward and can afford a larger
number of moving molecules/robots.

\paragraph{Results with the optimal $p_r$.}
Each algorithm on each model is now run with the best $p_r$ value
obtained via exhaustive search (Fig. \ref{fig:tfc-vs-pr}).  

Algorithms' performances improve with $p_r < 1$ when 
cluttered-ness increases, namely in \twoDimModel{big-obstacle},
\cradleModel{25}{tight} and \cradleModel{64}~(Fig.~\ref{fig:cradles-plots-sparse}).

For \cradleModel{25}{tight}, a low value of $p_r$ benefits RRT and
HARF which solve the \cradleModel{25}{tight} in 3.4s for \harfconnect
and 5.2s for \rrtconnect.  \harlconnect and \harconnect are also an order of
magnitude faster with their optimal $p_r$ value identified, as they go from
a $\sim 1-2$ minutes to $\sim 7-30$ seconds.

For the massive \cradleModel{64} model, \harfconnect can now find a
path in 17.3s with $p_r = 0.35$ and \rrtconnect in 41.6s with
$p_r=0.1$.  The \harlconnect and \harconnect algorithms were not able to solve
\cradleModel{64} within five minutes.
\medskip

Since tuning $p_r$ benefits every algorithm, the comparison at tuned $p_r$ is
fair. In doing so, \harfconnect remains ahead of \rrtconnect on all multi-robot
models but \cradleModel{25}{coarse}, by a factor $1.2$ to $2.4$, the margin
being largest on the largest instance.
The decisive advantage of the HAR-based samplers is therefore one of
robustness rather than of raw speed: they keep finding paths in regimes
where the fixed-step steering of \rrtconnect collapses.


\figBegins
\begin{center}
\resizebox{\coeffresizeboxCOLW\columnwidth}{!}{ 
\begin{tabular}{cc}
\includegraphics[width=.495\linewidth]{\wdir/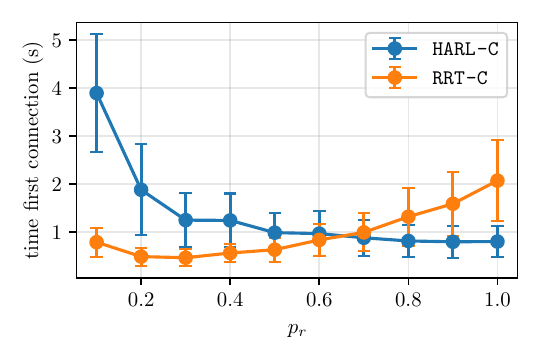}&
\includegraphics[width=.495\linewidth]{\wdir/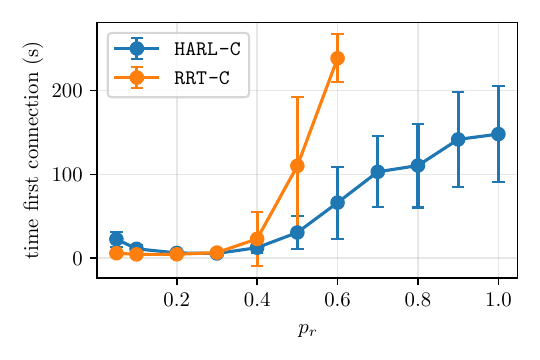}\\
\cradleModel{25}{coarse} & \cradleModel{25}{tight}
\end{tabular}}
\end{center}
\caption{{\bf Incidence of the per robot moving probability $p_r$ on the time
of first connection.} Run on
\cradleModel{25}{coarse} (left) and \cradleModel{25}{tight} (right) averaged over $N_r=50$ repetitions.}
\label{fig:tfc-vs-pr}
\figEnds

\figBegins
\begin{center}
\headerTDLarge
\resizebox{\coeffresizeboxCOLW\columnwidth}{!}{ 
\begin{tabular}{ccc}
\includegraphics[width=.2\linewidth]{\wdir/10r1o.png} &
\includegraphics[width=.4\linewidth]{\wdir/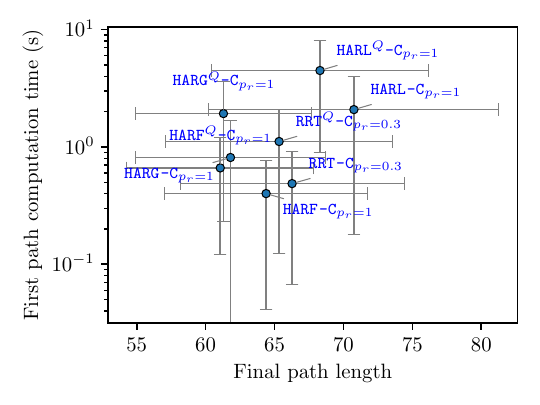} & 
\includegraphics[width=.4\linewidth]{\wdir/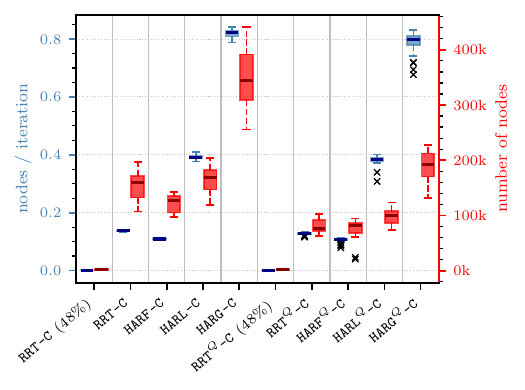}
\end{tabular}}

\headerCradleTFtightSparse
\resizebox{\coeffresizeboxCOLW\columnwidth}{!}{ 
\begin{tabular}{ccc}
\includegraphics[width=.2\linewidth]{\wdir/walls5x5-hard.png} &
\includegraphics[width=.4\linewidth]{\wdir/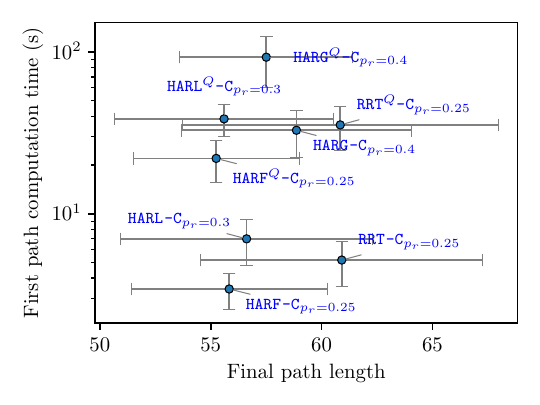} & 
\includegraphics[width=.4\linewidth]{\wdir/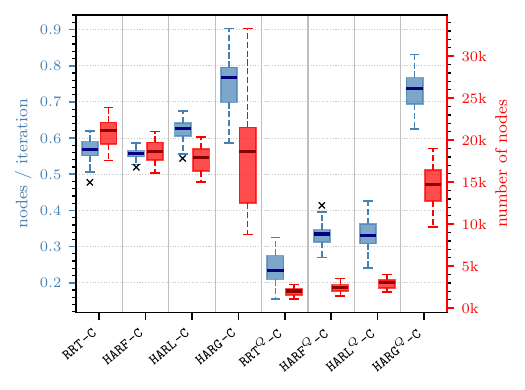}
\end{tabular}}

\headerCradleSF
\resizebox{\coeffresizeboxCOLW\columnwidth}{!}{ 
\begin{tabular}{ccc}
\includegraphics[width=.2\linewidth]{\wdir/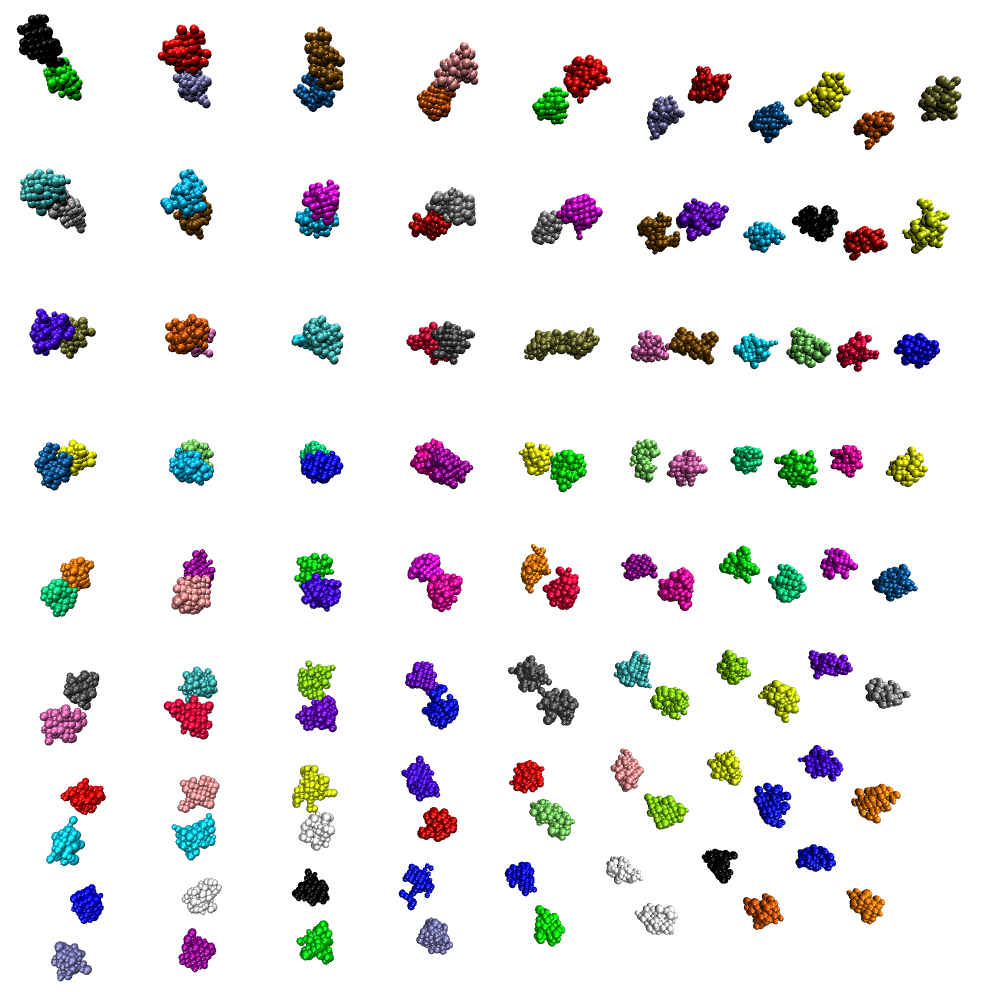} &
\includegraphics[width=.4\linewidth]{\wdir/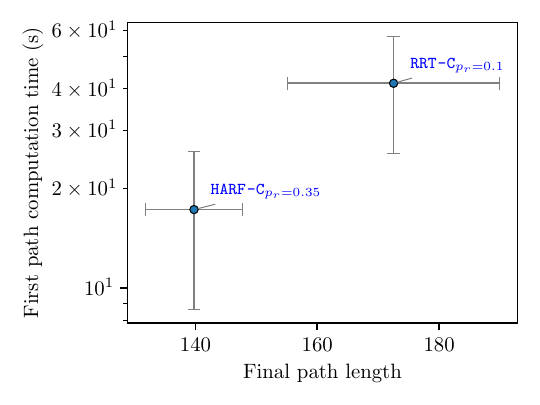} & 
\includegraphics[width=.4\linewidth]{\wdir/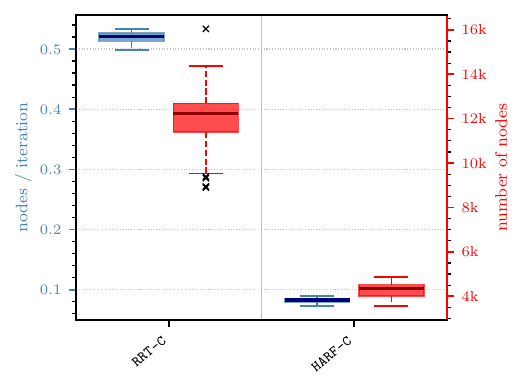}
\end{tabular}}

\end{center}
\caption{
{\bf Results for some cluttered models in {\em rolling start mode} ($p_r\in(0,1]$), 
averaged over $N_r=50$ runs. }
{\bf (Left column)} The models. For the \cradleModel{64} model, 
64 molecules are aligned on a grid with shuffled initial and final positions.
{\bf (Middle column)} Performances of contenders from Table
\ref{tab:contenders}. The nearest to the bottom left corner, the
better.  Percentages refer to the success rate over the $N_r$ repeats.
{\bf (Right column)} Successful nodes per iterations (blue) and total number of created nodes (red).}
\label{fig:cradles-plots-sparse}
\figEnds

\subsection{Observations irrespective of the start mode}

This battery of tests, with $p_r=1$ and $p_r<1$, 
which uses models of increasing complexity, reveals a precise picture:
\begin{itemize}
\item For single-robot problems (e.g.\ the OMPL \emph{Cubicles}),
  \rrtconnect performs best to find path quickly, and all variants are
  equally good in producing a short one.

\item For multi-robot problems, all algorithms can solve difficult
  problems, which may require tuning the fraction of moving robots
  $p_r$. \rrtconnect at $p_r=1$ is rarely ideal, struggling to find
  a path, a consequence of mass concentration in high dimensions and of the 
aforementioned forward mode.
\harlconnect and \harconnect are the most resilient at $p_r=1$, since
by construction they can take arbitrarily small productive steps and
benefit from backward steps.
For large instances ($25$ to $64$ molecules), tuning $p_r$ pays off, and
\rrtconnect and \harfconnect then solve these hard problems quickly. 

In terms of general performances, \harfconnect with low $p_r$ yields
very quick results.  As a contender, \rrtconnect with a suitable value
of $p_r$ (generally lower than that of \harfconnect) shows satisfactory
(albeit worse) performances.  \harlconnect and then \harconnect with a
tuned $p_r$ are slightly longer to find a path.

\item The \emph{Quick} heuristic makes every algorithm slower but
  shortens the final path; the gain is significant mostly in the
  presence of obstacles. Otherwise, simple path-shortening routines suffice to
  obtain high quality paths.

\item The \emph{successful steering} rate, which measures the
  probability of a steering step to generate a new sample, reflects
  the ability of each sampling rule to populate the free space.
 Algorithm HARG has the most productive iterations (it always places a
 node along the drawn direction, setting aside collisions for
 distances less than the threshold $\distcol$), HARL is very efficient
 as well, whereas HARF and RRT, which fix the step as
 $d(\qfrom,\qnew)\approx\delta$, produce fewer samples.
\end{itemize}

\subsection{Comparison against recent contenders}

While our work focuses on the benefits of decoupling positions and
directions to improve RRT-like algorithms, comments are in order with
respect to very recent work.

The {\em fiber bundle}-based approach from \cite{orthey2024multilevel}
stratifies the search space into a certain number of {\em levels}
to obtain lower-dimensional problems. 
On the pro side, the method handles complex cases, the largest
ones being the single robot hypercube (\adjdof{100}; 98 levels), eight flying
drones (\adjdof{48}; 8 levels), and a complex manipulator (\adjdof{72}; 3 levels).
For the flying drones scenario, we compare the original
results against our RRT and HAR algorithms.  The original paper plans
in 0.14 second (0.59 second with
\rrtconnect)~\cite{orthey2024multilevel}.  In mass start mode, our
\rrtconnect yields a path in 0.17 second and \harfconnect requires
0.21 second. In rolling start mode, the numbers fall to 0.062 second
and 0.14 second respectively. These results are comparable
in the mass start mode, and better in the rolling start mode.
In addition, our methods
do not require any complex operation on the configuration space.


The recent diffusion based approach~\cite{shaoul2025multi} reports
experiments on multi-robot systems with up to 20 robots and $\adjdof{40}$.
Running times are of the order of tens of seconds. For comparable systems,
our methods are one order of magnitude faster (Fig. \ref{fig:cradles-plots-sparse}).

\subsection{Steering in high dimensional spaces: different behaviors on rotations} 

The difficulties faced by RRT algorithms in the mass start mode
stem from the steering procedure and the value of $\delta$.
We illustrate this difficulty for our conformational spaces involving rotations.

By the classical mass concentration properties in high dimensional
spaces~\cite{blum2020foundationsdatascience}, the new sample
$\qsampled$ is outside the ball $B(\qfrom, \delta)$ with high
probability, so that $\qnew$ is retracted onto the associated sphere
$S(\qfrom, \delta)$ (Fig.~\ref{fig:steering-procedures}).  If $\delta$
is large and the clearance low, the collision probability is high.
On the other hand, if $\delta$ is small, the rotational shift
is small.  As an example consider an illustrative run of \rrtconnect
for \cradleModel{25}{tight} with $\delta=2$.
After a 1-minute run and 217k iterations, a mere 551 nodes have been
created (despite the small $\delta$ value), and all of their rotational components lie 
near the rotational component of the start or goal node/pose (Fig.~\ref{fig:rrt-rotation-failure}).
This is due to the fact that the global bounding box has dimensions 
$236\times236\times220$ and therefore the
steering procedure only retains a small percentage of the
rotation. Indeed, the ratio $\delta$ over $\dist(\qfrom, \qsampled)$ is often
very small thus after the steering procedure keeping only
the $\delta$ over $\dist(\qfrom, \qsampled)$ part of the segment,
both rotational components of $\qfrom$ and $\qnew$ become very close. 
This problem does not happen in low-dimensional scenarios as 
there is a higher chance that the sampled points are close to their nearest neighbors.

\begin{figure}[htb]
\begin{center}
\includegraphics[width=\coeffonefigCOLW\columnwidth]{\wdir/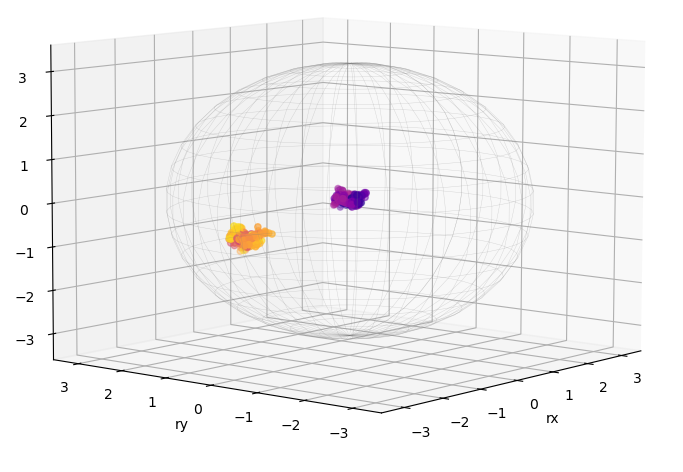}
\end{center}
\caption{{\bf Model \cradleModel{25}{tight}:     rotations of the molecule n${}^\circ1$ only after a one-minute run of
    \rrtconnect with $\delta=2$.}  Representing rotations with
  the angle-axis convention, a rotation is a point in the ball of
  radius $\pi$, with antipodal points on the bounding sphere being
  identified.
The two clusters illustrate the inability of \rrtconnect  to sample
rotations outside a small neighborhood of the starting and goal points.
}
\label{fig:rrt-rotation-failure} 
\end{figure}

\subsection{Failed improvements}

We conclude this section by reporting variations that looked
promising but did not bring any significant improvement.

\paragraph{Orthogonal sampling.} To maximize exploration in HARG, 
taking rays perpendicular to the previous ones seems appealing.  A
sample in the tree is located on a ray which has hit an obstacle
or the bounding box. Sampling a new ray in the orthogonal complement
might help to {\em slide} along the obstacle.
For multi-robot problems, this can be implemented in two ways: either
ray${}_1$ is perpendicular to ray${}_2$ for some scalar product, or
each center of mass of each robot in ray${}_2$ is moving
perpendicularly to its corresponding motion in ray${}_1$.

Neither of these two strategies gave any significant performance
changes.

\paragraph{Reset frequency.} 
The original HAR algorithm defines a Markov chain, sampling a new
point from the previous one.  Applying this idea to perform robot
planning results in a tree which is a path, with poor overall
performances (time to find a solution, and solution quality).  An
intermediate solution consists of following HAR for $\Ksteps$
consecutive steps, before choosing the sample to be extended with the
Voronoi bias rule.  However, experiments showed that using a reset at
each iteration, that is $\Ksteps=0$, is the best choice.

\paragraph{Using variable steering range in RRT.}
In RRT, using a fixed value $\delta$ for the steering range may not
be suitable if the clearance in $\Cfree$ varies.
This is particularly true for multi-robot planning, as a short
distance between a single pair of robots may jeopardize the sample
$\qnew$.
Sampling $\delta$ according to some distribution (\eg an exponential distribution)
provides more slack, with small values used in regions of low
clearance.  However, we did not observe any global consistent
improvement for multi-robot planning.

\paragraph{Weighting the rotational parts in $r$ parameter.} 
The distance metric used requires weighting the rotational part --
\eqref{eq:distSE3linear}, with large values of $r$ penalizing
rotations. But varying $r$ did not yield significant improvements,
so that we stuck to $r=1$.

\paragraph{Changing rotational parts of nodes in RRT.} 
As discussed above (Fig.~\ref{fig:rrt-rotation-failure}), RRT fails
for small $\delta$ because the rotational parts are not changing
enough.  Patching the steered points with random rotations
did not help either.

\section{Outlook}

\paragraph{A new tier of sampling-based planners.}
Rapidly-exploring random trees (RRT)  and Hit-and-Run (HAR) algorithms
are powerful techniques to explore high dimensional spaces.
The former has been introduced in motion planning and uses the
so-called Voronoi bias to diffuse in unseen regions of the
configuration space.
The latter comes from optimization and numerical mathematics.  It
defines a Markov chain that can be used to sample the uniform
measure.

Our interest in combining these techniques comes from multi-robot
planning, a challenging task due to the need to handle jointly fixed
obstacles and moving obstacles -- the robots themselves.  For such
cases, a large number of robots entails a high dimensional space, and
RRT faces difficulties in cluttered environments, since choosing the
steering distance faces conflicting constraints: a too small value
makes the robots barely advance, and too large values increase the
collision rate.
We show that a mixed sampling strategy combining RRT and HAR
(resulting in HARG, HARL or HARF algorithms) gets the best of both
worlds, the former favoring the exploration of unseen regions and the
latter enabling search in any possible direction, avoiding
mass-concentration phenomena.
Importantly, our algorithms are purely sampling-based and require
trivial changes in existing sampling-based planners -- they merely
require primitives to sample random points and random directions.

Tests on different scenarios show that RRT variants struggle with an
increased dimensionality.  While they are very good with a small number
of robots, they are not able to produce many nodes with a large cohort of
robots.  In contrast, HARL and HARG algorithms are more resilient to
the increase of dimensionality.

\paragraph{Taming down configuration space.}
Recent work has shown that complex planning problems could also be
tackled by stratifying the configuration space prior to using {\em
 classical} planners on the strata.  While highly effective, this
approach is relatively complex as it requires designing the
stratification upon determining an appropriate number of levels.
For multi-robot problems, a simpler (and apparently new) way to obtain
this behavior is the {\em rolling start mode} associated with a probability
$p_r<1$, which amounts to using randomization to decide which robots
move.

In fact, taming down a configuration space and developing novel planners,
as we do, are complementary approaches since the latter can be used
on strata of the original space.
More powerful planners also bear a clear advantage,
namely their simplicity and genericity.

\paragraph{Practical matters.}
As a practical guideline, RRT remains an excellent and well-understood
default to plan in low dimensional spaces.
As the number of robots grows and the space possibly becomes cluttered, the
fixed-step steering of RRT degrades and it increasingly needs a
reduced moving fraction $p_r$ to make progress; the HAR-based samplers,
and in particular HARF, then become markedly more effective. On
the tightest instances, no sampler copes with $p_r=1$, and tuning
$p_r$ benefits every algorithm. 
In short, we recommend using RRT for low dimensional problems, HAR
(and notably HARF, possibly with tuned $p_r$) for the many-robots case
motivating this work.

\paragraph{Further work.}
Our work opens interesting research avenues.
On the theoretical side, understanding the properties of our hybrid
algorithm would be extremely interesting, and would stress the
critical hyper-parameters as a function of the geometry of the free
space. Such questions appear challenging though, to capture the interplay
between RRT and HAR.
Application-wise, our planners will prove beneficial to address
the multi-robot planning problem for tens of robots in cluttered
environments. One particularly compelling application is the
elucidation of efflux mechanisms implemented by several cell types,
with applications to the design of antibiotics and anti-cancer drugs.



\bigskip
\noindent{\bf Acknowledgments.}  Pierre Alliez is acknowledged for his help to process
the flying drones geometric models using AlphaWrap.

This work has been supported by the French government through the 
3IA C\^ote d’Azur Investments (ANR-23-IACL-0001),
and by a grant from the French government, managed by the French National Research Agency (ANR) under the 
France 2030 program operated by the Inria Quadrant Programme (ANR-24-RRII-0002).

Akshat Jah is acknowledged for early work on this project.

\bibliographystyle{unsrt}




\clearpage

\appendix
\beginSI
\section{RRT and HAR: illustrations}

Figures \ref{fig:comparison-algorithms} and
\ref{fig:comparison-algorithms-har} depict the behavior of RRT and
HARL algorithms on the following toy system, which also illustrates
the genericity of the algorithms since a different configuration space is used:
\begin{itemize}
\item Consider a cradle model with $N$ molecules;
\item For each molecule, consider the geodesic path between the initial and final positions defined as a parameterized curve on $[0, 1]$ 
computed by Eq. \eqref{eq:segmentlinear}.
\item The overall configuration space is $\calC = [0, 1]^N$. Note that this corresponds to moving $N$ robots independently;
\item We use the algorithms to solve the associated motion planning problem
with  $\qstart= \latrans{ (0,\dots,0)} $ and $\qgoal=\latrans{(1,\dots,1)}$.
\item To illustrate the paths found by the various algorithms, we project nodes, edges and the solution onto the motion space of the first two robots.
\end{itemize}

\begin{figure}[htbp]
\begin{center}
\begin{tabular}{cc}
\includegraphics[width = .4\linewidth]{\wdir/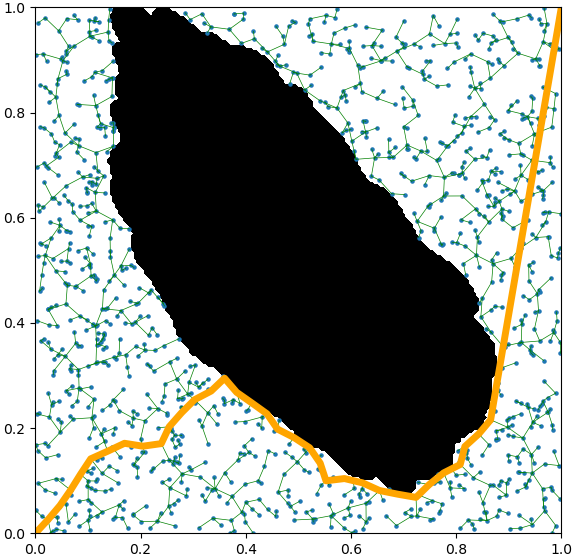} &
\includegraphics[width = .4\linewidth]{\wdir/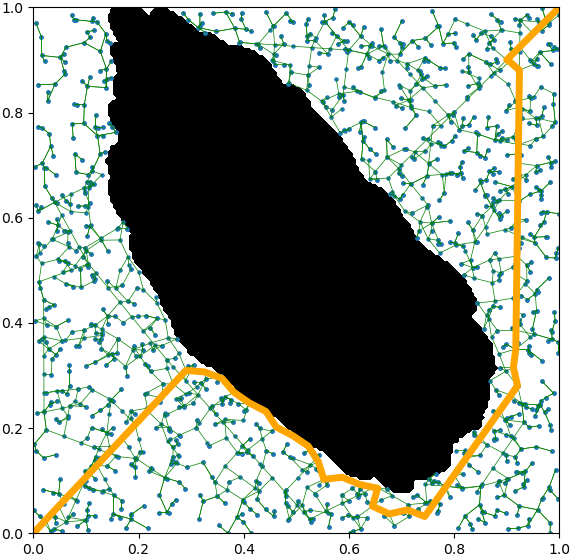} \\
\rrt & \rrtconnect\\
\includegraphics[width = .4\linewidth]{\wdir/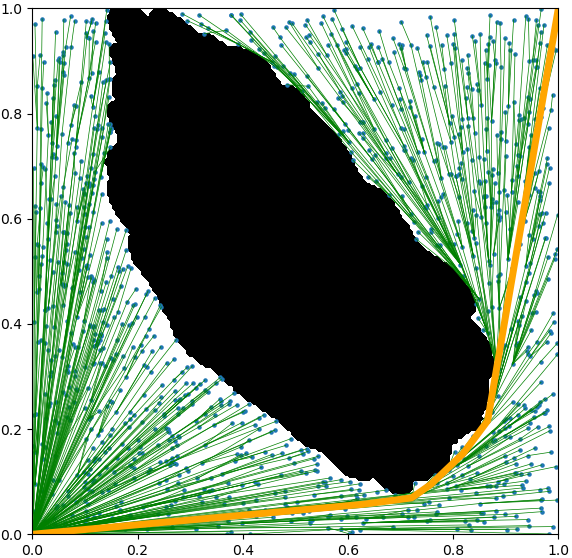} &
\includegraphics[width = .4\linewidth]{\wdir/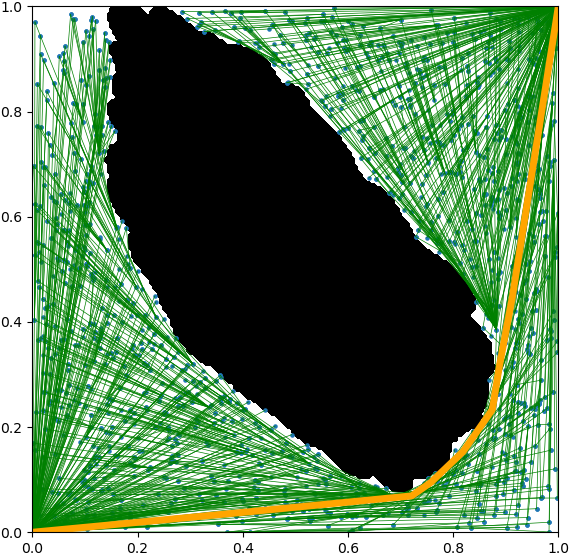} \\
\rrtstar & \rrtstarconnect\\
\includegraphics[width = .4\linewidth]{\wdir/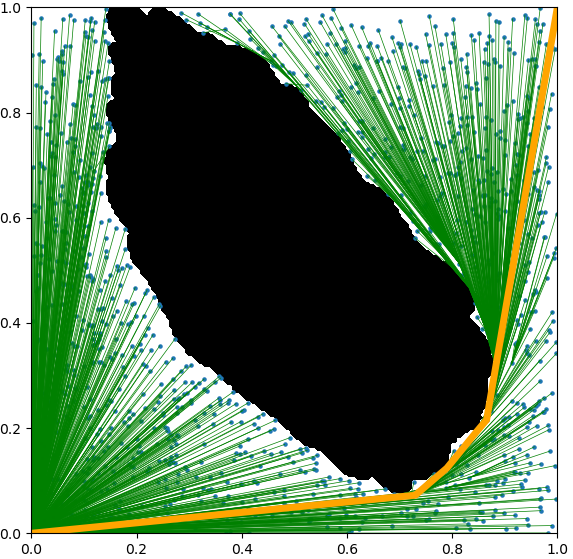} &
\includegraphics[width = .4\linewidth]{\wdir/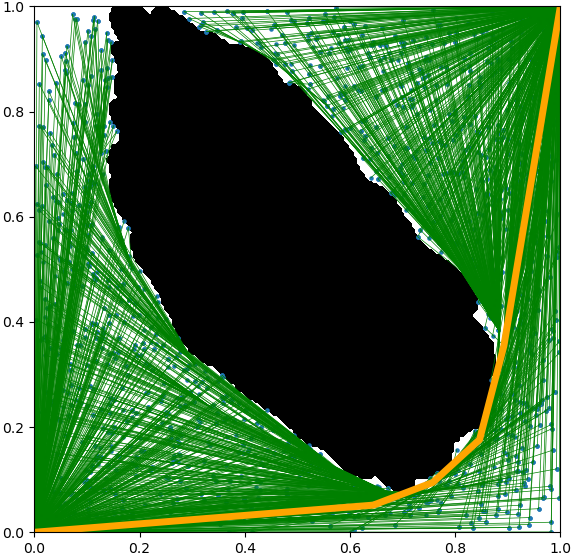} \\
\quickrrtstar & \quickrrtstarconnect
\end{tabular}
\end{center}
\caption{{\bf Molecular motions of two molecules in a cradle.}
 The collision zone is depicted in black. We show the behavior of the
 6 RRT algorithms with an exagerately reduced steering value of 2\%
 and after 1500 iterations. The best path is shown in orange.}
\label{fig:comparison-algorithms}
\end{figure}

\begin{figure}[htbp]
\begin{center}
\begin{tabular}{cc}
\includegraphics[width = .4\linewidth]{\wdir/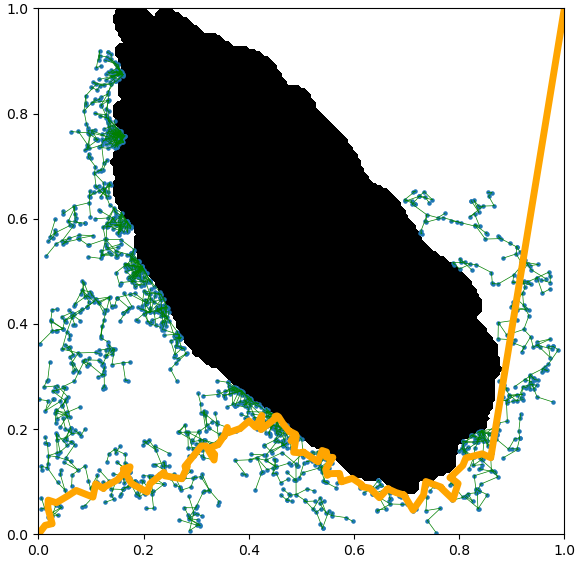}&
\includegraphics[width = .4\linewidth]{\wdir/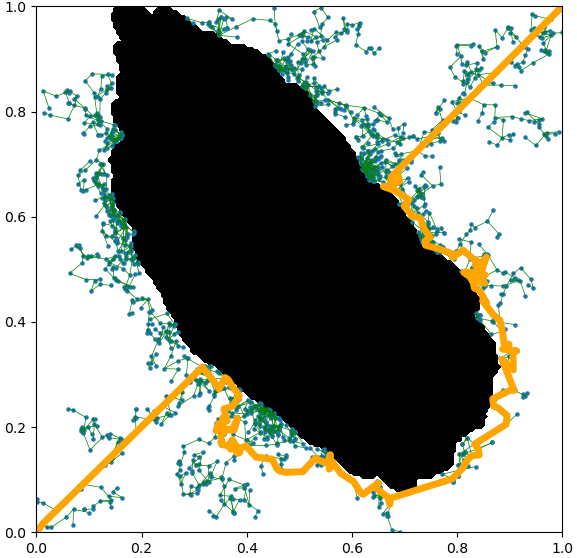} \\
\har & \harC\\
\includegraphics[width = .4\linewidth]{\wdir/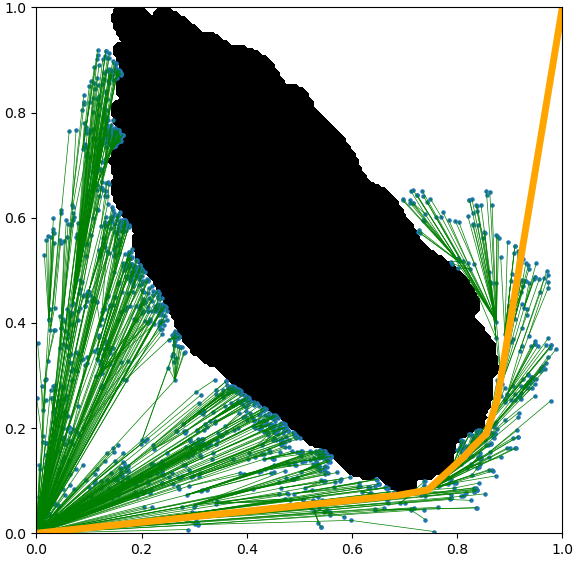}&
\includegraphics[width = .4\linewidth]{\wdir/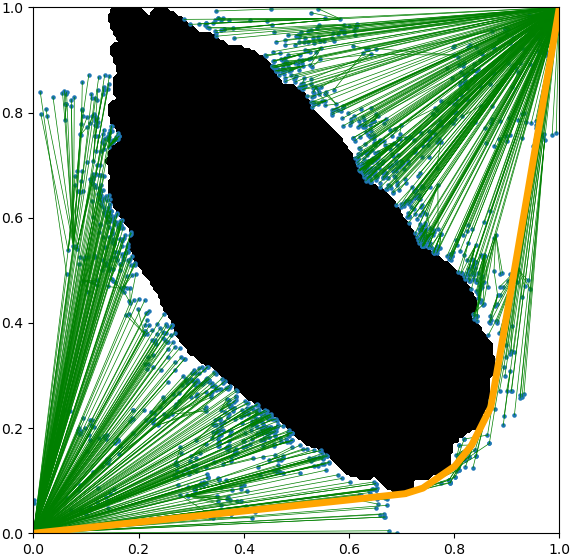} \\
\harS & \harSC\\
\includegraphics[width = .4\linewidth]{\wdir/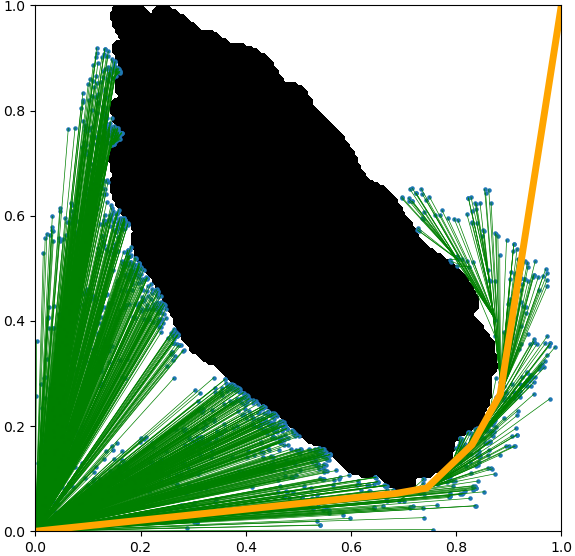}&
\includegraphics[width = .4\linewidth]{\wdir/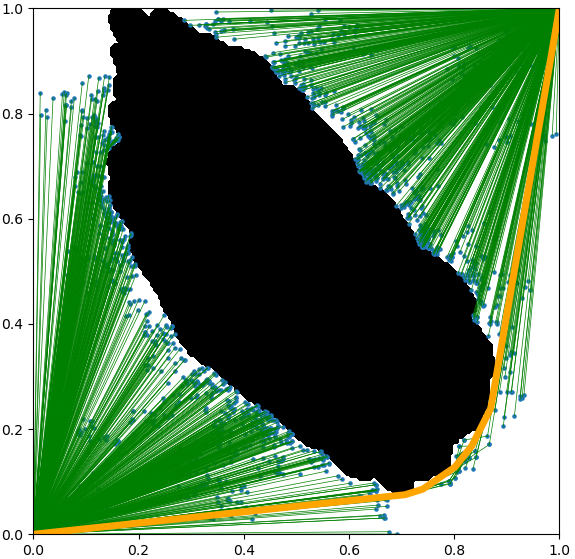} \\
\harQ  & \harQC
\end{tabular}
\end{center}
\caption{As Figure \ref{fig:comparison-algorithms} \textit{mutatis mutandis} for HAR algorithms.}
\label{fig:comparison-algorithms-har}
\end{figure}

\clearpage
\FloatBarrier

\section{Models}

\begin{figure}[htb]
\centerline{\includegraphics[width=\coeffonefigCOLW\columnwidth]{\wdir/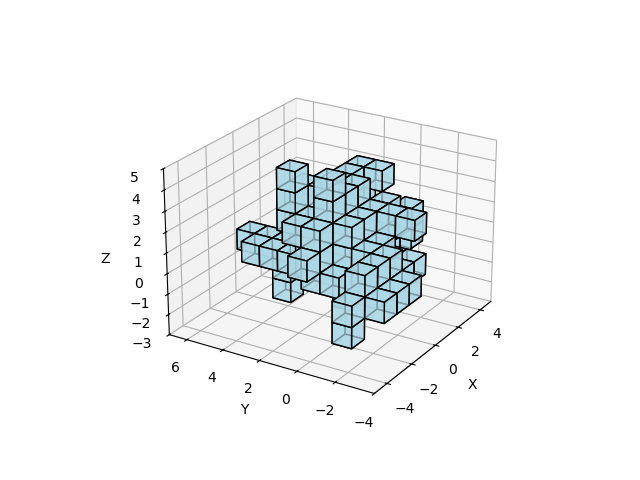}}
\caption{{\bf A random domain defined as a connected region of unit cubes -- a 3D polycube.}
Replacing the cubes by balls yields a fictitious random molecule/domain.
See also \url{https://en.wikipedia.org/wiki/Polycube}} 
\label{fig:random-domain} 
\end{figure}

\section{Sampling coupled translations}
\label{sec:norm_3Nsampling}
\label{sec:sampling-coupled-transformations}

Let $T_1, \ldots, T_N \in \mathbb R^3$ such that $(T_1, \ldots T_N)$
is sampled on $\Sd{3N-1}$. 
The following proves lemma \ref{lem:R}.

\begin{proposition}\label{prop:T_normal_law}
For all $i$, let $G_i \sim \calN(0, I_3)$ and $G = (G_1, \ldots, G_N)$. Then $T_i$ and $G_i / \|G\|$ have the same law.
\end{proposition}

\begin{proof}
This is direct from the previous section as we use gaussian distributions to sample on the sphere
\end{proof}

\begin{proposition}\label{prop:normalkhibeta}

The following relationships hold between the Normal, $\chi^2$  and Beta distributions:
\ifTWOCOLUMNS
\[
\|G_i\|^2 \sim \chi^2(3)\qquad \sum_{j\neq i} \|G_j\|^2 \sim \chi^2(3N-3)
\]
\[
\textrm{ therefore }\dfrac{\|G_i\|^2}{\|G\|^2} \sim \textrm{Beta}(1.5, 1.5N-1.5)
\]
\else
\[
\|G_i\|^2 \sim \chi^2(3)\qquad \sum_{j\neq i} \|G_j\|^2 \sim \chi^2(3N-3) \textrm{ therefore }\dfrac{\|G_i\|^2}{\|G\|^2} \sim \textrm{Beta}(1.5, 1.5N-1.5)
\]
\fi
\end{proposition}
\begin{proof}
The first two come from the definition of $\chi^2$ and the last comes from the fact that $X\sim \chi^2(a), Y\sim \chi^2(b)$ and $X$ and $Y$ independent implies $\dfrac{X}{X+Y}\sim\textrm{Beta}(a/2, b/2)$.
\end{proof}

\begin{proposition}\label{prop:expenctancyT} Using the density function of Beta law we have:
\[
E(\|T_i\|) = \dfrac{2}{\sqrt\pi}\dfrac{\Gamma(1.5N)}{\Gamma(1.5N + 0.5)}
\]
\end{proposition}

\begin{proof}
From Proposition \ref{prop:normalkhibeta} and \ref{prop:T_normal_law}, we have that $\|T_i\|^2\sim \textrm{Beta}(\alpha, \beta)$ with $\alpha=1.5$ and $\beta=1.5N-1.5$. Therefore the expectancy of $\|T_i\|$ is a simple computation:
\ifTWOCOLUMNS
\[
\begin{array}{rl}
E(\|T_i\|) &= \displaystyle\frac{1}{B(\alpha, \beta)} \int_0^1 \sqrt{x} \times x^{\alpha-1} (1-x)^{\beta-1} \,\textrm dx \\ 
&= \displaystyle\dfrac{1}{B(1.5, 1.5N-1.5)} \int_{0}^{1} x (1-x)^{1.5N-2.5} \,\textrm dx \\
&= \displaystyle\dfrac{B(2, 1.5N-1.5)}{B(1.5, 1.5N-1.5)} \\
&= \dfrac{\Gamma(2)\Gamma(1.5N-1.5)}{\Gamma(1.5N+0.5)} \dfrac{\Gamma(1.5N)}{\Gamma(1.5N-1.5)\Gamma(1.5)} \\
&= \dfrac{2\Gamma(1.5N)}{\sqrt{\pi}\Gamma(1.5N+0.5)}
\end{array}
\]
\else
\[
\begin{array}{rl}
E(\|T_i\|) = E(\sqrt{\|T_i\|^2}) &= \displaystyle\frac{1}{B(\alpha, \beta)} \int_0^1 \sqrt{x} \times x^{\alpha-1} (1-x)^{\beta-1} \,\textrm dx \\ 
&= \displaystyle\dfrac{1}{B(1.5, 1.5N-1.5)} \int_{0}^{1} x (1-x)^{1.5N-2.5} \,\textrm dx \\
&= \displaystyle\dfrac{B(2, 1.5N-1.5)}{B(1.5, 1.5N-1.5)} \\
&= \dfrac{\Gamma(2)\Gamma(1.5N-1.5)}{\Gamma(1.5N+0.5)} \dfrac{\Gamma(1.5N)}{\Gamma(1.5N-1.5)\Gamma(1.5)} \\
&= \dfrac{2\Gamma(1.5N)}{\sqrt{\pi}\Gamma(1.5N+0.5)}
\end{array}
\]
\fi
\end{proof}

\begin{remark} Some values can be computed with some values of $N$:
\begin{itemize}
\item \underline{For $N=1$:} As $\Gamma(1.5) = \frac{\sqrt\pi}2$ and $\Gamma(2) = 1$, we have $E(\|T_1\|) = 1$ which is logical as $T$ is just $T_1$ thus $T_1$ is sampled on the unit sphere of $\mathbb R^3$.
\item \underline{For $N=2$:} we have $\Gamma(3) = 2$ and $\Gamma(3.5) = \frac{15}8\sqrt{\pi}$ therefore $E(\|T_i\|) = \dfrac{32}{15\pi} \approx 0.68$
\end{itemize}
\end{remark}

\begin{proposition} \label{prop:equivalentN}
For $N\to\infty$, $E(\|T_i\|) \sim \sqrt{\dfrac{8}{3\pi N}}$
\end{proposition}
\begin{proof}
It follows from Proposition \ref{prop:expenctancyT} and from the fact that $\dfrac{\Gamma(x+a)}{\Gamma(x+b)}\sim_\infty x^{a-b}$.
\end{proof}
Lemma \ref{lem:R} follows by sampling $T$ on $\rho\,\Sd{3N-1}$ with $\rho = \sqrt{\dfrac{3\pi N}{8}} \delta$.

\section{Detailed results}

\figBegins
\headerCubicles
\centerline{%
\begin{minipage}{0.24\linewidth}
\includegraphics[width=\linewidth]{\wdir/omp-cubicles-planning-problem.png}
\end{minipage}
\begin{minipage}{0.75\linewidth}
  \small
\begin{tabular}{lccccc}
\toprule
Algorithm & $p_r$ & Time first iter.\ [s] & Path length & {\setlength\extrarowheight{0pt}\begin{tabular}[c]{@{}c@{}}Successful steering\\per iteration\end{tabular}} & Nodes \\
\midrule
\rrtC & $1$ & $\mathbf{0.846 \pm 0.56}$ & $1951.6 \pm 66.4$ & $0.062$ & $4.0 \times 10^{4}$ \\
\harfC & $1$ & $\underline{3.55 \pm 2.4}$ & $1917.3 \pm 66.6$ & $0.063$ & $3.3 \times 10^{4}$ \\
\harlC & $1$ & $3.59 \pm 1.5$ & $1936.7 \pm 60.6$ & $0.206$ & $6.7 \times 10^{4}$ \\
\harC & $1$ & $7.65 \pm 2.7$ & $1955.3 \pm 79.0$ & $0.817$ & $1.0 \times 10^{5}$ \\
\midrule
\rrtQC & $1$ & $4.93 \pm 3.1$ & $1711.3 \pm 32.9$ & $0.057$ & $4.8 \times 10^{3}$ \\
\harfQC & $1$ & $16.8 \pm 7.1$ & $\mathbf{1696.2 \pm 39.9}$ & $0.068$ & $5.4 \times 10^{3}$ \\
\harlQC & $1$ & $20.0 \pm 7.1$ & $1747.7 \pm 53.5$ & $0.213$ & $6.8 \times 10^{3}$ \\
\harQC(94\%) & $1$ & $37.7 \pm 12$ & $1837.9 \pm 67.9$ & $0.734$ & $9.6 \times 10^{3}$ \\
\bottomrule
\end{tabular}
\end{minipage}}
\headerTDSmall
\centerline{%
\begin{minipage}{0.24\linewidth}
\includegraphics[width=\linewidth]{\wdir/10r1o_mid.png}
\end{minipage}
\begin{minipage}{0.75\linewidth}
  \small
\begin{tabular}{lccccc}
\toprule
Algorithm & $p_r$ & Time first iter.\ [s] & Path length & {\setlength\extrarowheight{0pt}\begin{tabular}[c]{@{}c@{}}Successful steering\\per iteration\end{tabular}} & Nodes \\
\midrule
\harfC & $1$ & $\mathbf{0.0169 \pm 0.015}$ & $49.1 \pm 3.8$ & $0.119$ & $1.5 \times 10^{5}$ \\
\harC & $0.9$ & $\underline{0.0212 \pm 0.025}$ & $47.4 \pm 3.2$ & $0.871$ & $4.3 \times 10^{5}$ \\
\harC & $1$ & $0.0238 \pm 0.051$ & $46.4 \pm 1.5$ & $0.874$ & $4.3 \times 10^{5}$ \\
\rrtC & $0.3$ & $0.0288 \pm 0.021$ & $51.3 \pm 4.2$ & $0.304$ & $2.9 \times 10^{5}$ \\
\harlC & $1$ & $0.0401 \pm 0.041$ & $48.0 \pm 2.7$ & $0.436$ & $2.9 \times 10^{5}$ \\
\rrtC & $1$ & $0.510 \pm 0.31$ & $56.0 \pm 6.0$ & $0.080$ & $1.6 \times 10^{5}$ \\
\midrule
\harfQC & $1$ & $0.0407 \pm 0.051$ & $\mathbf{45.6 \pm 0.1}$ & $0.006$ & $4.0 \times 10^{3}$ \\
\rrtQC & $0.3$ & $0.0782 \pm 0.061$ & $46.3 \pm 2.5$ & $0.449$ & $1.1 \times 10^{5}$ \\
\harQC & $1$ & $0.0794 \pm 0.12$ & $\mathbf{45.6 \pm 0.0}$ & $0.683$ & $1.2 \times 10^{5}$ \\
\harQC & $0.9$ & $0.0825 \pm 0.13$ & $\mathbf{45.6 \pm 0.0}$ & $0.689$ & $1.2 \times 10^{5}$ \\
\harlQC & $1$ & $0.135 \pm 0.13$ & $\mathbf{45.6 \pm 0.0}$ & $0.165$ & $7.1 \times 10^{4}$ \\
\rrtQC & $1$ & $0.646 \pm 0.38$ & $51.5 \pm 5.6$ & $0.079$ & $8.0 \times 10^{4}$ \\
\bottomrule
\end{tabular}
\end{minipage}}
\headerTDLarge
\centerline{%
\begin{minipage}{0.24\linewidth}
\includegraphics[width=\linewidth]{\wdir/10r1o.png}
\end{minipage}
\begin{minipage}{0.75\linewidth}
  \small
\begin{tabular}{lccccc}
\toprule
Algorithm & $p_r$ & Time first iter.\ [s] & Path length & {\setlength\extrarowheight{0pt}\begin{tabular}[c]{@{}c@{}}Successful steering\\per iteration\end{tabular}} & Nodes \\
\midrule
\harfC & $1$ & $\mathbf{0.400 \pm 0.36}$ & $64.4 \pm 7.4$ & $0.109$ & $1.2 \times 10^{5}$ \\
\rrtC & $0.3$ & $\underline{0.487 \pm 0.42}$ & $66.3 \pm 8.1$ & $0.137$ & $1.5 \times 10^{5}$ \\
\harC & $1$ & $0.813 \pm 0.88$ & $\mathbf{61.8 \pm 6.9}$ & $0.819$ & $3.5 \times 10^{5}$ \\
\harlC & $1$ & $2.08 \pm 1.9$ & $70.8 \pm 10.5$ & $0.392$ & $1.6 \times 10^{5}$ \\
\rrtC(48\%) & $1$ & $42.8 \pm 11$ & $86.2 \pm 19.2$ & $0.000$ & $2.7 \times 10^{3}$ \\
\midrule
\harfQC & $1$ & $0.662 \pm 0.54$ & $\mathbf{61.1 \pm 6.8}$ & $0.105$ & $7.7 \times 10^{4}$ \\
\rrtQC & $0.3$ & $1.11 \pm 0.99$ & $65.3 \pm 8.2$ & $0.129$ & $8.1 \times 10^{4}$ \\
\harQC & $1$ & $1.92 \pm 1.7$ & $\mathbf{61.3 \pm 6.4}$ & $0.791$ & $1.9 \times 10^{5}$ \\
\harlQC & $1$ & $4.48 \pm 3.6$ & $68.3 \pm 7.9$ & $0.382$ & $9.7 \times 10^{4}$ \\
\rrtQC(48\%) & $1$ & $42.7 \pm 9.4$ & $77.2 \pm 16.9$ & $0.000$ & $2.5 \times 10^{3}$ \\
\bottomrule
\end{tabular}
\end{minipage}}
\caption{
  {\bf Table results of contenders on first three models:} Cubicles, \twoDimModel{small-obstacle} and \twoDimModel{big-obstacle}. Results averaged
  over 50 runs. The results for $p_r=1$ and the optimal $p_r$ (if different from 1) are presented and are
  sorted by time to get the first path.
Each table is split into a \emph{Connect} block and a \emph{Quick-Connect} block.
Bold marks the best entry of a block, underline the second; a percentage after a
name is its success rate when below $100\%$.
}
\label{fig:piano-tables}
\figEnds

\figBegins
\headerCradleNine
\centerline{%
\begin{minipage}{0.24\linewidth}
\includegraphics[width=\linewidth]{\wdir/walls3x3.png}
\end{minipage}
\begin{minipage}{0.75\linewidth}
\small
\begin{tabular}{lccccc}
\toprule
Algorithm & $p_r$ & Time first iter.\ [s] & Path length & {\setlength\extrarowheight{0pt}\begin{tabular}[c]{@{}c@{}}Successful steering\\per iteration\end{tabular}} & Nodes \\
\midrule
\harfC & $1$ & $\mathbf{0.0366 \pm 0.020}$ & $33.1 \pm 1.7$ & $0.494$ & $6.7 \times 10^{4}$ \\
\rrtC & $1$ & $\underline{0.0495 \pm 0.030}$ & $34.9 \pm 3.8$ & $0.424$ & $6.1 \times 10^{4}$ \\
\harlC & $1$ & $0.0816 \pm 0.050$ & $33.4 \pm 1.9$ & $0.628$ & $7.3 \times 10^{4}$ \\
\harC & $1$ & $0.182 \pm 0.12$ & $32.6 \pm 1.5$ & $0.942$ & $7.2 \times 10^{4}$ \\
\midrule
\harfQC & $1$ & $0.158 \pm 0.14$ & $31.9 \pm 0.7$ & $0.169$ & $2.3 \times 10^{3}$ \\
\rrtQC & $1$ & $0.299 \pm 0.34$ & $32.3 \pm 1.5$ & $0.236$ & $7.0 \times 10^{3}$ \\
\harQC & $1$ & $0.523 \pm 0.56$ & $31.8 \pm 0.9$ & $0.890$ & $9.5 \times 10^{3}$ \\
\harlQC & $1$ & $0.691 \pm 0.66$ & $\mathbf{31.4 \pm 0.7}$ & $0.323$ & $7.0 \times 10^{3}$ \\
\bottomrule
\end{tabular}
\end{minipage}}

\headerCradleTFcoarse
\centerline{%
\begin{minipage}{0.24\linewidth}
\includegraphics[width=\linewidth]{\wdir/walls5x5.png}
\end{minipage}
\begin{minipage}{0.75\linewidth}
\small
\begin{tabular}{lccccc}
\toprule
Algorithm & $p_r$ & Time first iter.\ [s] & Path length & {\setlength\extrarowheight{0pt}\begin{tabular}[c]{@{}c@{}}Successful steering\\per iteration\end{tabular}} & Nodes \\
\midrule
\rrtC & $0.25$ & $\mathbf{0.380 \pm 0.16}$ & $68.5 \pm 2.8$ & $0.775$ & $3.0 \times 10^{4}$ \\
\harfC & $1$ & $\underline{0.400 \pm 0.13}$ & $66.7 \pm 2.2$ & $0.110$ & $1.1 \times 10^{4}$ \\
\harlC & $1$ & $0.639 \pm 0.21$ & $66.0 \pm 2.0$ & $0.384$ & $2.0 \times 10^{4}$ \\
\rrtC & $1$ & $1.13 \pm 0.45$ & $70.5 \pm 4.1$ & $0.446$ & $2.6 \times 10^{4}$ \\
\harC & $1$ & $1.88 \pm 0.86$ & $65.6 \pm 2.2$ & $0.970$ & $3.3 \times 10^{4}$ \\
\midrule
\harfQC & $1$ & $1.11 \pm 0.49$ & $65.3 \pm 1.4$ & $0.204$ & $1.9 \times 10^{3}$ \\
\rrtQC & $0.25$ & $2.86 \pm 1.5$ & $68.4 \pm 3.3$ & $0.415$ & $1.4 \times 10^{3}$ \\
\harlQC & $1$ & $4.34 \pm 2.4$ & $65.2 \pm 1.8$ & $0.309$ & $2.7 \times 10^{3}$ \\
\rrtQC & $1$ & $6.78 \pm 4.1$ & $70.5 \pm 4.1$ & $0.231$ & $1.0 \times 10^{3}$ \\
\harQC & $1$ & $8.32 \pm 4.4$ & $\mathbf{64.4 \pm 1.6}$ & $0.874$ & $4.0 \times 10^{3}$ \\
\bottomrule
\end{tabular}
\end{minipage}}

\headerCradleTFtight
\centerline{%
\begin{minipage}{0.24\linewidth}
\includegraphics[width=\linewidth]{\wdir/walls5x5-hard.png}
\end{minipage}
\begin{minipage}{0.75\linewidth}
\small
\begin{tabular}{lccccc}
\toprule
Algorithm & $p_r$ & Time first iter.\ [s] & Path length & {\setlength\extrarowheight{0pt}\begin{tabular}[c]{@{}c@{}}Successful steering\\per iteration\end{tabular}} & Nodes \\
\midrule
\harfC & $0.25$ & $\mathbf{3.43 \pm 0.85}$ & $\mathbf{55.8 \pm 4.4}$ & $0.557$ & $1.9 \times 10^{4}$ \\
\rrtC & $0.25$ & $\underline{5.17 \pm 1.6}$ & $60.9 \pm 6.4$ & $0.567$ & $2.1 \times 10^{4}$ \\
\harlC & $0.3$ & $6.99 \pm 2.2$ & $\mathbf{56.6 \pm 5.7}$ & $0.622$ & $1.8 \times 10^{4}$ \\
\harC & $0.4$ & $32.8 \pm 10$ & $58.9 \pm 5.2$ & $0.751$ & $1.8 \times 10^{4}$ \\
\harC & $1$ & $62.6 \pm 20$ & $58.3 \pm 5.4$ & $0.885$ & $1.2 \times 10^{5}$ \\
\harlC & $1$ & $107 \pm 46$ & $59.4 \pm 5.4$ & $0.214$ & $4.3 \times 10^{4}$ \\
\midrule
\harfQC & $0.25$ & $22.0 \pm 6.4$ & $\mathbf{55.2 \pm 3.7}$ & $0.333$ & $2.4 \times 10^{3}$ \\
\rrtQC & $0.25$ & $35.4 \pm 11$ & $60.8 \pm 7.1$ & $0.236$ & $2.0 \times 10^{3}$ \\
\harlQC & $0.3$ & $38.6 \pm 8.5$ & $\mathbf{55.6 \pm 4.9}$ & $0.337$ & $2.9 \times 10^{3}$ \\
\harQC & $0.4$ & $92.9 \pm 32$ & $57.5 \pm 3.9$ & $0.734$ & $1.5 \times 10^{4}$ \\
\harQC & $1$ & $134 \pm 42$ & $58.1 \pm 4.7$ & $0.631$ & $1.8 \times 10^{4}$ \\
\harlQC & $1$ & $179 \pm 55$ & $57.5 \pm 4.9$ & $0.129$ & $1.4 \times 10^{4}$ \\
\bottomrule
\end{tabular}
\end{minipage}}

\headerCradleSF
\centerline{%
\begin{minipage}{0.24\linewidth}
\includegraphics[width=\linewidth]{\wdir/walls8x8.png}
\end{minipage}
\begin{minipage}{0.75\linewidth}
\small
\begin{tabular}{lccccc}
\toprule
Algorithm & $p_r$ & Time first iter.\ [s] & Path length & {\setlength\extrarowheight{0pt}\begin{tabular}[c]{@{}c@{}}Successful steering\\per iteration\end{tabular}} & Nodes \\
\midrule
\harfC & $0.35$ & $\mathbf{17.3 \pm 8.6}$ & $\mathbf{139.8 \pm 8.0}$ & $0.082$ & $4.3 \times 10^{3}$ \\
\rrtC & $0.1$ & $\underline{41.6 \pm 16}$ & $172.5 \pm 17.4$ & $0.519$ & $1.2 \times 10^{4}$ \\
\bottomrule
\end{tabular}
\end{minipage}}
\caption{
    {\bf Table results of contenders on molecular cradle models:} 
\cradleModel{9}, \cradleModel{25}{coarse}, \cradleModel{25}{tight} and \cradleModel{64}.
Results averaged over 50 runs. The results for $p_r=1$ and the optimal $p_r$ (if different from 1) are presented and are
  sorted by time to get the first path.
Each table is split into a \emph{Connect} block and a \emph{Quick-Connect} block.
Bold marks the best entry of a block, underline the second; a percentage after a
name is its success rate when below $100\%$.
}
\label{fig:cradles-tables}
\figEnds



\clearpage
\tableofcontents
\end{document}